%% file: main.tex
\documentclass{article}

\PassOptionsToPackage{numbers, compress}{natbib}

\usepackage[main, final]{neurips_2026}
\usepackage[utf8]{inputenc}
\usepackage[T1]{fontenc}
\usepackage{hyperref}
\usepackage{url}
\usepackage{booktabs}
\usepackage{amsfonts}
\usepackage{nicefrac}
\usepackage{microtype}
\usepackage[dvipsnames]{xcolor}

\usepackage{algorithmic}
\usepackage{algorithm}

\usepackage{amsmath}
\usepackage{amssymb}
\usepackage{mathtools}
\usepackage{amsthm}
\usepackage{thmtools} 
\usepackage{thm-restate}

\usepackage{multirow}
\usepackage{multicol}

\theoremstyle{plain}
\newtheorem{theorem}{Theorem}
\newtheorem{proposition}{Proposition}
\newtheorem{lemma}{Lemma}
\newtheorem{corollary}{Corollary}
\newtheorem{assumption}{Assumption}
\theoremstyle{definition}

\newtheorem{remark}{Remark}

\usepackage{pgfplots}
\graphicspath{{media/}{media/pgfs/pgf_images/}}

\usepackage{graphicx}
\usepackage{subfigure}
\usepackage{makecell}
\usepackage{enumitem}

\usepackage{glossaries}
\usepackage{acronym}

\newacronym{smc}{SMC}{sequential Monte Carlo}
\newacronym{dpsmc}{DPSMC}{diffusion path sequential Monte Carlo}
\newacronym{tsi}{TSI}{target score identity}
\newacronym{dsi}{DSI}{denoising score identity}
\newacronym{msi}{MSI}{mixed score identity}
\newacronym{scv}{SCV}{scalar control-variate schedule}
\newacronym{mcv}{MCV}{matrix control-variate schedule}
\newacronym{ula}{ULA}{unadjusted Langevin algorithm}
\newacronym{mala}{MALA}{Metropolis-adjusted Langevin algorithm}
\newacronym{ald}{ALD}{annealed Langevin dynamics}
\newacronym{dald}{DALD}{diffusion-annealed Langevin dynamics}
\newacronym{dalmc}{DALMC}{diffusion-annealed Langevin Monte Carlo}
\newacronym{kl}{KL}{Kullback-Leibler}
\newacronym{ks}{KS}{Kolmogorov-Smirnov}
\newacronym{snr}{SNR}{signal-to-noise ratio}
\newacronym{ais}{AIS}{annealed importance sampling}
\newacronym{ll}{LL}{Langevin-within-Langevin}
\newacronym{si}{SI}{stochastic interpolant}
\newacronym{em}{EM}{Euler-Maruyama}
\newacronym{ei}{EI}{exponential integrator}
\newacronym{ou}{OU}{Ornstein–Uhlenbeck}
\newacronym{pt}{PT}{parallel tempering}
\newacronym{dll}{DLL}{duality of log-concavity}

\usepackage{wrapfig}
\usepackage[titletoc, title]{appendix}
\usepackage{titletoc}

\newcommand\midrulefill{%
  \leavevmode\leaders\hrule height \dimexpr0.5ex+0.2pt\relax
      depth  \dimexpr0.2pt-0.5ex\relax
  \hfill\kern0pt}

\usepackage[textwidth=1.8cm, textsize=scriptsize]{todonotes}

\newcommand{\fparagraph}[1]{\noindent\textbf{#1}}

\DeclareMathOperator*{\argmin}{argmin}

\title{Diffusion Path Samplers via Sequential Monte Carlo}

\newcommand*{\tabres}[3][]{\( #1{#2 {\scriptstyle \pm #3}} \)}
\newcommand*{\scisign}[1]{\ifx-#1\text{-}\else\ifx+#1\text{+}\else#1\fi\fi}
\newcommand*{\sci}[2]{#1\mathrm{e}\scisign#2}

\newcommand{\hlmath}[1]{{\fboxsep0pt\colorbox{YellowGreen!40}{#1}}}

\author{
    \bf James Matthew Young\textsuperscript{1} \quad Paula Cordero-Encinar\textsuperscript{1}\\[0.2em]
    \bf Sebastian Reich\textsuperscript{2} \quad Andrew B. Duncan\textsuperscript{1} \quad O. Deniz Akyildiz\textsuperscript{1}\\[0.3em]
    \textsuperscript{1}Department of Mathematics, Imperial College London\\
    \textsuperscript{2}Institut für Mathematik, Universität Potsdam\\[0.3em]
    \texttt{\{matthew.young20, paula.cordero-encinar22\}@imperial.ac.uk,}\\
    \texttt{sereich@uni-potsdam.de, \{a.duncan, deniz.akyildiz\}@imperial.ac.uk}
}

\begin{document}

\maketitle

\begin{abstract}
  We develop diffusion-based samplers for target distributions known up to a normalising constant. To this end, we rely on the well-known diffusion path that smoothly interpolates between a simple base distribution and the target, popularised by diffusion models. We tackle the score estimation problem by developing an efficient sequential Monte Carlo sampler that evolves auxiliary variables from conditional distributions along the path, providing principled score and density estimates for time-varying distributions. To control the variance of score estimates, we further propose practical control variate schedules that incur minimal overhead. We adapt this general framework to paths induced by the Ornstein–Uhlenbeck (OU) time-reversal process, stochastic interpolants, and diffusion annealed Langevin dynamics, outlining their trade-offs. Finally, we provide theoretical guarantees and empirically demonstrate the effectiveness of our method on several synthetic and real-world datasets. Our code is available at: \href{https://github.com/akyildiz-group/dpsmc}{https://github.com/akyildiz-group/dpsmc}.
\end{abstract}

\input{sections/intro}
\input{sections/background}
\input{sections/dpsmc}
\input{sections/theory}
\input{sections/experiments}
\input{sections/discussion}

\newpage

\section*{Acknowledgements}
J.M.Y. is supported by the Roth Scholarship provided by the Department of Mathematics, Imperial College London. P.C.E. is funded by EPSRC through the Centre for Doctoral Training in Modern Statistics and Statistical Machine Learning (StatML), grant no. EP/S023151/1. S.R.'s work is partially funded by the Deutsche Forschungsgemeinschaft (DFG) under Project-ID 318763901 – SFB1294. Lastly, J.M.Y. would like to thank Tim Wang for helpful discussions.

\bibliographystyle{plainnat}
\bibliography{ref}

\newpage
\appendix

\section*{Organisation of the appendix}
We present proofs for propositions stated in the main text as well as supplementary information like implementation details and additional experiments. The appendix is organised as follows:
\startcontents
\printcontents{}{1}{}

\setcounter{theorem}{0}
\setcounter{lemma}{0}
\setcounter{proposition}{0}
\setcounter{corollary}{0}
\setcounter{remark}{0}

\renewcommand{\thetheorem}{A.\arabic{theorem}}
\renewcommand{\thelemma}{A.\arabic{lemma}}
\renewcommand{\theproposition}{A.\arabic{proposition}}
\renewcommand{\thecorollary}{A.\arabic{corollary}}
\renewcommand{\theremark}{A.\arabic{remark}}

\input{appendix/a_additional_background}
\input{appendix/b_smc}
\input{appendix/c_score_identities}
\input{appendix/d_diffusion_path_design}
\input{appendix/e_implementation}
\input{appendix/f_theoretical_results}
\input{appendix/g_experiments}
\input{appendix/h_additional_experiments}


\end{document}

%% file: sections/intro.tex
\section{Introduction}
Sampling from a target distribution $\pi \propto \exp({-V_\pi})$ known up to a normalising constant is a fundamental problem in many fields of science. It underpins Bayesian inference \citep{gelman1995bayesian}, machine learning \citep{bishop2006pattern}, statistical physics \citep{newman1999monte}, finance \citep{glasserman2004monte}, computational biology \citep{wilkinson2018stochastic}, and many others. Despite this, sampling directly from complex targets remains a notoriously difficult task. A standard technique is to introduce an auxiliary path of distributions that bridges a simple base distribution and the target \citep{neal2001annealed, delmoral2006sequential}. This allows the sampler to gradually adapt to the complexity of the target distribution. A popular idea in this line of work is to consider the \textit{geometric annealing path} \citep{gelman1998simulating, neal2001annealed}. While convenient, such paths can distort relative mode weights, leading to torpid mixing or \textit{mass teleportation} symptoms \citep{woodard2009sufficient,tawn2020weight, davies2023contributions,mate2023learning,phillips2024particle}, and degrade log-Sobolev and related functional inequality constants \citep{chehab2025provable,cordero-encinar2025nonasymptotic}. These issues collectively hinder the efficiency of sampling algorithms along such paths.

In this work, we focus on a powerful alternative, the \textit{diffusion path}, which has recently gained significant attention in generative modelling \citep{sohl-dickstein2015deep,ho2020denoising,song2021scorebased}. As opposed to the geometric path, the diffusion path is well behaved: under mild assumptions, it admits uniformly bounded functional inequalities and finite action \citep{cordero-encinar2025nonasymptotic}. Their mathematical regularity, combined with the celebrated performance of diffusion models, make the diffusion path an appealing choice for sampling and the centre of a number of recent contributions (see Section~\ref{sec:related_works}). We expand this area by developing a novel, efficient, and highly-parallelisable sampler on diffusion paths. Specifically:

\textbf{(C1)} We develop a framework for score estimation, building on \gls*{smc} samplers \citep{delmoral2006sequential} that sample auxiliary variables sequentially, leading to highly parallelisable samplers.

\textbf{(C2)} We introduce novel control variate schedules for minimising the variance of score estimates, describing their optimum and providing practical estimates that take on minimal overhead. 

\textbf{(C3)} We provide a new bound on the score estimation error, extending the results of \citet{he2024zerothorder}, and establish convergence guarantees for our algorithms.

\textbf{(C4)} We demonstrate the effectiveness of our sampler applied to different diffusion dynamics across synthetic and real-world benchmarks, including the task of estimating normalising constants.

\noindent\textbf{Notation.} We denote by $\lVert\cdot\rVert$ the Euclidean norm for vectors and $\mathrm{B}(m,R)$ a ball centred at $m$ with radius $R$. We write by $\mathcal{N}(x; m, \mathbf{C})$ the density of a Gaussian random variable with mean $m$ and covariance matrix $\mathbf{C}$ evaluated at $x$. We let $W_2$ denote the Wasserstein-2 distance between probability measures \citep{villani2008optimal}: $
    W_2^2(\mu, \nu) = \inf_{\gamma\in\Gamma(\mu, \nu)} \int \lVert x - y\rVert^2 \mathrm{d}\gamma(x,y), 
$ where $\Gamma(\mu, \nu)$ is the set of couplings between $\mu$ and $\nu$. We denote by $\mathrm{KL}(\mu\;\|\;\nu)$ the \gls*{kl} divergence between probability measures $\mu$ and $\nu$. Finally, we write the total variance as $\mathrm{Var}_p[f(X)] := \mathrm{Tr}(\mathrm{Cov}_p[f(X)])$.

%% file: sections/background.tex
\section{Technical Background}

Recent advances in diffusion models \citep{sohl-dickstein2015deep,ho2020denoising,song2021scorebased} have demonstrated that paths based on diffusion processes scale well in high-dimensional and multimodal settings. These paths, which we refer to as \textit{diffusion paths}, can be defined as a continuous-time family of distributions $(\mu_t)_{t\in[0,1]}$ interpolating between a base and target distribution \citep{chehab2024practical,chehab2025provable,cordero-encinar2025nonasymptotic}, i.e. $\mu_0=\nu$ and $\mu_1=\pi$, written as:
\begin{align}
    \mu_t(x) = \tfrac{1}{\sqrt{1 - \lambda_t}^{d}}\nu\Bigl(\tfrac{x}{\sqrt{1 - \lambda_t}}\Bigr)\ast\tfrac{1}{\sqrt{\lambda_t}^{d}}\pi\Bigl(\tfrac{x}{\sqrt{\lambda_t}}\Bigr),\label{eq:diffusion_path_marginal}
\end{align}
for $t\in[0, 1]$ and a monotonically increasing schedule $\lambda_t \in [0, 1]$ with $\lambda_1 = 1$. This marginal coincides with the law of $I_t = \sqrt{1 - \lambda_t} Z + \sqrt{\lambda_t} X$, where $X \sim \pi$ and $Z \sim \nu$ independent---a spatially linear stochastic interpolant \citep{albergo2025stochastic}.
Throughout, we will assume a Gaussian base distribution $\nu = \mathcal{N}(0, \sigma^2\mathbf{I})$, for some $\sigma > 0$, which further makes the connection with diffusion models clear.

We may instantiate the path in \eqref{eq:diffusion_path_marginal} with two stochastic differential equations (SDEs): (i) a diffusion process that matches the path marginals exactly, and (ii) \gls*{dald} \citep{cordero-encinar2025nonasymptotic}, which approximates the path with an error that can be controlled by a parameter $\epsilon$. 

\textbf{Diffusion path SDE.} The diffusion path marginals in \eqref{eq:diffusion_path_marginal} are precisely induced by the SDE
\begin{equation}
    \mathrm{d}X_t = \left[\tfrac{\dot{\lambda}_t}{2\lambda_t}X_t + \left(\sigma^2\tfrac{\dot{\lambda}_t}{2\lambda_t}+\gamma_t\right)\nabla\log\mu_t(X_t)\right]\mathrm{d}t + \sqrt{2\gamma_t}\mathrm{d}W_t,\label{eq:general_diffusion_path_sde}
\end{equation}
for $t\in[0,1]$, $\gamma_t\geq0$, and $X_0\sim \mu_0$ (see Appendix \ref{app:subsec:stochastic_interpolants_dynamics}). The time-reversal of the \gls*{ou} process can then be recovered by choosing $\gamma_t=\sigma^2{\dot{\lambda}_t}/({2\lambda_t})$ and $\lambda_t = \exp(-2T(1-t))>0$, where $T>0$, yielding $\dot{\lambda}_t/\lambda_t=2T$. However, the \gls*{ou} process requires an infinite time horizon to reach its invariant Gaussian measure, and consequently $T$ has to be chosen large so $\mu_0 \approx \nu$. In fact, the SDE in \eqref{eq:general_diffusion_path_sde} also accommodates schedules that reach both endpoints in finite time \citep{albergo2025stochastic}, i.e. $\lambda_0 = 0$. The drift in \eqref{eq:general_diffusion_path_sde} appears singular at $t=0$ unlike OU-type exponential schedules with constant $\dot{\lambda}_t/\lambda_t$, but it can be alleviated by breaking down the score and choosing a suitable schedule as we show later.

\glsreset{dald}
\textbf{Diffusion-annealed Langevin dynamics.} The SDE in \eqref{eq:general_diffusion_path_sde} adopts a convenient form due to the Gaussian base distribution. More generally, for arbitrary base distributions, we can consider \gls*{dald} \citep{cordero-encinar2025nonasymptotic} which runs the SDE
\begin{align}\label{eq:diffusion_annealed_langevin_dynamics}
    \mathrm{d}X_t = \tfrac{1}{\epsilon}\nabla\log\mu_{t}(X_t)\mathrm{d}t + \sqrt{\tfrac{2}{\epsilon}}\mathrm{d}W_t,
\end{align}
where $t \in [0, 1]$, $X_0 \sim \nu$, and $\epsilon\in (0,1]$. Here, $\epsilon$ controls how closely the SDE's true marginals track the reference path $(\mu_t)_{t\in[0,1]}$, i.e. a smaller $\epsilon$ results in dynamics that remain closer to the instantaneous target distributions. Both \eqref{eq:general_diffusion_path_sde} and \eqref{eq:diffusion_annealed_langevin_dynamics} have practical trade-offs which we discuss later.

\textbf{Score formulae.} Central to the implementation of these processes is access to score functions $(\nabla\log\mu_t)_{t\in[0,1]}$. The path's convolutional structure admits two expressions for the score, known as the \gls*{dsi} \citep{vincent2011connection} and \gls*{tsi} \citep{de2024target}, respectively
\begin{align}\label{eq:initial_score_identities}
    \nabla\log\mu_t(x) &= \mathbb{E}_{\varrho_{t,x}}\left[\tfrac{\sqrt{\lambda_t}Y - x}{\sigma^2(1 - \lambda_t)}\right] = \mathbb{E}_{\varrho_{t,x}}\left[\tfrac{
{-}\nabla V_\pi(Y)}{\sqrt{\lambda_t}}\right],
\end{align}
where $V_\pi = -\log\pi$ is the target potential and $\varrho_{t,x}$ is the so-called posterior (conditional) distribution of the auxiliary (clean) variable $Y$ given (the noisy value) $x$:
\begin{equation}\label{eq:posterior_distribution}
    \varrho_{t,x}(y) \propto \nu\left(\tfrac{x - \sqrt{\lambda_t}y}{\sqrt{1 - \lambda_t}}\right)\pi(y).
\end{equation}
While the score is generally intractable, it can be learned by a neural network \citep{akhound-sadegh2024iterated, vargas2024transport, richter2024improved, chen2025sequential, zhang2022path, noble2025learned} or estimated with Monte Carlo \citep{huang2024reverse, chen2024diffusive, grenioux2024stochastic, saremi2024chain, he2024zerothorder, cordero-encinar2025sampling}. We follow the latter approach in this work and provide a discussion on related neural samplers in Appendix \ref{app:subsec:neural_samples_related_work}.

\subsection{Related Work}\label{sec:related_works}

\noindent\textbf{Diffusion-based samplers using Monte Carlo.} \citet{huang2024reverse} follow the time-reversed OU process and use importance sampling and Langevin MCMC to estimate the score. They initialise samples using a \gls*{ll} procedure by running \gls*{ula} to target $\mu_0$ using Monte Carlo score estimates which are themselves obtained via \gls*{ula}. \citet{grenioux2024stochastic} builds upon this work by introducing a stochastic localisation framework with flexible denoising schedules and a discretisation scheme based on the \gls*{snr} of the observation process. They start sampling at a later time, a hyperparameter, when both the path marginal and posterior are approximately log-concave. The reliance of these methods on MCMC, however, necessitates equilibration at every time step, creating a sequential bottleneck that cannot be overcome by simply increasing sample size. In an alternative line of work, \citet{he2024zerothorder} propose rejection sampling for score estimation, only making zero-order queries to the target density. While effective in low dimensions, the complexity of rejection sampling grows exponentially with the dimension, even for log-concave targets. Notably, the methods discussed all incur errors coming from their score estimation or SDE discretisation. Introducing a wrapper around reverse-diffusion based samplers, \citet{wu2025reverse} use the auxiliary variables from the posterior to estimate the path marginal density. From which, they propose a nested \gls*{smc} sampler, RDSMC, where samples can be resampled to better adhere to the diffusion path marginals. 
In contrast, we employ an \gls*{smc} sampler \textit{only} for the auxiliary samples. While fundamentally different, our scheme provides unbiased estimates for the path marginal density, allowing it to be nested inside RDSMC.

\noindent\textbf{Score identities.} Some works have found convex combinations of the score identities to be beneficial towards the variance of loss objectives. The \gls*{msi} \citep{de2024target, phillips2024particle, he2025training} is one example, assigning weights $1 - \lambda_t$ and $\lambda_t$ to \gls*{dsi} and \gls*{tsi}, respectively. This choice is intuitive as the identities become unstable at opposite ends. A piecewise schedule is similarly justified \citep{cordero-encinar2025sampling}. This leads us to examine more general schedules, namely those minimising the variance of score estimates. Concurrent to our work, \citet{kahouli2025control} and \citet{ko2025latent} propose a similar control variate score identity in the reverse diffusion context, opting for position dependent schedules \citep{kahouli2025control} or learning the schedule itself \citep{ko2025latent}. We minimise the variance in expectation and consider matrix coefficients.

%% file: sections/dpsmc.tex
\section{Diffusion Path Sequential Monte Carlo}
\glsreset{ll}
\subsection{Discretisation of Diffusion Path Dynamics}
We describe here two schemes to numerically simulate the diffusion path dynamics in \eqref{eq:general_diffusion_path_sde} and \eqref{eq:diffusion_annealed_langevin_dynamics}.

\textbf{Initialisation.} As the path may start at some $\lambda_0>0$, we consider a \gls*{ll} initialisation \citep{grenioux2024stochastic}. We discretise the time interval as $0 = t_0=\dots=t_{N_\text{LL}}<\ldots<t_K=1$, freezing the time index for $N_{\text{LL}}$ steps, akin to doing \gls*{ll}. The remaining time steps are discretised uniformly, i.e. $t_k={(k-N_{\text{LL}})}/{(K-N_{\text{LL}})}$, leading to a sequence of distributions $(\mu_{t_k})_{k=0}^{K}$. To start, samples are drawn from the base distribution and undergo \gls*{ula} updates with step size $h_{\text{LL}}$ for $k<N_{\text{LL}}$.

For $k\geq N_{\text{LL}}$, we discretise the SDEs in \eqref{eq:general_diffusion_path_sde} or \eqref{eq:diffusion_annealed_langevin_dynamics} with step size $h=1/(K-N_{\text{LL}})$ as follows.

\textbf{Discretisation of diffusion path SDE \eqref{eq:general_diffusion_path_sde}.} We use the OU schedule and the \gls*{ei} \citep{durmus2015quantitative}, yielding the update rule
\begin{align}\label{eq:ou_discretisation}
    X_{k+1} = e^{hT} X_k + 2\sigma^2(e^{hT}-1)\nabla\log\mu_{t_k}(X_k) + \sigma\sqrt{e^{2hT}-1}\xi_k,
\end{align}
where $\xi_k \sim \mathcal{N}(0, \mathbf{I})$. We also consider a different schedule but postpone our discussion to Section~\ref{subsec:path_design_choices}.

\textbf{Discretisation of \gls*{dald} \eqref{eq:diffusion_annealed_langevin_dynamics}.} We use an \gls*{em} discretisation, resulting in the update
\begin{align}\label{eq:ald_discretisation}
X_{k+1} = X_k + \tfrac{h}{\epsilon}\nabla\log\mu_{t_k}(X_k) + \sqrt{\tfrac{2h}{\epsilon}}\xi_k.
\end{align}

We denote these updates by a generic Markov kernel $\mathsf{K}_{k\smash{+1}}^{\text{sde}}$. Note the score function can be written as
\begin{align}
    \nabla\log\mu_{t_k}(x) &= \mathbb{E}_{\varrho_{t_k,x}}[\varphi_{k,x}(Y)], \label{eq:score_test_function}
\end{align}
for some vector-valued test function $\varphi_{k,x}: \mathbb{R}^d \to\mathbb{R}^d$ (choices of which can be judiciously made as we discuss in Section~\ref{subsec:score_estimation}) and $\varrho_{t_k,x}$, the posterior distribution defined in \eqref{eq:posterior_distribution}. The time-varying expectation in \eqref{eq:score_test_function} naturally suggests an \gls*{smc} scheme for score estimation, which we describe next.

\begin{wrapfigure}[28]{r}{0.52\textwidth}
    \begin{minipage}{0.52\textwidth}
{
\vspace{-2.4em}
\begin{algorithm}[H]
    \caption{DPSMC (Generic)}
    \label{alg:diffusion_path_sampling_smc}
    \begin{algorithmic}[1]
        \INPUT base variance $\sigma^2$, initial aux. distribution $q_{0}$, particle count $N$, disc. step count $K$, LL step count $N_\text{LL}$, LL step size $h_{\text{LL}}$, test function $\varphi_{k,x}$, ESS threshold $\kappa$, and SDE forward kernel $\mathsf{K}^{\text{sde}}_k$
        \item[] \midrulefill\ {\small{Initialisation}}\ \midrulefill
        \STATE $X_0 \sim \mathcal{N}(0, \sigma^2\mathbf{I})$, $Y^{i}_0 \sim q_{0}$ for all $i\in[N]$
        \STATE $W^{i}_0 = \tilde{\rho}_0(Y_0^i)/q_{0}(Y_0^i)$ for all $i\in[N]$
        \STATE $w^{i}_{0} = {W_0^i}/{\sum_{j=1}^{N} W_0^j}$ for all $i\in[N]$
        \STATE  $S_{0}^N=\sum_{i=1}^{N}{w^{i}_{0} \varphi_{0,X_{0}}({Y}_{0}^{i})}$

        \item[] \midrulefill\ {\small{Main Loop}}\ \midrulefill
        
        \FOR{$k=1,\ldots,K-1$}
                \STATE \algorithmicif\ $k< N_{\text{LL}}$ \algorithmicthen\\
                \STATE \quad $X_k \sim \mathcal{N}(X_{k-1}+h_{\text{LL}}S_{k-1}^N,2h_{\text{LL}}\mathbf{I})${\small\COMMENT{LL}}\ \STATE \algorithmicelse\ $X_k \sim \mathsf{K}_{k}^{\mathrm{sde}}(\cdot \mid X_{k-1}, S_{k-1}^N)$ {\small\COMMENT{SDE step}}
                
                \FOR{$i=1,\ldots,N$}
                    \STATE $\bar{Y}^{i}_{k} \sim \mathsf{K}_{k}(\cdot \mid Y_{k-1}^{i}, X_{k})$
                    \STATE $W_k^i = W_{k-1}^i \frac{\tilde{\rho}_{k}(\bar{Y}_{k}^{i})\mathsf{L}_{k-1}(\bar{Y}_{k}^{i}, Y_{k-1}^{i})}{\tilde{\rho}_{k-1}(Y_{k-1}^{i}) \mathsf{K}_{k}(Y_{k-1}^{i}, \bar{Y}_{k}^{i})}$
                \ENDFOR
                \STATE $w^{i}_{k} = {W_k^i}/{\sum_{j=1}^{N} W_k^j}$ for all $i\in[N]$
                \STATE  $S_{k}^N=\sum_{i=1}^{N}{w^{i}_{k} \varphi_{k,X_{k}}(\bar{Y}_{k}^{i})}$ {\small\COMMENT{Score estimation}}
                
                \STATE \algorithmicif\ $\widehat{\mathrm{ESS}}(\{w^{i}_k\}_{i=1}^{N}) < \kappa N$ \algorithmicthen\\
                    \STATE \quad $\{Y^{i}_{k}\}_{i=1}^{N} \sim \sum_{j=1}^{N} {w_k^{j}\delta_{\bar{Y}_{k}^{j}}}$ {\small\COMMENT{Resampling}}
                    \STATE \quad $W_k^i = 1$ for all $i \in[N]$
                \STATE \algorithmicelse\ $Y_k^i = \bar{Y}_{k}^i$ for all $i\in[N]$
            \ENDFOR
            \STATE $X_K \sim \mathsf{K}_{K}^{\mathrm{sde}}(\cdot \mid X_{K-1}, S_{K-1}^N)$ {\small\COMMENT{SDE step}}
        \OUTPUT approximate sample from the target $X_K$
    \end{algorithmic}
\end{algorithm}
}
\end{minipage}
\end{wrapfigure}

\subsection{SMC Samplers for Score Estimation}\label{subsec:general_smc_posterior}
We define a sequence of target distributions $({\rho}_k)_{k=0}^{K}$, given by ${\rho}_k := \varrho_{t_k, X_k}$, where $X_k$ is our sample at time $t_k$ following our discretised dynamics in \eqref{eq:ou_discretisation} or \eqref{eq:ald_discretisation}. We then define the unnormalised target and proposal distributions on the (extended) path space $\mathbb{R}^{d(k+1)}$
\begin{align*}
\tilde{\rho}_{0:k}(y_{0:k}) &:= \tilde{\rho}_k(y_k) {\textstyle\prod}_{p=1}^{k}{\mathsf{L}_{p-1}(y_{p}, y_{p-1})},\\
    {q}_{0:k}(y_{0:k}) &:= q_{0}(y_0) {\textstyle\prod}_{p=1}^{k}{\mathsf{K}_{p}(y_{p-1}, y_{p})},
\end{align*}
where $(\mathsf{K}_{p})_{p=1}^{k}$ and $(\mathsf{L}_{p})_{p=0}^{k-1}$ are a sequence of forward and backward Markov kernels,
\begin{equation*}
    \tilde{\rho}_k(y)=\tfrac{1}{\sqrt{1-\lambda_{t_k}}^d}\nu\biggl(\tfrac{X_k-\sqrt{\lambda_{t_k}}y}{\sqrt{1-\lambda_{t_k}}}\biggr)\pi(y),
\end{equation*}
and $q_0$ is some initial auxiliary distribution. It can be readily checked that the $y_k$-marginal of (normalised) $\rho_{0:k}$ is indeed ${\rho}_k$, justifying the choice of the extended target. The resulting importance weights can be written in a recursive fashion
\begin{align*}
    W_{k}(y_{0:k}) &:= \frac{\tilde{\rho}_{0:k}(y_{0:k})}{q_{0:k}(y_{0:k})}\\
    &= W_{k-1}(y_{0:k-1}) G_{k}(y_{k-1}, y_k),
\end{align*}
with the potential $G_k: \mathbb{R}^d \times \mathbb{R}^d \to \mathbb{R}_+$:
\begin{align*}
    G_{k}(y_{k-1}, y_k) = \frac{\tilde{\rho}_{k}(y_{k})\mathsf{L}_{k-1}(y_{k}, y_{k-1})}{\tilde{\rho}_{k-1}(y_{k-1}) \mathsf{K}_{k}(y_{k-1}, y_{k})}.
\end{align*}
We then have the following result through the Feynman-Kac formulae \citep{delmoral2004feynmankac,delmoral2013mean}.

\vspace{1em}
\begin{restatable}{proposition}{fkIdentity}\label{prop:fk_identity} 
    For any $k \geq 1$, let $Y_{0:k} \sim q_{0:k}$. If $\varphi_{k,x}$ forms a score identity as in \eqref{eq:score_test_function}, then we have
    \begin{align*}
        \nabla \log {\mu}_{t_k}(X_k) &= \frac{\mathbb{E}[\varphi_{k,X_k}(Y_{k})\prod_{p=0}^{k}{G_{p}(Y_{p-1}, Y_{p})}]}{\mathbb{E}[\prod_{p=0}^{k}{G_{p}(Y_{p-1}, Y_{p})}]},\quad \mu_{t_k}(X_k) = \mathbb{E}\left[{\textstyle\prod}_{p=0}^{k}{G_{p}(Y_{p-1}, Y_{p})}\right]
    \end{align*}
    where $G_0(y_{-1}, y_0) \coloneqq \tilde{\rho}_0(y_0)/q_0(y_0)$ and the expectations are taken with respect to the path measure induced by the forward kernels $(\mathsf{K}_p)_{p=1}^k$ and initial distribution $q_{0}$.
\end{restatable}
See Appendix~\ref{app:proof:fk_identity_proof} for the proof. Proposition~\ref{prop:fk_identity} therefore expresses the score and marginal density in terms of expectations under the forward path measure, which can be estimated directly with an \gls*{smc} sampler. Notably, this provides unbiased estimates for the marginal density which makes the scheme compatible with pseudo-marginal or \textit{exact approximation} methods \citep{andrieu2009pseudomarginal, andrieu2010particle, wu2025reverse} which we demonstrate later. Next, we outline our sampling method based on this construction.

\subsection{The DPSMC Sampler}

The previous framework leads us to introduce our \gls*{dpsmc} sampler, summarised in Algorithm \ref{alg:diffusion_path_sampling_smc}. Quite simply, we have a discretised diffusion process whose dynamics are intertwined with an SMC sampler. The state of the process defines the SMC sampler's intermediate target, and the SMC sampler provides a score estimate for updating the process.

We walk through a single iteration of our algorithm to demonstrate. At the end of time $t_{k-1}$, we have three things: (i) our sample $X_{k-1}$, a result of integrating the SDE for $k-1$ steps with estimated scores; (ii) a weighted ensemble $\{(W_{k-1}^i, Y_{k-1}^{i})\}_{i=1}^N$, an output of our \gls*{smc} sampler targeting $\rho_{k-1}$; and (iii) a score estimate $S_{k-1}^N$, formed by a Monte Carlo approximation of \eqref{eq:score_test_function} using the weighted ensemble and test function $\varphi_{k-1,X_{k-1}}$. At time $t_k$, we first update our sample to $X_{k}$ with an integration step involving the current score estimate. Then, we run one iteration of the \gls*{smc} sampler to target $\rho_k$. This involves propagating each $Y_{k-1}^i$ by a forward kernel $\mathsf{K}_{k}(\cdot\mid Y_{k-1}^{i}, X_k)$, updating their weights, and resampling, resulting in a new weighted ensemble $\{(W_k^i, Y_k^i)\}_{i=1}^{N}$ targeting $\rho_k$. Finally, an updated score estimate $S_k^N$ is produced using the new ensemble and test function $\varphi_{k,X_k}$. The procedure is repeated until time $t_{K-1}$, after which we perform one final integration step to obtain $X_K$. Note that we perform an \gls*{ll} initialisation for the first $N_{\text{LL}}$ steps if $\lambda_0>0$, whereby the only change is the underlying diffusion dynamics being discretised is replaced with an overdamped Langevin diffusion.

In the next sections, we carefully investigate design choices to make our generic scheme practical.

\subsection{Score Identities}\label{subsec:score_estimation}

We now look at stable choices for test functions $\varphi_{k,x}$ that form score identities as in \eqref{eq:score_test_function}. Building upon \gls*{dsi} and \gls*{tsi}, as in \eqref{eq:initial_score_identities}, we introduce a class of test functions
\begin{align}
    {\phi}^{\mathbf{A}}_{t,x}(y) := \tfrac{1}{\sqrt{1 - \lambda_t}} \mathbf{A} & \nabla\log\nu\left(\tfrac{x - \sqrt{\lambda_t}y}{\sqrt{1 - \lambda_t}}\right) +\tfrac{1}{\sqrt{\lambda_t}} (\mathbf{I} - \mathbf{A}) \nabla\log\pi(y), \label{eq:general_cv_score_test_function}
\end{align}
for $t\in[0,1]$ and $\mathbf{A}\in\mathbb{R}^{d\times d}$, which indeed satisfy $\nabla\log\mu_t(x) = \mathbb{E}_{\varrho_{t,x}}[\phi_{t,x}^{\mathbf{A}}(Y)]$ (see Proposition~\ref{app:prop:convex_combination_matrix} for a proof). In practice, we set $\varphi_{k,x} = {\phi}_{t_k,x}^{\mathbf{A}}$ at time step $k$. We are interested in the variance behaviour of these test functions under the true marginals, parameterised by matrix schedules $\mathbf{A}_t \in \mathbb{R}^{d\times d}$. We refer to these as \textit{control variate} (CV) schedules, given we can arrive at the same test function in \eqref{eq:general_cv_score_test_function} by introducing a control variate matrix (see Remark~\ref{app:remark:matrix_control_variate_identity}). We first consider the case when $\mathbf{A}_t = \alpha_t\mathbf{I}$, i.e. a scalar weighting of \gls*{dsi} and \gls*{tsi}. The following result characterises the optimal choice of $\alpha_t$ for minimising the \textit{expected} variance of the corresponding score estimator.

\begin{restatable}{proposition}{propOptimalScalarCVSchedule}\label{prop:optimal_scalar_cv_schedule}
    The scalar CV schedule $\alpha_t^\ast\in \mathbb{R}$ minimising the expected variance is given by
    \begin{align*}
        \alpha_t^\ast
        &= {\argmin_{\alpha\in\mathbb{R}}}\ {\mathbb{E}_{X\sim \mu_t}[\mathrm{Var}_{Y\sim \varrho_{t,X}}[{\phi}^{\alpha\mathbf{I}}_{t,X}(Y)]]}
        = \frac{\frac{1}{\lambda_t}\mathrm{Var}_\pi[\nabla\log\pi(X)]}{\frac{1}{1-\lambda_t}\mathrm{Var}_\nu[\nabla\log\nu(X)] + \frac{1}{\lambda_t}\mathrm{Var}_\pi[\nabla\log\pi(X)]}.
    \end{align*}
\end{restatable}
See Appendix~\ref{app:proof:optimal_scalar_cv_schedule} for the proof.
This expression decouples the dependence of time from statistics of the endpoint distributions. It aligns with our intuition of avoiding the instabilities at both ends, having $\alpha_t^\ast$ be non-increasing and going from one to zero. Unlike heuristic choices that start with DSI and end with TSI, this result shows explicitly when one is favoured, e.g. it is possible that $\alpha_t^\ast \approx 1$ for majority of the time when the target has much larger score variance than the base. When they coincide in score variance, we recover $\alpha_t^\ast = 1 - \lambda_t$, i.e. \gls*{msi}. For matrices, we have the next result.

\begin{restatable}{proposition}{propOptimalCVSchedule}\label{prop:optimal_matrix_cv_schedule}
    The matrix CV schedule $\mathbf{A}_t^\ast\in\mathbb{R}^{d\times d}$ minimising the expected variance is given by
    \begin{align*}
        \mathbf{A}_t^\ast &= {\argmin_{\mathbf{A}\in\mathbb{R}^{d\times d}}}\ \mathbb{E}_{X\sim \mu_t}[\mathrm{Var}_{Y\sim\varrho_{t,X}}[{\phi}_{t,X}^{\mathbf{A}}(Y)]] = \tfrac{1}{\lambda_t}\mathcal{I}_\pi\left(\tfrac{1}{1-\lambda_t}\mathcal{I}_\nu + \tfrac{1}{\lambda_t} \mathcal{I}_\pi\right)^{-1},
    \end{align*}
    where $\mathcal{I}_p :=\mathrm{Cov}_p[\nabla\log p(X)]=\mathbb{E}_p[\nabla\log p(X)\nabla\log p(X)^\top]$ denotes the score covariance of $p$.
\end{restatable}

See Appendix~\ref{app:proof:optimal_matrix_cv_schedule} for the proof. This optimal weighting accounts for the pairwise covariances of individual score components. Note the scalar schedule averages them out. We therefore expect a discrepancy between the scalar and matrix schedules when the base and target have a covariance mismatch, e.g. for extremely anisotropic targets. An example is given in Figure \ref{fig:mse_of_score_estimators}, documenting the MSE evolution of different score identities. Ideally, one would choose base densities with a similar covariance structure as the target to maximise the effectiveness of scalar schedules.

In practice, we need estimates of varying degrees for the target score covariance, i.e. either its trace, diagonal, or the full matrix. By Lemma \ref{lemma:double_expectation_results} in Appendix, we can write it as 
\begin{align*}
    \mathcal{I}_\pi = \mathbb{E}_{X\sim\mu_t}[\mathbb{E}_{Y\sim\varrho_{t,X}}[\nabla\log\pi(Y)\nabla\log\varrho_{t,X}(Y)^\top]],
\end{align*}
which is amenable to Monte Carlo estimation. We use the empirical distribution formed by our samples as a proxy for $\mu_t$ and our auxiliary variables for the inner expectation. This incurs no additional energy evaluations, as they can be reused for the score estimate.

{
\parfillskip=0pt
\parskip=0pt
\par}
\begin{wrapfigure}[15]{r}{0.38\textwidth}
    \hfill\resizebox{0.4\textwidth}{!}{\input{media/pgfs/cv.pgf}}\hfill
    \vspace{-1.6em}
    \caption{\textit{Average MSE of different score identities using exact samples from diffusion path marginals and auxiliary variables from their conditional distribution.} MCVSI and MMCVSI consistently remain under the DSI/TSI lower envelope.}
    \label{fig:mse_of_score_estimators}
\end{wrapfigure}
Following \citet{kahouli2025control}, we call the identities formed by our schedules as \textit{marginalised control-variate score identities} (MCVSI) and refer to the matrix case as MMCVSI. Note that we can minimise the estimator's variance for each sample $x$ but do so in expectation. This admits an interpretable form, indicating trade-offs between scalar, diagonal, and full-matrix control-variate coefficients, while practically being more stable, as the target score covariance is a static unknown and all auxiliary variables can be used for estimation. Further, while our score estimates via SMC are biased due to self-normalisation, we empirically find our schedules consistently improve the score estimates. This can be justified, since the variance of the unbiased estimator coincides with the asymptotic variance of the self-normalised estimator when the proposal matches the target. See Appendix \ref{app:subsec:score_identities_ablation} for ablations on score identities.

\subsection{Diffusion Path Design}\label{subsec:path_design_choices}

Our diffusion path is parameterised by the base Gaussian variance $\sigma^2$ and the schedule $\lambda_t$. It is not surprising that poor choices for these parameters can render the sampler inefficient.

\fparagraph{Base Gaussian variance.} If $\sigma^2$ is too small, then $\mu_t(x) \approx \pi\left(x/\sqrt{\lambda_t}\right)/\sqrt{\lambda_t}^d$, causing the mode structure to be determined within a few time steps and the remainder of the path spent dilating the modes. If it is too large, the path marginals become indistinguishable from a Gaussian for most of the path. To get around this, we consult the \gls*{snr} of our stochastic process. We define
\begin{align*}
    \mathrm{SNR}(t) := \tfrac{\mathbb{E}[\lVert \sqrt{\lambda_t} X\rVert^2]}{\mathbb{E}[\lVert \sqrt{1 - \lambda_t} Z\rVert^2]} = \tfrac{\mathbb{E}_\pi[\lVert X\rVert^2]\lambda_t}{\sigma^2d(1 - \lambda_t)},
\end{align*}
using the power or second moment of the target $\mathbb{E}_\pi[\lVert X\rVert^2]$ as opposed to its scalar variance \citep{grenioux2024stochastic}. Here, choosing $\sigma^2 \propto \mathbb{E}_\pi[\lVert X\rVert^2] / d$ yields an SNR independent of the target. We set $\sigma^2 = \mathbb{E}_\pi[\lVert X\rVert^2] / d$, as it is the most natural choice, keeping the second moment fixed across the path marginals.

\fparagraph{Diffusion path schedule.} We seek a $\lambda_t$ schedule for \gls*{dald} and an alternative to the OU schedule for the exact dynamics in \eqref{eq:general_diffusion_path_sde}. Intuitively, the \gls*{dald} schedule should spend more time where the path is rapidly changing and less where it does not. Let $\mu = (\mu_t)_{t\in[0,1]}$ denote the curve of probability measures corresponding to the diffusion path. We can consider the (Wasserstein) action functional $\mathcal{A}$ as a measure of how rapidly the path marginals evolve throughout the curve. This quantity upper bounds the KL divergence between the target and final marginal distribution of \gls*{dald} \citep{guo2025provable,cordero-encinar2025nonasymptotic} (see Appendix~\ref{sec_app:action_analysis} for details). While choosing $\lambda_t$ to minimise it is a natural choice, the action is generally intractable, so we instead minimise the upper bound functional
\begin{align}
    \mathcal{A}(\mu) := \int_0^T \lvert\dot{\mu}\rvert_t^2 \mathrm{d}t \leq \int_0^T\Bigl[ \tfrac{\sigma^2d}{4(1 - \lambda_t)} +\tfrac{\mathbb{E}_\pi[\lVert X\rVert^2]}{4\lambda_t}\Bigr] \dot{\lambda}_t^2\mathrm{d}t,
\end{align}
where $\lvert\dot{\mu}\rvert_t = \lim_{\delta\rightarrow 0} {W_2(\mu_{t+\delta}, \mu_t)}/{\lvert\delta\rvert}$ is the metric derivative of $\mu$ at time $t$. The minimiser is available analytically when $\sigma^2 = \mathbb{E}_\pi[\lVert X\rVert^2]/d$ and is given by $$\lambda_t^\ast = \sin^2\left(\tfrac{\pi t}{2}\right) = 1 - \cos^2\left(\tfrac{\pi t}{2}\right),$$ 
which is the well-known cosine schedule (see Appendix~\ref{subsec_app:action_minimising_schedule} for the proof). This schedule eases in and out, allowing more time at the start for samples to localise in modes and at the end for details to be gradually refined. This also provides a handy finite upper bound on the action.

We previously hinted that a non-exponential schedule with $\lambda_0=0$ can be chosen to admit the exact SDE in \eqref{eq:general_diffusion_path_sde}, despite the seemingly singular drift coefficient. With \gls*{dsi}, the drift can be rewritten as
\begin{equation*}
    b_t(X_t) := \tfrac{\dot{\lambda}_t}{2\lambda_t}[X_t+\sigma^2\nabla\log\mu_t(X_t)] = \tfrac{\dot{\lambda}_t}{2\sqrt{\lambda_t}}\mathbb{E}_{\varrho_{t,X_t}}\left[\tfrac{Y-\sqrt{\lambda}_t X_t}{1-\lambda_t}\right],
\end{equation*}
and now we simply require $\dot{\lambda}_t/\sqrt{\lambda_t}$ remaining finite at $t=0$. While there are several options, we set $\gamma_t=\sigma^2$ and $\lambda_t=\sin^2(\frac{\pi t}{2})$ to match the diffusion coefficient of the OU process while following the curve at the same speed as \gls*{dald}. This way, we can more closely examine how the dynamics differ under similar conditions. Notably, this allows us to start exactly from a Gaussian while preserving exact dynamics, rendering it free from the critical hyperparameters $T$ and $\epsilon$ that alter the diffusion process. We refer to this as following a \gls*{si} dynamics \citep{albergo2025stochastic}.

\subsection{Algorithmic Choices}\label{subsec:algorithmic_choices}

We briefly detail our implementation of the \gls*{dpsmc} algorithm in practice.

\begin{wrapfigure}[15]{r}{0.22\textwidth}
    \vspace{-1.5em}
    \hfill\resizebox{0.20\textwidth}{!}{\input{media/pgfs/diffusion_path_overview.pgf}}\hfill
    \vspace{-0.25em}
    \caption{\textit{Evolution of path marginal and posterior.} In red is a sample trajectory.}
    \label{fig:diffusion_path}
\end{wrapfigure}

\fparagraph{Initialisation.} A seemingly pathological quality shared by diffusion-based samplers is the fact that $\varrho_{0,x}=\pi$ (see Figure \ref{fig:diffusion_path}). To see this, one can think of $\varrho_{t,x}$ as the distribution of where a (noisy) sample $x$ at time $t$ will end up at time $t=1$. This also implies $\varrho_{1,x}=\delta_x$. Notably, under some conditions, there exists a time window where both path marginals and posteriors are log-concave or approximately so \citep{grenioux2024stochastic}. Hence, starting at a later time not at a Gaussian, e.g. in this window, may be more of a practical choice rather than a disadvantage. As such, we search for $\lambda_{0}$ in $[0,1]$ uniformly and perform an \gls*{ll} initialisation for OU dynamics. Our other schemes with $\lambda_0=0$ sidestep the need for \gls*{ll} but require sampling from $\rho_{0,x}$. Indeed, in the next section, we show that initial posterior samples can be rather poor and still produce good score estimates. We choose $q_{0} = \mathcal{N}(0,\tilde{\sigma}^2\min(1,(1-\lambda_{0})/\lambda_{0})\mathbf{I})$ to initialise the auxiliary variables. We follow \citet{grenioux2024stochastic} in setting $\tilde{\sigma}^2 = R^2d + \tau^2$, where we assume $\pi=\mathcal{N}(0,\tau^2\mathbf{I})\ast\eta$ for $\eta$ compactly supported on $\mathrm{B}(\mathbb{E}_\pi[X], R^2d)$. We remark that \gls*{ll} for $\lambda_0>0$ incurs an initialisation bias that appears later in the overall algorithm error.

\fparagraph{Propagation.} We propagate auxiliary variables according to a \gls*{mala} kernel $\mathsf{K}_k^{\text{MALA}}$ with an adaptive step size, adjusted to target an acceptance ratio. We make the standard choice of backward kernels being reversed kernels in relation to the forward kernels \citep[Section 3.3.2.3]{delmoral2006sequential}, i.e. $\mathsf{L}_{k-1}(y_{k}, y_{k-1})=\rho_k(y_{k-1})\mathsf{K}_k(y_{k-1}, y_k)/\rho_{k}(y_k)$, simplifying the weight update (see Appendix \ref{subsec_app:mala_weight_update}).
We also use stratified resampling, a low-variance resampling method, and trigger it adaptively with threshold $\kappa=0.5$. With our samples and auxiliary variables, we estimate the matrix schedule in Proposition \ref{prop:optimal_matrix_cv_schedule} and hence the scores with the test function in \eqref{eq:general_cv_score_test_function}.

Our practical implementation of \gls*{dpsmc} is given in Algorithm \ref{alg:dpsmc_interacting} in Appendix. While our samples are interacting in step size adaptation and target score covariance estimation, one could first estimate these underlying quantities and consider the non-interacting version of our sampler with these settings.

%% file: media/pgfs/diffusion_path_overview.pgf
\begingroup%
\makeatletter%
\begin{pgfpicture}%
\pgfpathrectangle{\pgfpointorigin}{\pgfqpoint{2.066766in}{3.156889in}}%
\pgfusepath{use as bounding box, clip}%
\begin{pgfscope}%
\pgfsetbuttcap%
\pgfsetmiterjoin%
\definecolor{currentfill}{rgb}{1.000000,1.000000,1.000000}%
\pgfsetfillcolor{currentfill}%
\pgfsetlinewidth{0.000000pt}%
\definecolor{currentstroke}{rgb}{1.000000,1.000000,1.000000}%
\pgfsetstrokecolor{currentstroke}%
\pgfsetdash{}{0pt}%
\pgfpathmoveto{\pgfqpoint{0.000000in}{0.000000in}}%
\pgfpathlineto{\pgfqpoint{2.066766in}{0.000000in}}%
\pgfpathlineto{\pgfqpoint{2.066766in}{3.156889in}}%
\pgfpathlineto{\pgfqpoint{0.000000in}{3.156889in}}%
\pgfpathlineto{\pgfqpoint{0.000000in}{0.000000in}}%
\pgfpathclose%
\pgfusepath{fill}%
\end{pgfscope}%
\begin{pgfscope}%
\pgfsetbuttcap%
\pgfsetmiterjoin%
\definecolor{currentfill}{rgb}{1.000000,1.000000,1.000000}%
\pgfsetfillcolor{currentfill}%
\pgfsetlinewidth{0.000000pt}%
\definecolor{currentstroke}{rgb}{0.000000,0.000000,0.000000}%
\pgfsetstrokecolor{currentstroke}%
\pgfsetstrokeopacity{0.000000}%
\pgfsetdash{}{0pt}%
\pgfpathmoveto{\pgfqpoint{0.050000in}{0.050000in}}%
\pgfpathlineto{\pgfqpoint{0.921731in}{0.050000in}}%
\pgfpathlineto{\pgfqpoint{0.921731in}{2.874408in}}%
\pgfpathlineto{\pgfqpoint{0.050000in}{2.874408in}}%
\pgfpathlineto{\pgfqpoint{0.050000in}{0.050000in}}%
\pgfpathclose%
\pgfusepath{fill}%
\end{pgfscope}%
\begin{pgfscope}%
\pgfpathrectangle{\pgfqpoint{0.050000in}{0.050000in}}{\pgfqpoint{0.871731in}{2.824408in}}%
\pgfusepath{clip}%
\pgfsys@transformshift{0.050000in}{0.050000in}%
\pgftext[left,bottom]{\includegraphics[interpolate=true,width=0.880000in,height=2.830000in]{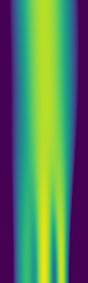}}%
\end{pgfscope}%
\begin{pgfscope}%
\definecolor{textcolor}{rgb}{0.000000,0.000000,0.000000}%
\pgfsetstrokecolor{textcolor}%
\pgfsetfillcolor{textcolor}%
\pgftext[x=0.485866in,y=3.013297in,,]{\color{textcolor}\rmfamily\fontsize{15.000000}{18.000000}\selectfont \(\displaystyle \mu_t\)}%
\end{pgfscope}%
\begin{pgfscope}%
\pgfpathrectangle{\pgfqpoint{0.050000in}{0.050000in}}{\pgfqpoint{0.871731in}{2.824408in}}%
\pgfusepath{clip}%
\pgfsetrectcap%
\pgfsetroundjoin%
\pgfsetlinewidth{1.505625pt}%
\definecolor{currentstroke}{rgb}{1.000000,0.000000,0.000000}%
\pgfsetstrokecolor{currentstroke}%
\pgfsetdash{}{0pt}%
\pgfpathmoveto{\pgfqpoint{0.445253in}{2.871650in}}%
\pgfpathlineto{\pgfqpoint{0.400570in}{2.866134in}}%
\pgfpathlineto{\pgfqpoint{0.388068in}{2.860617in}}%
\pgfpathlineto{\pgfqpoint{0.383332in}{2.855101in}}%
\pgfpathlineto{\pgfqpoint{0.420743in}{2.844068in}}%
\pgfpathlineto{\pgfqpoint{0.420683in}{2.838552in}}%
\pgfpathlineto{\pgfqpoint{0.323147in}{2.833035in}}%
\pgfpathlineto{\pgfqpoint{0.330113in}{2.827519in}}%
\pgfpathlineto{\pgfqpoint{0.295307in}{2.822002in}}%
\pgfpathlineto{\pgfqpoint{0.287560in}{2.816486in}}%
\pgfpathlineto{\pgfqpoint{0.315318in}{2.810970in}}%
\pgfpathlineto{\pgfqpoint{0.312221in}{2.805453in}}%
\pgfpathlineto{\pgfqpoint{0.354867in}{2.799937in}}%
\pgfpathlineto{\pgfqpoint{0.373953in}{2.794420in}}%
\pgfpathlineto{\pgfqpoint{0.372833in}{2.788904in}}%
\pgfpathlineto{\pgfqpoint{0.309440in}{2.783387in}}%
\pgfpathlineto{\pgfqpoint{0.259338in}{2.777871in}}%
\pgfpathlineto{\pgfqpoint{0.257471in}{2.772355in}}%
\pgfpathlineto{\pgfqpoint{0.272783in}{2.766838in}}%
\pgfpathlineto{\pgfqpoint{0.310865in}{2.761322in}}%
\pgfpathlineto{\pgfqpoint{0.252404in}{2.755805in}}%
\pgfpathlineto{\pgfqpoint{0.237216in}{2.750289in}}%
\pgfpathlineto{\pgfqpoint{0.297496in}{2.744773in}}%
\pgfpathlineto{\pgfqpoint{0.326272in}{2.739256in}}%
\pgfpathlineto{\pgfqpoint{0.301926in}{2.733740in}}%
\pgfpathlineto{\pgfqpoint{0.316179in}{2.728223in}}%
\pgfpathlineto{\pgfqpoint{0.274314in}{2.722707in}}%
\pgfpathlineto{\pgfqpoint{0.334745in}{2.717190in}}%
\pgfpathlineto{\pgfqpoint{0.375379in}{2.711674in}}%
\pgfpathlineto{\pgfqpoint{0.407414in}{2.706158in}}%
\pgfpathlineto{\pgfqpoint{0.360817in}{2.695125in}}%
\pgfpathlineto{\pgfqpoint{0.354375in}{2.689608in}}%
\pgfpathlineto{\pgfqpoint{0.322645in}{2.684092in}}%
\pgfpathlineto{\pgfqpoint{0.320729in}{2.678575in}}%
\pgfpathlineto{\pgfqpoint{0.335046in}{2.673059in}}%
\pgfpathlineto{\pgfqpoint{0.320302in}{2.667543in}}%
\pgfpathlineto{\pgfqpoint{0.357805in}{2.662026in}}%
\pgfpathlineto{\pgfqpoint{0.365142in}{2.656510in}}%
\pgfpathlineto{\pgfqpoint{0.359115in}{2.650993in}}%
\pgfpathlineto{\pgfqpoint{0.398022in}{2.645477in}}%
\pgfpathlineto{\pgfqpoint{0.364369in}{2.639960in}}%
\pgfpathlineto{\pgfqpoint{0.367707in}{2.634444in}}%
\pgfpathlineto{\pgfqpoint{0.381736in}{2.628928in}}%
\pgfpathlineto{\pgfqpoint{0.461683in}{2.623411in}}%
\pgfpathlineto{\pgfqpoint{0.392829in}{2.617895in}}%
\pgfpathlineto{\pgfqpoint{0.341133in}{2.612378in}}%
\pgfpathlineto{\pgfqpoint{0.317691in}{2.606862in}}%
\pgfpathlineto{\pgfqpoint{0.339457in}{2.601346in}}%
\pgfpathlineto{\pgfqpoint{0.374606in}{2.595829in}}%
\pgfpathlineto{\pgfqpoint{0.372498in}{2.590313in}}%
\pgfpathlineto{\pgfqpoint{0.354595in}{2.584796in}}%
\pgfpathlineto{\pgfqpoint{0.353085in}{2.579280in}}%
\pgfpathlineto{\pgfqpoint{0.345594in}{2.573763in}}%
\pgfpathlineto{\pgfqpoint{0.345555in}{2.568247in}}%
\pgfpathlineto{\pgfqpoint{0.323345in}{2.562731in}}%
\pgfpathlineto{\pgfqpoint{0.343343in}{2.557214in}}%
\pgfpathlineto{\pgfqpoint{0.393852in}{2.551698in}}%
\pgfpathlineto{\pgfqpoint{0.487777in}{2.546181in}}%
\pgfpathlineto{\pgfqpoint{0.462800in}{2.540665in}}%
\pgfpathlineto{\pgfqpoint{0.444198in}{2.535148in}}%
\pgfpathlineto{\pgfqpoint{0.482319in}{2.529632in}}%
\pgfpathlineto{\pgfqpoint{0.444119in}{2.518599in}}%
\pgfpathlineto{\pgfqpoint{0.481162in}{2.513083in}}%
\pgfpathlineto{\pgfqpoint{0.505000in}{2.507566in}}%
\pgfpathlineto{\pgfqpoint{0.546406in}{2.502050in}}%
\pgfpathlineto{\pgfqpoint{0.521366in}{2.496534in}}%
\pgfpathlineto{\pgfqpoint{0.594476in}{2.491017in}}%
\pgfpathlineto{\pgfqpoint{0.577505in}{2.485501in}}%
\pgfpathlineto{\pgfqpoint{0.602998in}{2.479984in}}%
\pgfpathlineto{\pgfqpoint{0.548455in}{2.474468in}}%
\pgfpathlineto{\pgfqpoint{0.535287in}{2.468951in}}%
\pgfpathlineto{\pgfqpoint{0.513239in}{2.463435in}}%
\pgfpathlineto{\pgfqpoint{0.485408in}{2.457919in}}%
\pgfpathlineto{\pgfqpoint{0.476135in}{2.452402in}}%
\pgfpathlineto{\pgfqpoint{0.512693in}{2.446886in}}%
\pgfpathlineto{\pgfqpoint{0.514552in}{2.441369in}}%
\pgfpathlineto{\pgfqpoint{0.542794in}{2.435853in}}%
\pgfpathlineto{\pgfqpoint{0.540331in}{2.430336in}}%
\pgfpathlineto{\pgfqpoint{0.550863in}{2.424820in}}%
\pgfpathlineto{\pgfqpoint{0.545813in}{2.419304in}}%
\pgfpathlineto{\pgfqpoint{0.553974in}{2.413787in}}%
\pgfpathlineto{\pgfqpoint{0.565243in}{2.408271in}}%
\pgfpathlineto{\pgfqpoint{0.553938in}{2.402754in}}%
\pgfpathlineto{\pgfqpoint{0.589708in}{2.397238in}}%
\pgfpathlineto{\pgfqpoint{0.597407in}{2.391721in}}%
\pgfpathlineto{\pgfqpoint{0.596500in}{2.386205in}}%
\pgfpathlineto{\pgfqpoint{0.681577in}{2.380689in}}%
\pgfpathlineto{\pgfqpoint{0.622459in}{2.375172in}}%
\pgfpathlineto{\pgfqpoint{0.609272in}{2.369656in}}%
\pgfpathlineto{\pgfqpoint{0.588369in}{2.364139in}}%
\pgfpathlineto{\pgfqpoint{0.583558in}{2.358623in}}%
\pgfpathlineto{\pgfqpoint{0.541056in}{2.353107in}}%
\pgfpathlineto{\pgfqpoint{0.607152in}{2.347590in}}%
\pgfpathlineto{\pgfqpoint{0.632349in}{2.342074in}}%
\pgfpathlineto{\pgfqpoint{0.589138in}{2.336557in}}%
\pgfpathlineto{\pgfqpoint{0.617693in}{2.331041in}}%
\pgfpathlineto{\pgfqpoint{0.621268in}{2.325524in}}%
\pgfpathlineto{\pgfqpoint{0.555495in}{2.320008in}}%
\pgfpathlineto{\pgfqpoint{0.522080in}{2.314492in}}%
\pgfpathlineto{\pgfqpoint{0.501858in}{2.308975in}}%
\pgfpathlineto{\pgfqpoint{0.541624in}{2.303459in}}%
\pgfpathlineto{\pgfqpoint{0.523524in}{2.297942in}}%
\pgfpathlineto{\pgfqpoint{0.530451in}{2.292426in}}%
\pgfpathlineto{\pgfqpoint{0.541470in}{2.286909in}}%
\pgfpathlineto{\pgfqpoint{0.558225in}{2.281393in}}%
\pgfpathlineto{\pgfqpoint{0.593422in}{2.275877in}}%
\pgfpathlineto{\pgfqpoint{0.552849in}{2.270360in}}%
\pgfpathlineto{\pgfqpoint{0.567638in}{2.264844in}}%
\pgfpathlineto{\pgfqpoint{0.544522in}{2.253811in}}%
\pgfpathlineto{\pgfqpoint{0.483541in}{2.248294in}}%
\pgfpathlineto{\pgfqpoint{0.480385in}{2.242778in}}%
\pgfpathlineto{\pgfqpoint{0.482786in}{2.237262in}}%
\pgfpathlineto{\pgfqpoint{0.514612in}{2.231745in}}%
\pgfpathlineto{\pgfqpoint{0.535550in}{2.226229in}}%
\pgfpathlineto{\pgfqpoint{0.515925in}{2.220712in}}%
\pgfpathlineto{\pgfqpoint{0.502304in}{2.215196in}}%
\pgfpathlineto{\pgfqpoint{0.484213in}{2.209680in}}%
\pgfpathlineto{\pgfqpoint{0.479890in}{2.204163in}}%
\pgfpathlineto{\pgfqpoint{0.511913in}{2.198647in}}%
\pgfpathlineto{\pgfqpoint{0.528902in}{2.193130in}}%
\pgfpathlineto{\pgfqpoint{0.529898in}{2.187614in}}%
\pgfpathlineto{\pgfqpoint{0.593154in}{2.176581in}}%
\pgfpathlineto{\pgfqpoint{0.605348in}{2.171065in}}%
\pgfpathlineto{\pgfqpoint{0.602052in}{2.165548in}}%
\pgfpathlineto{\pgfqpoint{0.580441in}{2.160032in}}%
\pgfpathlineto{\pgfqpoint{0.572574in}{2.154515in}}%
\pgfpathlineto{\pgfqpoint{0.582990in}{2.148999in}}%
\pgfpathlineto{\pgfqpoint{0.628395in}{2.137966in}}%
\pgfpathlineto{\pgfqpoint{0.668187in}{2.132450in}}%
\pgfpathlineto{\pgfqpoint{0.636079in}{2.126933in}}%
\pgfpathlineto{\pgfqpoint{0.642140in}{2.121417in}}%
\pgfpathlineto{\pgfqpoint{0.650570in}{2.115900in}}%
\pgfpathlineto{\pgfqpoint{0.617549in}{2.110384in}}%
\pgfpathlineto{\pgfqpoint{0.612872in}{2.104867in}}%
\pgfpathlineto{\pgfqpoint{0.599085in}{2.099351in}}%
\pgfpathlineto{\pgfqpoint{0.599211in}{2.093835in}}%
\pgfpathlineto{\pgfqpoint{0.576469in}{2.088318in}}%
\pgfpathlineto{\pgfqpoint{0.599907in}{2.082802in}}%
\pgfpathlineto{\pgfqpoint{0.565758in}{2.077285in}}%
\pgfpathlineto{\pgfqpoint{0.602779in}{2.071769in}}%
\pgfpathlineto{\pgfqpoint{0.582107in}{2.066253in}}%
\pgfpathlineto{\pgfqpoint{0.626522in}{2.060736in}}%
\pgfpathlineto{\pgfqpoint{0.619821in}{2.055220in}}%
\pgfpathlineto{\pgfqpoint{0.642020in}{2.049703in}}%
\pgfpathlineto{\pgfqpoint{0.649315in}{2.044187in}}%
\pgfpathlineto{\pgfqpoint{0.626913in}{2.038670in}}%
\pgfpathlineto{\pgfqpoint{0.632820in}{2.033154in}}%
\pgfpathlineto{\pgfqpoint{0.616722in}{2.027638in}}%
\pgfpathlineto{\pgfqpoint{0.582937in}{2.022121in}}%
\pgfpathlineto{\pgfqpoint{0.576237in}{2.016605in}}%
\pgfpathlineto{\pgfqpoint{0.577819in}{2.011088in}}%
\pgfpathlineto{\pgfqpoint{0.574573in}{2.005572in}}%
\pgfpathlineto{\pgfqpoint{0.596803in}{1.994539in}}%
\pgfpathlineto{\pgfqpoint{0.626053in}{1.989023in}}%
\pgfpathlineto{\pgfqpoint{0.604220in}{1.983506in}}%
\pgfpathlineto{\pgfqpoint{0.593380in}{1.977990in}}%
\pgfpathlineto{\pgfqpoint{0.587600in}{1.972473in}}%
\pgfpathlineto{\pgfqpoint{0.627044in}{1.966957in}}%
\pgfpathlineto{\pgfqpoint{0.622695in}{1.955924in}}%
\pgfpathlineto{\pgfqpoint{0.659287in}{1.950408in}}%
\pgfpathlineto{\pgfqpoint{0.654983in}{1.944891in}}%
\pgfpathlineto{\pgfqpoint{0.671779in}{1.939375in}}%
\pgfpathlineto{\pgfqpoint{0.660348in}{1.933858in}}%
\pgfpathlineto{\pgfqpoint{0.664205in}{1.928342in}}%
\pgfpathlineto{\pgfqpoint{0.619293in}{1.917309in}}%
\pgfpathlineto{\pgfqpoint{0.638478in}{1.911793in}}%
\pgfpathlineto{\pgfqpoint{0.603723in}{1.906276in}}%
\pgfpathlineto{\pgfqpoint{0.604602in}{1.900760in}}%
\pgfpathlineto{\pgfqpoint{0.595292in}{1.895243in}}%
\pgfpathlineto{\pgfqpoint{0.542521in}{1.889727in}}%
\pgfpathlineto{\pgfqpoint{0.538443in}{1.884211in}}%
\pgfpathlineto{\pgfqpoint{0.561932in}{1.878694in}}%
\pgfpathlineto{\pgfqpoint{0.553442in}{1.873178in}}%
\pgfpathlineto{\pgfqpoint{0.604203in}{1.867661in}}%
\pgfpathlineto{\pgfqpoint{0.605724in}{1.862145in}}%
\pgfpathlineto{\pgfqpoint{0.572080in}{1.856628in}}%
\pgfpathlineto{\pgfqpoint{0.576476in}{1.851112in}}%
\pgfpathlineto{\pgfqpoint{0.595199in}{1.840079in}}%
\pgfpathlineto{\pgfqpoint{0.624755in}{1.834563in}}%
\pgfpathlineto{\pgfqpoint{0.566087in}{1.829046in}}%
\pgfpathlineto{\pgfqpoint{0.565187in}{1.823530in}}%
\pgfpathlineto{\pgfqpoint{0.534873in}{1.818014in}}%
\pgfpathlineto{\pgfqpoint{0.528662in}{1.812497in}}%
\pgfpathlineto{\pgfqpoint{0.541093in}{1.806981in}}%
\pgfpathlineto{\pgfqpoint{0.549000in}{1.801464in}}%
\pgfpathlineto{\pgfqpoint{0.529261in}{1.795948in}}%
\pgfpathlineto{\pgfqpoint{0.556674in}{1.779399in}}%
\pgfpathlineto{\pgfqpoint{0.574928in}{1.773882in}}%
\pgfpathlineto{\pgfqpoint{0.567478in}{1.768366in}}%
\pgfpathlineto{\pgfqpoint{0.545184in}{1.762849in}}%
\pgfpathlineto{\pgfqpoint{0.547160in}{1.757333in}}%
\pgfpathlineto{\pgfqpoint{0.570944in}{1.751816in}}%
\pgfpathlineto{\pgfqpoint{0.565926in}{1.746300in}}%
\pgfpathlineto{\pgfqpoint{0.530492in}{1.740784in}}%
\pgfpathlineto{\pgfqpoint{0.543778in}{1.735267in}}%
\pgfpathlineto{\pgfqpoint{0.536537in}{1.729751in}}%
\pgfpathlineto{\pgfqpoint{0.564855in}{1.724234in}}%
\pgfpathlineto{\pgfqpoint{0.558554in}{1.718718in}}%
\pgfpathlineto{\pgfqpoint{0.565985in}{1.713201in}}%
\pgfpathlineto{\pgfqpoint{0.590925in}{1.707685in}}%
\pgfpathlineto{\pgfqpoint{0.556942in}{1.702169in}}%
\pgfpathlineto{\pgfqpoint{0.565226in}{1.696652in}}%
\pgfpathlineto{\pgfqpoint{0.579250in}{1.691136in}}%
\pgfpathlineto{\pgfqpoint{0.544467in}{1.685619in}}%
\pgfpathlineto{\pgfqpoint{0.537584in}{1.680103in}}%
\pgfpathlineto{\pgfqpoint{0.580265in}{1.674587in}}%
\pgfpathlineto{\pgfqpoint{0.566985in}{1.669070in}}%
\pgfpathlineto{\pgfqpoint{0.565708in}{1.663554in}}%
\pgfpathlineto{\pgfqpoint{0.575847in}{1.658037in}}%
\pgfpathlineto{\pgfqpoint{0.565561in}{1.652521in}}%
\pgfpathlineto{\pgfqpoint{0.567128in}{1.647004in}}%
\pgfpathlineto{\pgfqpoint{0.542457in}{1.641488in}}%
\pgfpathlineto{\pgfqpoint{0.537102in}{1.635972in}}%
\pgfpathlineto{\pgfqpoint{0.568859in}{1.630455in}}%
\pgfpathlineto{\pgfqpoint{0.574046in}{1.624939in}}%
\pgfpathlineto{\pgfqpoint{0.574626in}{1.619422in}}%
\pgfpathlineto{\pgfqpoint{0.567997in}{1.613906in}}%
\pgfpathlineto{\pgfqpoint{0.576509in}{1.608389in}}%
\pgfpathlineto{\pgfqpoint{0.592870in}{1.602873in}}%
\pgfpathlineto{\pgfqpoint{0.547611in}{1.597357in}}%
\pgfpathlineto{\pgfqpoint{0.544293in}{1.591840in}}%
\pgfpathlineto{\pgfqpoint{0.542953in}{1.586324in}}%
\pgfpathlineto{\pgfqpoint{0.525638in}{1.580807in}}%
\pgfpathlineto{\pgfqpoint{0.555518in}{1.575291in}}%
\pgfpathlineto{\pgfqpoint{0.537880in}{1.569774in}}%
\pgfpathlineto{\pgfqpoint{0.579272in}{1.558742in}}%
\pgfpathlineto{\pgfqpoint{0.501711in}{1.547709in}}%
\pgfpathlineto{\pgfqpoint{0.496201in}{1.542192in}}%
\pgfpathlineto{\pgfqpoint{0.502185in}{1.536676in}}%
\pgfpathlineto{\pgfqpoint{0.486811in}{1.531160in}}%
\pgfpathlineto{\pgfqpoint{0.495066in}{1.525643in}}%
\pgfpathlineto{\pgfqpoint{0.507277in}{1.520127in}}%
\pgfpathlineto{\pgfqpoint{0.500318in}{1.514610in}}%
\pgfpathlineto{\pgfqpoint{0.508292in}{1.509094in}}%
\pgfpathlineto{\pgfqpoint{0.494773in}{1.503577in}}%
\pgfpathlineto{\pgfqpoint{0.500405in}{1.498061in}}%
\pgfpathlineto{\pgfqpoint{0.536362in}{1.492545in}}%
\pgfpathlineto{\pgfqpoint{0.506908in}{1.481512in}}%
\pgfpathlineto{\pgfqpoint{0.548714in}{1.470479in}}%
\pgfpathlineto{\pgfqpoint{0.548037in}{1.464962in}}%
\pgfpathlineto{\pgfqpoint{0.553946in}{1.459446in}}%
\pgfpathlineto{\pgfqpoint{0.540412in}{1.453930in}}%
\pgfpathlineto{\pgfqpoint{0.543952in}{1.448413in}}%
\pgfpathlineto{\pgfqpoint{0.529180in}{1.442897in}}%
\pgfpathlineto{\pgfqpoint{0.548885in}{1.437380in}}%
\pgfpathlineto{\pgfqpoint{0.501642in}{1.426347in}}%
\pgfpathlineto{\pgfqpoint{0.508807in}{1.420831in}}%
\pgfpathlineto{\pgfqpoint{0.513283in}{1.415315in}}%
\pgfpathlineto{\pgfqpoint{0.540957in}{1.409798in}}%
\pgfpathlineto{\pgfqpoint{0.541609in}{1.404282in}}%
\pgfpathlineto{\pgfqpoint{0.535563in}{1.398765in}}%
\pgfpathlineto{\pgfqpoint{0.535390in}{1.393249in}}%
\pgfpathlineto{\pgfqpoint{0.498603in}{1.387733in}}%
\pgfpathlineto{\pgfqpoint{0.563055in}{1.376700in}}%
\pgfpathlineto{\pgfqpoint{0.567075in}{1.371183in}}%
\pgfpathlineto{\pgfqpoint{0.565691in}{1.365667in}}%
\pgfpathlineto{\pgfqpoint{0.574983in}{1.360150in}}%
\pgfpathlineto{\pgfqpoint{0.581481in}{1.354634in}}%
\pgfpathlineto{\pgfqpoint{0.559183in}{1.349118in}}%
\pgfpathlineto{\pgfqpoint{0.568415in}{1.343601in}}%
\pgfpathlineto{\pgfqpoint{0.545907in}{1.338085in}}%
\pgfpathlineto{\pgfqpoint{0.532878in}{1.332568in}}%
\pgfpathlineto{\pgfqpoint{0.551968in}{1.321535in}}%
\pgfpathlineto{\pgfqpoint{0.568967in}{1.316019in}}%
\pgfpathlineto{\pgfqpoint{0.559443in}{1.310503in}}%
\pgfpathlineto{\pgfqpoint{0.552215in}{1.304986in}}%
\pgfpathlineto{\pgfqpoint{0.547314in}{1.299470in}}%
\pgfpathlineto{\pgfqpoint{0.553709in}{1.293953in}}%
\pgfpathlineto{\pgfqpoint{0.527914in}{1.282920in}}%
\pgfpathlineto{\pgfqpoint{0.541989in}{1.277404in}}%
\pgfpathlineto{\pgfqpoint{0.543707in}{1.271888in}}%
\pgfpathlineto{\pgfqpoint{0.533459in}{1.266371in}}%
\pgfpathlineto{\pgfqpoint{0.527206in}{1.260855in}}%
\pgfpathlineto{\pgfqpoint{0.533669in}{1.255338in}}%
\pgfpathlineto{\pgfqpoint{0.495433in}{1.249822in}}%
\pgfpathlineto{\pgfqpoint{0.513584in}{1.244306in}}%
\pgfpathlineto{\pgfqpoint{0.497112in}{1.238789in}}%
\pgfpathlineto{\pgfqpoint{0.523968in}{1.227756in}}%
\pgfpathlineto{\pgfqpoint{0.524161in}{1.222240in}}%
\pgfpathlineto{\pgfqpoint{0.528180in}{1.216723in}}%
\pgfpathlineto{\pgfqpoint{0.545358in}{1.211207in}}%
\pgfpathlineto{\pgfqpoint{0.523608in}{1.205691in}}%
\pgfpathlineto{\pgfqpoint{0.526783in}{1.200174in}}%
\pgfpathlineto{\pgfqpoint{0.518136in}{1.194658in}}%
\pgfpathlineto{\pgfqpoint{0.504865in}{1.189141in}}%
\pgfpathlineto{\pgfqpoint{0.505551in}{1.183625in}}%
\pgfpathlineto{\pgfqpoint{0.518026in}{1.178108in}}%
\pgfpathlineto{\pgfqpoint{0.502580in}{1.172592in}}%
\pgfpathlineto{\pgfqpoint{0.503151in}{1.167076in}}%
\pgfpathlineto{\pgfqpoint{0.515259in}{1.161559in}}%
\pgfpathlineto{\pgfqpoint{0.490843in}{1.156043in}}%
\pgfpathlineto{\pgfqpoint{0.487204in}{1.150526in}}%
\pgfpathlineto{\pgfqpoint{0.498277in}{1.145010in}}%
\pgfpathlineto{\pgfqpoint{0.475000in}{1.139494in}}%
\pgfpathlineto{\pgfqpoint{0.472883in}{1.133977in}}%
\pgfpathlineto{\pgfqpoint{0.456454in}{1.128461in}}%
\pgfpathlineto{\pgfqpoint{0.473733in}{1.122944in}}%
\pgfpathlineto{\pgfqpoint{0.456147in}{1.117428in}}%
\pgfpathlineto{\pgfqpoint{0.471559in}{1.111911in}}%
\pgfpathlineto{\pgfqpoint{0.455597in}{1.106395in}}%
\pgfpathlineto{\pgfqpoint{0.467359in}{1.100879in}}%
\pgfpathlineto{\pgfqpoint{0.463945in}{1.095362in}}%
\pgfpathlineto{\pgfqpoint{0.473139in}{1.089846in}}%
\pgfpathlineto{\pgfqpoint{0.494595in}{1.084329in}}%
\pgfpathlineto{\pgfqpoint{0.500795in}{1.078813in}}%
\pgfpathlineto{\pgfqpoint{0.513495in}{1.073296in}}%
\pgfpathlineto{\pgfqpoint{0.518444in}{1.067780in}}%
\pgfpathlineto{\pgfqpoint{0.503191in}{1.062264in}}%
\pgfpathlineto{\pgfqpoint{0.502153in}{1.056747in}}%
\pgfpathlineto{\pgfqpoint{0.509732in}{1.051231in}}%
\pgfpathlineto{\pgfqpoint{0.501975in}{1.045714in}}%
\pgfpathlineto{\pgfqpoint{0.502554in}{1.040198in}}%
\pgfpathlineto{\pgfqpoint{0.496961in}{1.034681in}}%
\pgfpathlineto{\pgfqpoint{0.497062in}{1.029165in}}%
\pgfpathlineto{\pgfqpoint{0.483652in}{1.023649in}}%
\pgfpathlineto{\pgfqpoint{0.487635in}{1.018132in}}%
\pgfpathlineto{\pgfqpoint{0.530831in}{1.012616in}}%
\pgfpathlineto{\pgfqpoint{0.534446in}{1.007099in}}%
\pgfpathlineto{\pgfqpoint{0.546674in}{1.001583in}}%
\pgfpathlineto{\pgfqpoint{0.534128in}{0.996067in}}%
\pgfpathlineto{\pgfqpoint{0.531438in}{0.990550in}}%
\pgfpathlineto{\pgfqpoint{0.503405in}{0.985034in}}%
\pgfpathlineto{\pgfqpoint{0.538890in}{0.974001in}}%
\pgfpathlineto{\pgfqpoint{0.538102in}{0.968484in}}%
\pgfpathlineto{\pgfqpoint{0.495380in}{0.962968in}}%
\pgfpathlineto{\pgfqpoint{0.489086in}{0.957452in}}%
\pgfpathlineto{\pgfqpoint{0.484670in}{0.951935in}}%
\pgfpathlineto{\pgfqpoint{0.483847in}{0.946419in}}%
\pgfpathlineto{\pgfqpoint{0.490351in}{0.940902in}}%
\pgfpathlineto{\pgfqpoint{0.488504in}{0.935386in}}%
\pgfpathlineto{\pgfqpoint{0.517720in}{0.929869in}}%
\pgfpathlineto{\pgfqpoint{0.494183in}{0.924353in}}%
\pgfpathlineto{\pgfqpoint{0.498491in}{0.918837in}}%
\pgfpathlineto{\pgfqpoint{0.521973in}{0.907804in}}%
\pgfpathlineto{\pgfqpoint{0.537666in}{0.902287in}}%
\pgfpathlineto{\pgfqpoint{0.530568in}{0.896771in}}%
\pgfpathlineto{\pgfqpoint{0.526485in}{0.891254in}}%
\pgfpathlineto{\pgfqpoint{0.539704in}{0.885738in}}%
\pgfpathlineto{\pgfqpoint{0.536832in}{0.880222in}}%
\pgfpathlineto{\pgfqpoint{0.542451in}{0.874705in}}%
\pgfpathlineto{\pgfqpoint{0.517200in}{0.863672in}}%
\pgfpathlineto{\pgfqpoint{0.517904in}{0.858156in}}%
\pgfpathlineto{\pgfqpoint{0.525222in}{0.852640in}}%
\pgfpathlineto{\pgfqpoint{0.510057in}{0.841607in}}%
\pgfpathlineto{\pgfqpoint{0.523457in}{0.836090in}}%
\pgfpathlineto{\pgfqpoint{0.522426in}{0.830574in}}%
\pgfpathlineto{\pgfqpoint{0.551097in}{0.825057in}}%
\pgfpathlineto{\pgfqpoint{0.519970in}{0.819541in}}%
\pgfpathlineto{\pgfqpoint{0.533381in}{0.814025in}}%
\pgfpathlineto{\pgfqpoint{0.518575in}{0.808508in}}%
\pgfpathlineto{\pgfqpoint{0.523538in}{0.802992in}}%
\pgfpathlineto{\pgfqpoint{0.517261in}{0.797475in}}%
\pgfpathlineto{\pgfqpoint{0.513122in}{0.791959in}}%
\pgfpathlineto{\pgfqpoint{0.512882in}{0.786442in}}%
\pgfpathlineto{\pgfqpoint{0.490009in}{0.780926in}}%
\pgfpathlineto{\pgfqpoint{0.483211in}{0.775410in}}%
\pgfpathlineto{\pgfqpoint{0.483928in}{0.769893in}}%
\pgfpathlineto{\pgfqpoint{0.507317in}{0.764377in}}%
\pgfpathlineto{\pgfqpoint{0.484927in}{0.758860in}}%
\pgfpathlineto{\pgfqpoint{0.490375in}{0.753344in}}%
\pgfpathlineto{\pgfqpoint{0.482256in}{0.747827in}}%
\pgfpathlineto{\pgfqpoint{0.506480in}{0.742311in}}%
\pgfpathlineto{\pgfqpoint{0.482437in}{0.731278in}}%
\pgfpathlineto{\pgfqpoint{0.492828in}{0.725762in}}%
\pgfpathlineto{\pgfqpoint{0.489709in}{0.720245in}}%
\pgfpathlineto{\pgfqpoint{0.483395in}{0.714729in}}%
\pgfpathlineto{\pgfqpoint{0.495505in}{0.709213in}}%
\pgfpathlineto{\pgfqpoint{0.523818in}{0.703696in}}%
\pgfpathlineto{\pgfqpoint{0.533413in}{0.698180in}}%
\pgfpathlineto{\pgfqpoint{0.526963in}{0.692663in}}%
\pgfpathlineto{\pgfqpoint{0.517270in}{0.687147in}}%
\pgfpathlineto{\pgfqpoint{0.539245in}{0.681630in}}%
\pgfpathlineto{\pgfqpoint{0.496428in}{0.676114in}}%
\pgfpathlineto{\pgfqpoint{0.498879in}{0.670598in}}%
\pgfpathlineto{\pgfqpoint{0.508208in}{0.665081in}}%
\pgfpathlineto{\pgfqpoint{0.521011in}{0.659565in}}%
\pgfpathlineto{\pgfqpoint{0.528719in}{0.654048in}}%
\pgfpathlineto{\pgfqpoint{0.501106in}{0.648532in}}%
\pgfpathlineto{\pgfqpoint{0.503794in}{0.643015in}}%
\pgfpathlineto{\pgfqpoint{0.503571in}{0.637499in}}%
\pgfpathlineto{\pgfqpoint{0.501069in}{0.631983in}}%
\pgfpathlineto{\pgfqpoint{0.509017in}{0.626466in}}%
\pgfpathlineto{\pgfqpoint{0.495776in}{0.620950in}}%
\pgfpathlineto{\pgfqpoint{0.500578in}{0.615433in}}%
\pgfpathlineto{\pgfqpoint{0.508002in}{0.609917in}}%
\pgfpathlineto{\pgfqpoint{0.495081in}{0.604400in}}%
\pgfpathlineto{\pgfqpoint{0.500348in}{0.598884in}}%
\pgfpathlineto{\pgfqpoint{0.503213in}{0.593368in}}%
\pgfpathlineto{\pgfqpoint{0.497472in}{0.587851in}}%
\pgfpathlineto{\pgfqpoint{0.524116in}{0.582335in}}%
\pgfpathlineto{\pgfqpoint{0.524488in}{0.576818in}}%
\pgfpathlineto{\pgfqpoint{0.529701in}{0.571302in}}%
\pgfpathlineto{\pgfqpoint{0.530908in}{0.565786in}}%
\pgfpathlineto{\pgfqpoint{0.519779in}{0.560269in}}%
\pgfpathlineto{\pgfqpoint{0.515005in}{0.554753in}}%
\pgfpathlineto{\pgfqpoint{0.521663in}{0.549236in}}%
\pgfpathlineto{\pgfqpoint{0.498283in}{0.538203in}}%
\pgfpathlineto{\pgfqpoint{0.491218in}{0.532687in}}%
\pgfpathlineto{\pgfqpoint{0.512621in}{0.527171in}}%
\pgfpathlineto{\pgfqpoint{0.486389in}{0.521654in}}%
\pgfpathlineto{\pgfqpoint{0.513454in}{0.516138in}}%
\pgfpathlineto{\pgfqpoint{0.515090in}{0.510621in}}%
\pgfpathlineto{\pgfqpoint{0.527360in}{0.505105in}}%
\pgfpathlineto{\pgfqpoint{0.534473in}{0.499588in}}%
\pgfpathlineto{\pgfqpoint{0.499665in}{0.494072in}}%
\pgfpathlineto{\pgfqpoint{0.520504in}{0.488556in}}%
\pgfpathlineto{\pgfqpoint{0.524256in}{0.483039in}}%
\pgfpathlineto{\pgfqpoint{0.514385in}{0.477523in}}%
\pgfpathlineto{\pgfqpoint{0.513794in}{0.472006in}}%
\pgfpathlineto{\pgfqpoint{0.504253in}{0.466490in}}%
\pgfpathlineto{\pgfqpoint{0.517318in}{0.460973in}}%
\pgfpathlineto{\pgfqpoint{0.503658in}{0.449941in}}%
\pgfpathlineto{\pgfqpoint{0.514276in}{0.444424in}}%
\pgfpathlineto{\pgfqpoint{0.514391in}{0.438908in}}%
\pgfpathlineto{\pgfqpoint{0.506586in}{0.433391in}}%
\pgfpathlineto{\pgfqpoint{0.498484in}{0.422359in}}%
\pgfpathlineto{\pgfqpoint{0.508839in}{0.416842in}}%
\pgfpathlineto{\pgfqpoint{0.499429in}{0.411326in}}%
\pgfpathlineto{\pgfqpoint{0.501156in}{0.405809in}}%
\pgfpathlineto{\pgfqpoint{0.493466in}{0.400293in}}%
\pgfpathlineto{\pgfqpoint{0.503849in}{0.394776in}}%
\pgfpathlineto{\pgfqpoint{0.488482in}{0.389260in}}%
\pgfpathlineto{\pgfqpoint{0.506050in}{0.378227in}}%
\pgfpathlineto{\pgfqpoint{0.508685in}{0.372711in}}%
\pgfpathlineto{\pgfqpoint{0.505431in}{0.367194in}}%
\pgfpathlineto{\pgfqpoint{0.499764in}{0.361678in}}%
\pgfpathlineto{\pgfqpoint{0.501225in}{0.356161in}}%
\pgfpathlineto{\pgfqpoint{0.487339in}{0.345129in}}%
\pgfpathlineto{\pgfqpoint{0.506674in}{0.339612in}}%
\pgfpathlineto{\pgfqpoint{0.507130in}{0.334096in}}%
\pgfpathlineto{\pgfqpoint{0.499528in}{0.328579in}}%
\pgfpathlineto{\pgfqpoint{0.525604in}{0.323063in}}%
\pgfpathlineto{\pgfqpoint{0.508447in}{0.317547in}}%
\pgfpathlineto{\pgfqpoint{0.496682in}{0.312030in}}%
\pgfpathlineto{\pgfqpoint{0.515783in}{0.306514in}}%
\pgfpathlineto{\pgfqpoint{0.505245in}{0.300997in}}%
\pgfpathlineto{\pgfqpoint{0.490757in}{0.295481in}}%
\pgfpathlineto{\pgfqpoint{0.501255in}{0.289964in}}%
\pgfpathlineto{\pgfqpoint{0.504850in}{0.284448in}}%
\pgfpathlineto{\pgfqpoint{0.492873in}{0.278932in}}%
\pgfpathlineto{\pgfqpoint{0.508383in}{0.273415in}}%
\pgfpathlineto{\pgfqpoint{0.494675in}{0.267899in}}%
\pgfpathlineto{\pgfqpoint{0.519778in}{0.262382in}}%
\pgfpathlineto{\pgfqpoint{0.499685in}{0.256866in}}%
\pgfpathlineto{\pgfqpoint{0.505777in}{0.245833in}}%
\pgfpathlineto{\pgfqpoint{0.495951in}{0.240317in}}%
\pgfpathlineto{\pgfqpoint{0.495355in}{0.234800in}}%
\pgfpathlineto{\pgfqpoint{0.490698in}{0.229284in}}%
\pgfpathlineto{\pgfqpoint{0.482531in}{0.223767in}}%
\pgfpathlineto{\pgfqpoint{0.499252in}{0.218251in}}%
\pgfpathlineto{\pgfqpoint{0.475959in}{0.212734in}}%
\pgfpathlineto{\pgfqpoint{0.500183in}{0.207218in}}%
\pgfpathlineto{\pgfqpoint{0.481751in}{0.201702in}}%
\pgfpathlineto{\pgfqpoint{0.493787in}{0.196185in}}%
\pgfpathlineto{\pgfqpoint{0.498908in}{0.190669in}}%
\pgfpathlineto{\pgfqpoint{0.494528in}{0.179636in}}%
\pgfpathlineto{\pgfqpoint{0.496864in}{0.174120in}}%
\pgfpathlineto{\pgfqpoint{0.492225in}{0.168603in}}%
\pgfpathlineto{\pgfqpoint{0.496852in}{0.163087in}}%
\pgfpathlineto{\pgfqpoint{0.498099in}{0.157570in}}%
\pgfpathlineto{\pgfqpoint{0.484024in}{0.152054in}}%
\pgfpathlineto{\pgfqpoint{0.499687in}{0.146537in}}%
\pgfpathlineto{\pgfqpoint{0.499543in}{0.135505in}}%
\pgfpathlineto{\pgfqpoint{0.483840in}{0.129988in}}%
\pgfpathlineto{\pgfqpoint{0.494352in}{0.124472in}}%
\pgfpathlineto{\pgfqpoint{0.484320in}{0.118955in}}%
\pgfpathlineto{\pgfqpoint{0.494782in}{0.113439in}}%
\pgfpathlineto{\pgfqpoint{0.493556in}{0.107922in}}%
\pgfpathlineto{\pgfqpoint{0.494128in}{0.102406in}}%
\pgfpathlineto{\pgfqpoint{0.489444in}{0.096890in}}%
\pgfpathlineto{\pgfqpoint{0.487545in}{0.091373in}}%
\pgfpathlineto{\pgfqpoint{0.491944in}{0.085857in}}%
\pgfpathlineto{\pgfqpoint{0.500004in}{0.080340in}}%
\pgfpathlineto{\pgfqpoint{0.478603in}{0.074824in}}%
\pgfpathlineto{\pgfqpoint{0.472597in}{0.069307in}}%
\pgfpathlineto{\pgfqpoint{0.490701in}{0.063791in}}%
\pgfpathlineto{\pgfqpoint{0.494832in}{0.058275in}}%
\pgfpathlineto{\pgfqpoint{0.470001in}{0.052758in}}%
\pgfpathlineto{\pgfqpoint{0.470001in}{0.052758in}}%
\pgfusepath{stroke}%
\end{pgfscope}%
\begin{pgfscope}%
\pgfsetrectcap%
\pgfsetmiterjoin%
\pgfsetlinewidth{0.803000pt}%
\definecolor{currentstroke}{rgb}{0.000000,0.000000,0.000000}%
\pgfsetstrokecolor{currentstroke}%
\pgfsetdash{}{0pt}%
\pgfpathmoveto{\pgfqpoint{0.050000in}{0.050000in}}%
\pgfpathlineto{\pgfqpoint{0.050000in}{2.874408in}}%
\pgfusepath{stroke}%
\end{pgfscope}%
\begin{pgfscope}%
\pgfsetrectcap%
\pgfsetmiterjoin%
\pgfsetlinewidth{0.803000pt}%
\definecolor{currentstroke}{rgb}{0.000000,0.000000,0.000000}%
\pgfsetstrokecolor{currentstroke}%
\pgfsetdash{}{0pt}%
\pgfpathmoveto{\pgfqpoint{0.921731in}{0.050000in}}%
\pgfpathlineto{\pgfqpoint{0.921731in}{2.874408in}}%
\pgfusepath{stroke}%
\end{pgfscope}%
\begin{pgfscope}%
\pgfsetrectcap%
\pgfsetmiterjoin%
\pgfsetlinewidth{0.803000pt}%
\definecolor{currentstroke}{rgb}{0.000000,0.000000,0.000000}%
\pgfsetstrokecolor{currentstroke}%
\pgfsetdash{}{0pt}%
\pgfpathmoveto{\pgfqpoint{0.050000in}{0.050000in}}%
\pgfpathlineto{\pgfqpoint{0.921731in}{0.050000in}}%
\pgfusepath{stroke}%
\end{pgfscope}%
\begin{pgfscope}%
\pgfsetrectcap%
\pgfsetmiterjoin%
\pgfsetlinewidth{0.803000pt}%
\definecolor{currentstroke}{rgb}{0.000000,0.000000,0.000000}%
\pgfsetstrokecolor{currentstroke}%
\pgfsetdash{}{0pt}%
\pgfpathmoveto{\pgfqpoint{0.050000in}{2.874408in}}%
\pgfpathlineto{\pgfqpoint{0.921731in}{2.874408in}}%
\pgfusepath{stroke}%
\end{pgfscope}%
\begin{pgfscope}%
\pgfsetbuttcap%
\pgfsetmiterjoin%
\definecolor{currentfill}{rgb}{1.000000,1.000000,1.000000}%
\pgfsetfillcolor{currentfill}%
\pgfsetlinewidth{0.000000pt}%
\definecolor{currentstroke}{rgb}{0.000000,0.000000,0.000000}%
\pgfsetstrokecolor{currentstroke}%
\pgfsetstrokeopacity{0.000000}%
\pgfsetdash{}{0pt}%
\pgfpathmoveto{\pgfqpoint{0.921731in}{0.050000in}}%
\pgfpathlineto{\pgfqpoint{1.793462in}{0.050000in}}%
\pgfpathlineto{\pgfqpoint{1.793462in}{2.874408in}}%
\pgfpathlineto{\pgfqpoint{0.921731in}{2.874408in}}%
\pgfpathlineto{\pgfqpoint{0.921731in}{0.050000in}}%
\pgfpathclose%
\pgfusepath{fill}%
\end{pgfscope}%
\begin{pgfscope}%
\pgfpathrectangle{\pgfqpoint{0.921731in}{0.050000in}}{\pgfqpoint{0.871731in}{2.824408in}}%
\pgfusepath{clip}%
\pgfsys@transformshift{0.921731in}{0.050000in}%
\pgftext[left,bottom]{\includegraphics[interpolate=true,width=0.880000in,height=2.830000in]{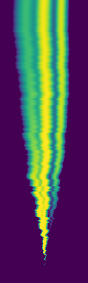}}%
\end{pgfscope}%
\begin{pgfscope}%
\definecolor{textcolor}{rgb}{0.000000,0.000000,0.000000}%
\pgfsetstrokecolor{textcolor}%
\pgfsetfillcolor{textcolor}%
\pgftext[x=1.357597in,y=3.013297in,,]{\color{textcolor}\rmfamily\fontsize{15.000000}{18.000000}\selectfont \(\displaystyle \varrho_{t,\textcolor{red}{X_t}}\)}%
\end{pgfscope}%
\begin{pgfscope}%
\pgfsetrectcap%
\pgfsetmiterjoin%
\pgfsetlinewidth{0.803000pt}%
\definecolor{currentstroke}{rgb}{0.000000,0.000000,0.000000}%
\pgfsetstrokecolor{currentstroke}%
\pgfsetdash{}{0pt}%
\pgfpathmoveto{\pgfqpoint{0.921731in}{0.050000in}}%
\pgfpathlineto{\pgfqpoint{0.921731in}{2.874408in}}%
\pgfusepath{stroke}%
\end{pgfscope}%
\begin{pgfscope}%
\pgfsetrectcap%
\pgfsetmiterjoin%
\pgfsetlinewidth{0.803000pt}%
\definecolor{currentstroke}{rgb}{0.000000,0.000000,0.000000}%
\pgfsetstrokecolor{currentstroke}%
\pgfsetdash{}{0pt}%
\pgfpathmoveto{\pgfqpoint{1.793462in}{0.050000in}}%
\pgfpathlineto{\pgfqpoint{1.793462in}{2.874408in}}%
\pgfusepath{stroke}%
\end{pgfscope}%
\begin{pgfscope}%
\pgfsetrectcap%
\pgfsetmiterjoin%
\pgfsetlinewidth{0.803000pt}%
\definecolor{currentstroke}{rgb}{0.000000,0.000000,0.000000}%
\pgfsetstrokecolor{currentstroke}%
\pgfsetdash{}{0pt}%
\pgfpathmoveto{\pgfqpoint{0.921731in}{0.050000in}}%
\pgfpathlineto{\pgfqpoint{1.793462in}{0.050000in}}%
\pgfusepath{stroke}%
\end{pgfscope}%
\begin{pgfscope}%
\pgfsetrectcap%
\pgfsetmiterjoin%
\pgfsetlinewidth{0.803000pt}%
\definecolor{currentstroke}{rgb}{0.000000,0.000000,0.000000}%
\pgfsetstrokecolor{currentstroke}%
\pgfsetdash{}{0pt}%
\pgfpathmoveto{\pgfqpoint{0.921731in}{2.874408in}}%
\pgfpathlineto{\pgfqpoint{1.793462in}{2.874408in}}%
\pgfusepath{stroke}%
\end{pgfscope}%
\begin{pgfscope}%
\definecolor{textcolor}{rgb}{0.000000,0.000000,0.000000}%
\pgfsetstrokecolor{textcolor}%
\pgfsetfillcolor{textcolor}%
\pgftext[x=1.967808in,y=0.191220in,,]{\color{textcolor}\rmfamily\fontsize{15.000000}{18.000000}\selectfont \(\displaystyle 1\)}%
\end{pgfscope}%
\begin{pgfscope}%
\definecolor{textcolor}{rgb}{0.000000,0.000000,0.000000}%
\pgfsetstrokecolor{textcolor}%
\pgfsetfillcolor{textcolor}%
\pgftext[x=1.967808in,y=2.733188in,,]{\color{textcolor}\rmfamily\fontsize{15.000000}{18.000000}\selectfont \(\displaystyle 0\)}%
\end{pgfscope}%
\begin{pgfscope}%
\definecolor{textcolor}{rgb}{0.000000,0.000000,0.000000}%
\pgfsetstrokecolor{textcolor}%
\pgfsetfillcolor{textcolor}%
\pgftext[x=1.967808in,y=1.462204in,,]{\color{textcolor}\rmfamily\fontsize{15.000000}{18.000000}\selectfont \(\displaystyle t\)}%
\end{pgfscope}%
\begin{pgfscope}%
\pgfsetroundcap%
\pgfsetroundjoin%
\pgfsetlinewidth{1.003750pt}%
\definecolor{currentstroke}{rgb}{0.000000,0.000000,0.000000}%
\pgfsetstrokecolor{currentstroke}%
\pgfsetdash{}{0pt}%
\pgfpathmoveto{\pgfqpoint{1.967808in}{1.293199in}}%
\pgfpathquadraticcurveto{\pgfqpoint{1.967808in}{0.826703in}}{\pgfqpoint{1.967808in}{0.375735in}}%
\pgfusepath{stroke}%
\end{pgfscope}%
\begin{pgfscope}%
\pgfsetroundcap%
\pgfsetroundjoin%
\pgfsetlinewidth{1.003750pt}%
\definecolor{currentstroke}{rgb}{0.000000,0.000000,0.000000}%
\pgfsetstrokecolor{currentstroke}%
\pgfsetdash{}{0pt}%
\pgfpathmoveto{\pgfqpoint{1.995586in}{0.431291in}}%
\pgfpathlineto{\pgfqpoint{1.967808in}{0.375735in}}%
\pgfpathlineto{\pgfqpoint{1.940030in}{0.431291in}}%
\pgfusepath{stroke}%
\end{pgfscope}%
\begin{pgfscope}%
\pgfsetroundcap%
\pgfsetroundjoin%
\pgfsetlinewidth{1.003750pt}%
\definecolor{currentstroke}{rgb}{0.000000,0.000000,0.000000}%
\pgfsetstrokecolor{currentstroke}%
\pgfsetdash{}{0pt}%
\pgfpathmoveto{\pgfqpoint{1.967808in}{1.631209in}}%
\pgfpathquadraticcurveto{\pgfqpoint{1.967808in}{2.097705in}}{\pgfqpoint{1.967808in}{2.548673in}}%
\pgfusepath{stroke}%
\end{pgfscope}%
\begin{pgfscope}%
\pgfsetroundcap%
\pgfsetroundjoin%
\pgfsetlinewidth{1.003750pt}%
\definecolor{currentstroke}{rgb}{0.000000,0.000000,0.000000}%
\pgfsetstrokecolor{currentstroke}%
\pgfsetdash{}{0pt}%
\pgfpathmoveto{\pgfqpoint{1.940030in}{2.493118in}}%
\pgfpathlineto{\pgfqpoint{1.967808in}{2.548673in}}%
\pgfpathlineto{\pgfqpoint{1.995586in}{2.493118in}}%
\pgfusepath{stroke}%
\end{pgfscope}%
\end{pgfpicture}%
\makeatother%
\endgroup%

%% file: sections/theory.tex
\section{Theoretical Analysis}

In this section, we analyse the different sources of error in our algorithm, building on the theoretical analysis presented in \cite{benton2024nearly, cordero-encinar2025nonasymptotic}. We first introduce the required assumptions.
\begin{assumption}\label{assumption:target_requirements}
    The target distribution has density w.r.t. Lebesgue,  $\pi \propto \exp({-V_\pi})\in\mathcal{P}(\mathbb{R}^d)$, has a finite second-order moment, $M_2 \coloneqq \mathbb{E}_\pi[\Vert X\Vert^2]<\infty$, and the potential $V_\pi$ is $L_\pi$-smooth.
\end{assumption}
\begin{assumption}\label{assumption:target_requirements_2}
    The potential $V_\pi$ is strongly convex outside a compact set, i.e., $\exists\,R>0$ such that $\inf_{\Vert x\Vert \geq R} \nabla^2 V_\pi \succ 0$.
\end{assumption}
We note that under Assumptions \ref{assumption:target_requirements} and \ref{assumption:target_requirements_2}, it follows from \citet[Lemma 3.2]{cordero-encinar2025nonasymptotic} and \citet[Lemma 3.8]{cattiaux2025diffusion} that $\nabla\log\mu_t$ is Lipschitz continuous with finite constant $L_t$, satisfying $\sup_t L_t\leq L$. A key ingredient to bound the \gls*{kl} divergence between the target $\pi$ and the final distribution $p_1$ obtained by the \gls*{dpsmc} algorithm is to control the score estimation error. First, assume there is no discretisation error, so that we can sample $X_k \sim {\mu}_{t_k}$ exactly and then generate the weighted particles $\{(w_k^i, \bar{Y}_k^i)\}_{i=1}^N$. For simplicity, let $\alpha_t$ be an arbitrary scalar schedule and define the score estimator
\begin{equation}\label{eq:mc_score_estimator_control_variate}
    S_k^N(X_k) := \sum_{i=1}^{N}  w_k^i\phi_{t_k,X_k}^{\alpha_{t_k}\mathbf{I}}(\bar{Y}_k^i).
\end{equation}
The following proposition provides a bound on the error of this estimator.
\begin{restatable}[Score estimation error]{proposition}{scoreEstimationError}\label{prop:score_estimation_error}
    Suppose Assumption \ref{assumption:target_requirements} holds and $0<\lambda_{t_k}<1$. Let $\{X_k\}_{k=1}^{K}$ be such that $\mu_{t_k} = \mathrm{Law}(X_k)$ for all $k\geq1$. Given the score estimator $S_k^N(X_k)$ defined in \eqref{eq:mc_score_estimator_control_variate}, assume that for some $\delta(k,N)\geq0$, $\mathbb{E}\left[W_2^2\left(\sum_{i=1}^{N} w_k^i\delta_{\bar{Y}_k^i}, \varrho_{t_k,X_k}\right)\right]\leq \delta(k,N)$ holds where the expectation is over the joint law of $X_k$, the particles, and the weights. Then
    \begin{align}
        &\mathbb{E}\left[\Vert\nabla\log\mu_{t_k}(X_k) - S_k^{N}( X_k)\Vert^2\right] \leq \,2\,\delta(k, N)\left(\alpha_{t_k}^2\tfrac{\lambda_{t_k}}{\sigma^4(1-\lambda_{t_k})^2} +(1-\alpha_{t_k})^2\tfrac{L_\pi^2}{\lambda_{t_k}}\right).
        \label{eq:score_estimation_error}
    \end{align}
\end{restatable}
See Appendix~\ref{subsec_app:score_convergence_analysis} for the proof. This result generalises Proposition 3.1 in \citep{he2024zerothorder}, which can be recovered by setting $\alpha_{t_k} = 1$ for all $k$ and using a system of equally weighted i.i.d particles. To see how this relates to posterior sampling difficulties at the start, consider the regime when $t_k$ is small. Here, $\lambda_{t_k}$ is small and $\alpha_{t_k}=\alpha_{t_k}^\ast$ is close to 1. The bracketed terms are hence of order $\mathcal{O}(\lambda_{t_k})$ which implies that having $\delta(k, N) = \mathcal{O}(\varepsilon^2/\lambda_{t_k})$ provides a bound on the first term of order $\mathcal{O}(\varepsilon^2)$. For example, when $\alpha_{t_k} = 1-\lambda_{t_k}$, i.e. \gls*{msi}, the bracketed terms become $\lambda_{t_k}(1/\sigma^4 + L_\pi^2)$. Hence, while $\varrho_{t_k,X_k}$ is hard to sample from for small $t_k$, relatively large mean-squared approximation errors (of order $\mathcal{O}(\varepsilon^2/\lambda_{t_k})$) between $\varrho_{t_k,X_k}$ and the distribution of our auxiliary variables can still result in a reliable score estimator when measured in the $L^2$  sense.
In other words, poorly sampled \gls*{smc} particles in the initial steps do not significantly degrade the score error. We remark, however, that it remains non-trivial how $N$ controls the $W_2^2$ distance between the SMC empirical measure and the posterior.

We can apply Proposition \ref{prop:score_estimation_error} to obtain a bound on the score estimation error for our weighted estimator \eqref{eq:mc_score_estimator_control_variate}. In particular, if we require the score estimation error over time to be bounded by $\varepsilon^2_{\text{score}}$, then assuming $\delta(k, N)$ vanishes with large $N$ for all $k$, we can choose $N$ large enough to satisfy
\begin{equation*}
    \sum_{k=0}^{K-1} h\,\mathbb{E}\left[\left\Vert \nabla \log {\mu}_{t_k}(X_{k}) - S_k^{N}(X_{k})\right\Vert^2\right] \leq \varepsilon^2_{\text{score}}.
\end{equation*}
This allows us to invoke \citet[Theorem 1]{benton2024nearly} and \citet[Theorem 3.4]{cordero-encinar2025nonasymptotic} to bound the final errors due to initialisation, discretisations, or imperfect tracking. In particular, we have the following result.
\begin{restatable}{theorem}{finalError}\label{thm:final_error}
    Under Assumption \ref{assumption:target_requirements} and further Assumption \ref{assumption:target_requirements_2} for (ii) and (iii), the \gls*{dpsmc} algorithm with (i) OU time-reversal dynamics initialised at $X_{N_{\text{LL}}}\sim p_{0}$, (ii) SI dynamics initialised at $X_0\sim\nu$, and (iii) \gls*{dald} initialised at $X_0\sim\nu$, satisfy respectively
    \begin{align*}
    (i) \, \,&\mathrm{KL}\left(\pi\;\|\;p_1\right)
        \lesssim \mathrm{KL}(\mu_{0} \;\|\; p_{0}) +  \tfrac{d}{K - N_{\text{LL}}} \left(1+ T\right) + \varepsilon_{\text{score}}^2,\\
    (ii) \, \,&\mathrm{KL}\left(\pi\;\|\;p_1\right)
        \lesssim  \tfrac{dL^2}{K}+  \varepsilon^2_{\text{score}},\\
        (iii) \, \,&\mathrm{KL}\left(\pi\;\|\;p_1\right)
        \lesssim \epsilon\left(1+\tfrac{L^2}{K^2\epsilon^4}\right)(M_2+d) + \tfrac{d L^2 }{K\epsilon^2}\left(1+  \tfrac{L}{K\epsilon}\right) +  \varepsilon^2_{\text{score}}.
    \end{align*}
\end{restatable}
See Appendix \ref{subsec_app:final_error_proof} for a proof. These bounds imply that, for $\varepsilon = \mathcal{O}(\varepsilon_{\text{score}})$, under appropriate choices for $K$, $p_0$, $N_{\text{LL}}$, and $\epsilon$, we can keep $\text{{KL}}\left(\pi\;\|\;p_1\right) \leq \varepsilon^2$. 

%% file: sections/experiments.tex
\section{Experiments}

We now evaluate \gls*{dpsmc} across several benchmarks. Tables \ref{tab:main_results} and \ref{tab:normalising_constant_estimation} summarise our findings, where the mean and standard error of the metrics are reported.

\fparagraph{Samplers.} We compare our methods against several classical samplers: \gls*{ais} \citep{neal2001annealed} and \gls*{smc} \citep{delmoral2006sequential}, and those based on reverse-diffusion, RDMC \citep{huang2024reverse} and SLIPS \citep{grenioux2024stochastic}. For completeness, we include \gls*{pt} \citep{geyer1995annealing, swendsen1986replica} as an equilibrium method baseline. Hyperparameter choices are found in Appendix \ref{subsec_app:benchmark_methods}. We fix the number of energy evaluations for all samplers, except RDMC with twice as many. We also globally fix the number of discretisation steps to $K=1024$. We test several instances of DPSMC, each driven by different dynamics: \gls*{dald}, \gls*{dald} with \gls*{ll}, the OU time-reversal, and \gls*{si}. We set $N_\text{LL}=64$ for methods with \gls*{ll} and $N_Y=128$ to match the budget. Samplers are assessed on the 4096 samples they produce across 10 seeds.

\fparagraph{Toy target distributions.} We consider several toy targets that test a sampler's ability to tackle multimodality and complicated geometries. Our first example is GMM40, a uniform mixture of 40 Gaussians with identity covariance and randomly sampled means. We consider dimensions $d=2$ and $d=50$. We also test targets with challenging geometries such as Rings ($d=2$) \cite{grenioux2024stochastic}, where regions of high density are localised along four concentric rings, and Funnel ($d=10$) \cite{neal2003slice}, which consists of an exponentially narrowing funnel. We evaluate samples relative to the target through their $W_2$ distance for GMM40 and Rings, and the sliced \gls*{ks} distance for Funnel. We find \gls*{dpsmc} yields the best performance in all benchmarks. For Funnel, massive gains are attributed to our matrix control-variates that handle anisotropy well. We find starting later with \gls*{ll} benefits \gls*{dald}, providing a finer discretisation and an easier initial posterior. On the other hand, \gls*{si}, free of critical hyperparameters, performs comparably to OU.

\fparagraph{Bayesian logistic regression.} To showcase how samplers deal with high-dimensional real-world data, we consider posterior sampling from a Bayesian logistic regression model under the Ionosphere ($d=35$) and Sonar ($d=61$) datasets. Each dataset is split $70$-$10$-$20$\% into training, validation, and testing sets. Samples, i.e. model weights, are evaluated on their predictive log-likelihood on the test set. For fairness, we tune hyperparameters only using the validation log-likelihood and combine validation and training sets for samplers that do not require tuning. We find \gls*{dald} and OU perform best. Here, the target is log-concave and so are the path marginals and posteriors. Despite both starting at a Gaussian, we believe \gls*{dald} is able to re-equilibrate w.r.t. the path marginals at times when scores are more reliable and hence do not propagate score errors the same way as in \gls*{si} that relies on exact dynamics. OU achieves a similar performance by initialising near the end of the path.

\begin{table*}[t!]
    \caption{\textit{Performance of algorithms on toy targets and Bayesian logistic regression datasets across 10 seeds.} Best result is in \textbf{bold}, and second-best is \underline{underlined}. The second column indicates the number of batched energy evaluations---a metric for measuring runtimes in the regime where hardware capacity exceeds batch size requirements. DPSMC is overall performant while highly parallelisable.}
    \centering
    \setlength{\tabcolsep}{1.5pt}
    \label{tab:main_results}
    {\small
    \begin{tabular*}{\textwidth}{lccccccc}
        \toprule
        Algorithm & \makecell{Batched\\Evals { ($\downarrow$)}} & \makecell{GMM40\\  ($d=2$)\\ $W_2$ ($\downarrow$)} & \makecell{GMM40\\ ($d=50$)\\ $W_2$ ($\downarrow$)} & \makecell{Rings\\ ($d=2$)\\ $W_2$ ($\downarrow$)} & \makecell{Funnel\\($d=10$)\\ KS ($\downarrow$)} & \makecell{Ionosphere\\ ($d=35$)\\ Log-LL ($\uparrow$)} & \makecell{Sonar\\ ($d=61$)\\ Log-LL ($\uparrow$)}\\
        \midrule
        AIS & 131K & \tabres{2.35}{0.41}& \tabres{115.13}{2.98} & \tabres[]{0.19}{0.01} & \tabres{0.037}{0.003} & \tabres{\text{-}87.79}{0.10} & \tabres{\text{-}110.94}{0.10}\\
        SMC & 131K & \tabres{5.50}{0.22} & \tabres{111.81}{4.84} & \tabres[\mathbf]{0.18}{0.01} & \tabres{0.035}{0.003}& \tabres{\text{-}87.79}{0.10}& \tabres{\text{-}110.94}{0.10}\\
        \midrule
        RDMC {\small \citep{huang2024reverse}} & 32K & \tabres{12.35}{0.11} & \tabres{79.37}{0.10} & \tabres{0.29}{0.02} & \tabres{0.082}{0.003} & \tabres{\text{-}109.97}{0.23} & \tabres{\text{-}129.82}{0.18}\\
        SLIPS {\small \citep{grenioux2024stochastic}} & 32K & \tabres{4.07}{0.40} & \tabres{27.57}{0.83} & \tabres[\mathbf]{0.18}{0.01} & \tabres{0.073}{0.003} & \tabres{\text{-}87.88}{0.10} & \tabres{\text{-}110.88}{0.06}\\
        \midrule
        DPSMC {\footnotesize(DALD)} & \textbf{1K} & \tabres{1.92}{0.22} & \tabres{29.05}{1.51} & \tabres{0.20}{0.02} & \tabres{0.062}{0.003} & \tabres[\underline]{\text{-}83.77}{0.07} & \tabres[\underline]{\text{-}108.45}{0.05}\\
        \quad {" w/ LL} & \textbf{1K} & \tabres{1.57}{0.33} & \tabres[\mathbf]{26.82}{1.38}  & \tabres[]{0.19}{0.02} & \tabres{0.046}{0.004} & \tabres{\text{-}85.22}{0.11} & \tabres{\text{-}109.34}{0.08}\\
        DPSMC {\footnotesize(OU)} & \textbf{1K} & \tabres[\mathbf]{1.14}{0.24} & 	\tabres[\underline]{27.27}{1.06} & \tabres[\mathbf]{0.18}{0.02} & \tabres[\underline]{0.028}{0.002} &	\tabres[\mathbf]{\text{-}83.73}{0.16} & \tabres[\mathbf]{\text{-}108.23}{0.07} \\
        DPSMC {\footnotesize(SI)} & \textbf{1K} & \tabres[\underline]{1.23}{0.30} & \tabres{28.78}{1.78} & \tabres[\mathbf]{0.18}{0.03} & \tabres[\mathbf]{0.026}{0.004} & \tabres{\text{-}87.54}{0.11} & \tabres{\text{-}111.06}{0.08}\\
        \midrule
        PT & 2M & \tabres{1.13}{0.18} & \tabres{42.44}{2.13} & \tabres{0.18}{0.02} & \tabres{0.022}{0.005}& \tabres{\text{-}87.91}{0.17} & \tabres{\text{-}111.04}{0.08}\\
        \bottomrule
    \end{tabular*}}
\end{table*}

\begin{table}[t]
    \caption{\textit{Performance of algorithms on normalising constant estimation and sample quality on Rings and Funnel targets}. Best result is in \textbf{bold}, and second-best is \underline{underlined}. DPSMC provides improved estimates for the normalising constant. AIS, SMC, and RDSMC results are lifted from \citet{wu2025reverse}.}\label{tab:normalising_constant_estimation}
    \centering
    \setlength{\tabcolsep}{2pt}
    \vspace{-0.7em}
    {\small\begin{tabular}{lcccccc}
        \toprule
        \multirow{3}*{Algorithm} & \multicolumn{3}{c}{Rings ($d=2$)} & \multicolumn{2}{c}{Funnel ($d=10$)} \\
        \cmidrule(lr){2-4}\cmidrule(lr){5-6}
        & \makecell{Radius TVD ($\downarrow$)} & \makecell{Angle TVD ($\downarrow$)} & \makecell{$\log Z_\pi$ Bias ($\downarrow$)} & KS ($\downarrow$) & \makecell{$\log Z_\pi$ Bias ($\downarrow$)}\\
        \midrule
        AIS & \tabres[\mathbf]{0.10}{0.00} & \tabres[\underline]{0.15}{0.00} & \tabres{0.05}{0.00} & \tabres[\underline]{0.07}{0.00} & \tabres[\underline]{0.28}{0.01}\\
        SMC & \tabres[\mathbf]{0.10}{0.00} & \tabres[\underline]{0.15}{0.00} & \tabres{0.05}{0.00} & \tabres[\underline]{0.07}{0.00} & \tabres[\underline]{0.28}{0.01}\\
        RDSMC & \tabres[\underline]{0.13}{0.01} & \tabres{0.21}{0.01} & \tabres[\underline]{0.03}{0.00} & \tabres{0.11}{0.00} & \tabres[\underline]{0.28}{0.10}\\
        \midrule
        DPSMC {\footnotesize (OU)} & \tabres{0.31}{0.01} & \tabres[\mathbf]{0.14}{0.01} & \tabres[\mathbf]{0.02}{0.01}& \tabres[\mathbf]{0.06}{0.00} & \tabres[\mathbf]{0.12}{0.07}\\
        \bottomrule
    \end{tabular}}
\end{table}

\fparagraph{Normalising constant estimation.} An important aspect of samplers is their ability to estimate normalising constants of unnormalised target densities. We adapt DPSMC (OU) to use the RDSMC framework \citep{wu2025reverse}, an exact approximation method enabling AIS/SMC on the diffusion path itself, and test it on Rings and Funnel under a comparable number of energy evaluations and discretisation, adjusting the number of MALA steps to match. This is possible as our SMC setup produces unbiased estimates for the path marginal densities (see Proposition \eqref{prop:fk_identity}) unlike RDMC or SLIPS. Here, we initialise auxiliary variables at $\rho_{0}$ with AIS and use our SMC framework thereafter. We measure the total variation distance (TVD) of radial and angular components w.r.t. ground truth samples for Rings and measure KS for Funnel. We find DPSMC achieves superior normalising constant estimates while maintaining good samples. When we deviate from RDSMC's settings, we observe that a finer discretisation with fewer MALA steps leads to better sample quality, e.g. lower radius TVD, at the cost of more biased constant estimates. We can pursue this trade-off as we evolve auxiliary variables along conditional distributions instead of running \gls*{smc} between a prior and $\rho_k$ at every step.

Overall, we find DPSMC has speed-ups of several orders of magnitude compared to other samplers while achieving comparable if not better sample quality. In the modern age of GPUs and parallel computing, inference at scale is primarily achieved by distributing work horizontally and has significantly rewarded methods that can leverage this \citep{vaswani2017attention, chen2023accelerating}. We argue that a sampler's \textit{ability to batch energy evaluations} is hence an increasingly important design axis. In essence, methods that can exploit this batching for estimating better transport paths can achieve better wall-clock times and sample quality. Next, we perform some ablation experiments to test the limitations of our method and assumptions.

\begin{figure}
    \centering
    {
    \def\mathdefault#1{#1}
    \resizebox{\textwidth}{!}{\input{media/pgfs/ablation_results.pgf}}
    }
    \vspace{-1.75em}
    \caption{\textit{Ablation experiments on a five-dimensional bi-modal target as mode separation is increased.} (Left) predicted weight of heavier mode when auxiliary variables are initialised at the lighter mode. (Middle and right) empirical rates for how fast the $W_2^2$ from the posterior and score MSE vanish as particle count is increased. Log-log ordinary least-squares slopes are affixed next to each trend line.}
    \label{fig:ablation_experiments}
\end{figure}
\fparagraph{Initialisation dependency.} Given the challenging posterior at time $t=0$, we aim to quantify how poor initialisation of auxiliary variables affects our algorithm's performance. We consider sampling from an uneven bi-modal Gaussian mixture target in five dimensions while adversarially initialising auxiliary variables at the lighter mode. Further details and results are in Appendix \ref{app:subsec:initialisation_dependency_ablation}. We test our \gls*{si} and \gls*{dald} instantiations of \gls*{dpsmc}, which start at $\lambda_0=0$. Previously, we conjectured that \gls*{dald} re-equilibrates at more favourable intermediate times. Figure \ref{fig:ablation_experiments} (left) shows that \gls*{dald} does recover the heavier mode weight even past the known threshold where the marginal and posterior share a log-concavity window \citep[Proposition 19]{grenioux2024stochastic}, whereas \gls*{si} does not. Eventually, all samples collapse onto the lighter mode past a certain separation. Indeed, since our auxiliary variables are evolved via local moves, multimodality remains challenging in the absence of global information.

\fparagraph{Vanishing score error.} Our error analysis implicitly assumes the $W_2^2$ distance between the SMC empirical measure and the posterior is controlled by the particle count $N$. To test this assumption, we run \gls*{dald} on the previous bi-modal mixture target, with varying particle counts and mode separations, and measure the score MSE and empirical $W_2^2$ distance against ground-truth posterior samples. Figure \ref{fig:ablation_experiments} (middle and right) shows our findings while further details are in Appendix \ref{app:subsec:vanishing_error_ablation}. Notably, with modest separation, we achieve a $W_2^2$ convergence rate that closely recapitulates the i.i.d. rate of $\mathcal{O}(N^{-2/d})$ in \citet{fournier2015rate} and an MSE rate close to the standard $\mathcal{O}(N^{-1})$ for bounded test functions \citep[Section 3.5.1]{delmoral2004feynmankac}. The rate degrades, however, with increasing separation. Hence, we believe our assumption of score error control is sound under modest separation and limited on well-separated targets that may require far larger particle counts than is practically feasible.

%% file: sections/discussion.tex
\section{Discussion}

In this work, we develop diffusion path sequential Monte Carlo (DPSMC), a sampler on diffusion paths using \gls*{smc} as a backbone for score estimation. We further introduce control-variate schedules and motivate principled parameterisations of the diffusion path. Finally, under different instantiations of the path, we provide theoretical guarantees for DPSMC and test it on a range of benchmark targets.

\fparagraph{Limitations and future work.} Our results open several promising directions. First, our analysis would greatly benefit from a finite-particle complexity bound which, given the target sequence depends on the particle approximations, requires tools from mean-field particle analysis \citep{delmoral2011concentration} and deserves its own careful treatment. Second, we find that pseudomarginal importance weights under the nested SMC scheme \citep{wu2025reverse} can exhibit some variance. Exploring other transition kernels and \gls*{si} instantiations that induce smoother, better-conditioned posteriors along the path is a promising strategy. Third, while we choose hyperparameters with care, their choice hinges on oracle (evaluation) metrics which may be unavailable. We think the normalising constant estimates can be informative for this. Finally, as our SMC subroutines rely on local moves, multimodality can only be tackled well in mild, controlled settings. Incorporating progressive sampling strategies, e.g. tempering, is a natural next step for tackling very high-dimensional, multimodal targets found in real-world settings.

%% file: appendix/a_additional_background.tex
\section{Additional Background}\label{app:additional_background}

\subsection{Related Work on Diffusion-Based Neural Samplers}\label{app:subsec:neural_samples_related_work}

There are several works that use neural networks for sampling from target distributions while leveraging the diffusion path, most commonly through the OU process and its time-reversal. Below, we provide a brief discussion of related works that fall under this neural sampler category. We remark that our approach is \textit{training-free} in contrast.

Some works directly train a diffusion model, leveraging the same score-matching objective without assuming access to samples beforehand. For example, \citet{akhound-sadegh2024iterated} estimate the time-varying score functions using self-normalised importance sampling together with TSI and use them as regression targets. On the other hand, some works temper the target to more easily draw an initial set of samples, from which a diffusion model is trained and steered to produce samples from a less tempered version of the learned distribution, and the process is repeated \citep{akhound-sadegh2025progressive, rissanen2025progressive}.

Several works approach the problem from a variational inference (VI) perspective, whereby a forward and reverse process have parameterised drift functions, and these are optimised to minimise the divergence between their induced path measures, i.e. so they become time-reversals of each other \citep{zhang2022path, vargas2024transport, richter2024improved, chen2025sequential}. These can be analogously framed under a stochastic optimal control light.

Related to our work is PDDS \citep{phillips2024particle} that reframes sampling from a target as sampling from the reward-tilted distribution of a diffusion model. \citet{phillips2024particle} define an SMC sampler along this path using twisting potentials parameterised by neural networks to more effectively guide the sampling process. To facilitate training, they consider their novel score matching loss, where they introduce a control-variate in the score identity, the choice of which can be viewed as MSI.

%% file: appendix/b_smc.tex
\section{Sequential Monte Carlo samplers}\label{app:smc_intro}
In this section, we show certain \gls{smc} results which we have used in the paper.

\subsection{Marginal Score and Density (Proof of Proposition~\ref{prop:fk_identity})}\label{app:proof:fk_identity_proof}

\fkIdentity*
\begin{proof}
    Expanding the product we have that
    \begin{align*}
        \prod_{p=0}^{k}{G_p(y_{p-1}, y_p)}
        &= \frac{\tilde{\rho}_0(y_0)}{q_0(y_0)}\prod_{p=1}^{k}{\frac{\tilde{\rho}_p(y_p)\mathsf{L}_{p-1}(y_p, y_{p-1})}{\tilde{\rho}_{p-1}(y_{p-1}) \mathsf{K}_{p}(y_{p-1}, y_p)}}\\
        &= \frac{\tilde{\rho}_{k}(y_k)\prod_{p=1}^{k}{\mathsf{L}_{p-1}(y_p, y_{p-1})}}{q_0(y_0)\prod_{p=1}^{k}{\mathsf{K}_p(y_{p-1}, y_p)}} = \frac{\tilde{\rho}_{0:k}(y_{0:k})}{q_{0:k}(y_{0:k})}.
    \end{align*}
    The ratio in the proposition is therefore given by
    \begin{align*}
        \frac{\mathbb{E}[\varphi_{k,x_k}(Y_k)\prod_{p=1}^{k}{G_p(Y_{p-1}, Y_p)}]}{\mathbb{E}[\prod_{p=1}^{k}{G_p(Y_{p-1}, Y_p)}]}
        &= \frac{\int q_{0:k}(y_{0:k}) \varphi_{k,x}(y_k)\frac{\tilde{\rho}_{0:k}(y_{0:k})}{q_{0:k}(y_{0:k})}\mathrm{d}y_{0:k} }{\int q_{0:k}(y_{0:k})\frac{\tilde{\rho}_{0:k}(y_{0:k})}{q_{0:k}(y_{0:k})}\mathrm{d}y_{0:k}}\\
        &= \frac{\int \tilde{\rho}_{0:k}(y_{0:k})\varphi_{k,x_k}(y_k)\mathrm{d}y_{0:k}}{\int\tilde{\rho}_{0:k}(y_{0:k})\mathrm{d}y_{0:k}}\\
        &= \int \rho_{0:k}(y_{0:k})\varphi_{k,x_k}(y_k)\mathrm{d}y_{0:k} = \int \rho_{k}(y_{k})\varphi_{k,x_k}(y_k)\mathrm{d}y_k,
    \end{align*}
    where we have used the fact that $\rho_{0:k}(y_{0:k})$ admits $\rho_k(y_k)$ as a marginal. Since $\rho_{k}=\varrho_{t_k,x_k}$ and our test function $\varphi_{k,x}$ satisfies \eqref{eq:score_test_function} for any $x\in\mathbb{R}^{d}$, then we can conclude that
    \begin{align*}
        \frac{\mathbb{E}[\varphi_{k,x}(Y_k)\prod_{p=1}^{k}{G_p(Y_{p-1}, Y_p)}]}{\mathbb{E}[\prod_{p=1}^{k}{G_p(Y_{p-1}, Y_p)}]} = \mathbb{E}_{\varrho_{t_k,x_k}}[\varphi_{k,x_k}(Y_k)] = \nabla\log{\mu}_{t_k}(x_k).
    \end{align*}
    Now, we show the result for the path marginal density. We have
    \begin{align*}
        \mathbb{E}\left[\prod_{p=0}^{k}G_p(Y_{p-1}, Y_p)\right] &= \int q_{0:k}(y_{0:k})\frac{\tilde{\rho}_{0:k}(y_{0:k})}{q_{0:k}(y_{0:k})}\mathrm{d}y_{0:k}=\int\tilde{\rho}_k(y_k)\prod_{p=1}^{k}\mathsf{L}_{p-1}(y_p, y_{p-1})\mathrm{d}y_{0:k}\\
        &=\int \tilde{\rho}_k(y_k)\mathrm{d}y_k = \int \tfrac{1}{\sqrt{1-\lambda_{t_k}}^d}\nu\left(\tfrac{X_k-\sqrt{\lambda_{t_k}}y_k}{\sqrt{1 - \lambda_{t_k}}}\right)\pi(y_k)\mathrm{d}y_k.
    \end{align*}
With a change of variables $u=\sqrt{\lambda_{t_k}}y_k$, we hence have the result
\begin{align*}
    \mathbb{E}\left[\prod_{p=0}^{k}G_p(Y_{p-1}, Y_p)\right] &= \int \tfrac{1}{\sqrt{1-\lambda_{t_k}}^d}\nu\left(\tfrac{X_k-u}{\sqrt{1 - \lambda_{t_k}}}\right)\tfrac{1}{\sqrt{\lambda_{t_k}}^d}\pi\left(\tfrac{u}{\sqrt{\lambda_{t_k}}}\right)\mathrm{d}u = \mu_{t_k}(X_k).
\end{align*}
\end{proof}

\subsection{Weight Update for MALA Kernel}\label{subsec_app:mala_weight_update}
In Section~\ref{subsec:general_smc_posterior}, we write our potential functions for general kernels $(\mathsf{K}_p)_{p=1}^{k}$ and $(\mathsf{L}_{p-1})_{p=1}^{k}$ as
\begin{align}
    G_{k}(y_{k-1}, y_k) = \frac{\tilde{\rho}_{k}(y_{k})\mathsf{L}_{k-1}(y_{k}, y_{k-1})}{\tilde{\rho}_{k-1}(y_{k-1}) \mathsf{K}_{k}(y_{k-1}, y_{k})}. \label{app_eq:general_potential}
\end{align}
As we mention in Section~\ref{subsec:algorithmic_choices}, we choose the forward kernel to be the \gls*{mala} kernel which simplifies the weights. In order to see this, if $\mathsf{K}_k$ leaves $\rho_k$ invariant and we can choose the backward kernel as its time-reversal,
\begin{align}
    \mathsf{L}_{k-1}(y_k, y_{k-1}) = \frac{\rho_k(y_{k-1})\mathsf{K}_k(y_{k-1}, y_k)}{\rho_k(y_k)},\label{eq:markov_backwards_kernel}
\end{align}
then the incremental weight simplifies, which yields the familiar annealed importance sampling weights. In particular, we can write
\begin{align}
    {G}_k(y_{k-1}, y_k)
    &= \frac{\tilde{\rho}_k(y_k)\mathsf{L}_{k-1}(y_k, y_{k-1})}{\tilde{\rho}_{k-1}(y_{k-1})\mathsf{K}_k(y_{k-1}, y_k)}\nonumber\\
    &= \frac{\tilde{\rho}_k(y_k)}{\tilde{\rho}_{k-1}(y_{k-1})}\cdot\frac{\rho_k(y_{k-1})}{\rho_k(y_k)}
    = \frac{\tilde{\rho}_k(y_{k-1})}{\tilde{\rho}_{k-1}(y_{k-1})},
\end{align}
where we used $\rho_k = \tilde{\rho}_k/Z_k$ so that $\tilde{\rho}_k(y_k)/\rho_k(y_k) = Z_k$ and $\rho_k(y_{k-1}) = \tilde{\rho}_k(y_{k-1})/Z_k$.

%% file: appendix/c_score_identities.tex
\section{Score Identities}\label{sec_app:score_identities}

In this section, we show how we can derive control-variate schedules for minimising the variance of score estimators under different regimes. We do so for general base and target densities.

\subsection{Denoising and Target Score Identities for General Base $\nu$}\label{subsec_app:general_score_identities}
To define a general score test function with a control-variate schedule, we first present lemmas that extend \gls{dsi} and \gls{tsi} to general base distributions.

\begin{lemma}[Denoising Score Identity for General Base $\nu$]\label{app:lemma:score_identity_general_nu}
    For $t \in [0,1]$, it holds that
    \begin{align*}
        \nabla\log\mu_t(x) = \mathbb{E}_{\varrho_{t,x}}\left[\tfrac{1}{\sqrt{1 - \lambda_t}}{\nabla\log\nu\left(\tfrac{x - \sqrt{\lambda_t}Y}{\sqrt{1 - \lambda_t}}\right)}\right].
    \end{align*}
\end{lemma}
\begin{proof}
Recall that the diffusion path marginal $\mu_t$ is defined by the convolution
\begin{align*}
    \mu_t(x) = \int \tfrac{1}{\sqrt{\lambda_t(1 - \lambda_t)}^d}\nu\left(\tfrac{x-u}{\sqrt{1 - \lambda_t}}\right) \pi\left(\tfrac{u}{\sqrt{\lambda_t}}\right) \mathrm{d}u.
\end{align*}
With a change of variables $u = \sqrt{\lambda_t} y$, we have $\mathrm{d}u = \sqrt{\lambda_t}^d \mathrm{d}y$ and
\begin{align*}
    \mu_t(x) = \int\tfrac{1}{\sqrt{1 - \lambda_t}^d}\nu\left(\tfrac{x-\sqrt{\lambda_t}y}{\sqrt{1 - \lambda_t}}\right)\pi(y) \mathrm{d}y.
\end{align*}
Note that the integrand is precisely the unnormalised posterior. From here, we can solve for the score
\begin{align*}
    \nabla\log\mu_t(x) &= \tfrac{\nabla\mu_t(x)}{\mu_t(x)} = \tfrac{1}{\mu_t(x)}\int \nabla \left(\tfrac{1}{\sqrt{1 - \lambda_t}^d}\nu\left(\tfrac{x - \sqrt{\lambda_t} y}{\sqrt{1 - \lambda_t}}\right)\pi(y)\right)\mathrm{d}y\\
    &= \int\left(\tfrac{1}{\mu_t(x)}\tfrac{1}{\sqrt{1 - \lambda_t}^d}\nu\left(\tfrac{x - \sqrt{\lambda_t} y}{\sqrt{1 - \lambda_t}}\right)\pi(y)\right)\tfrac{1}{\sqrt{1 - \lambda_t}}\nabla\log\nu\left(\tfrac{x - \sqrt{\lambda_t}y}{\sqrt{1 - \lambda_t}}\right)\mathrm{d}y\\
    &= \int \varrho_{t,x}(y)\tfrac{1}{\sqrt{1 - \lambda_t}}\nabla\log\nu\left(\tfrac{x - \sqrt{\lambda_t}y}{\sqrt{1 - \lambda_t}}\right) \mathrm{d}y = \mathbb{E}_{\varrho_{t,x}}\left[\tfrac{1}{\sqrt{1 - \lambda_t}}\nabla\log\nu\left(\tfrac{x - \sqrt{\lambda_t}Y}{\sqrt{1 - \lambda_t}}\right)\right] .
\end{align*}
\end{proof}
\begin{lemma}[Target Score Identity for General Base $\nu$]\label{app:lemma:target_score_identity_general_nu}
    For $t \in [0,1]$, it holds that
    \begin{align*}
        \nabla\log\mu_t(x) = \mathbb{E}_{\varrho_{t,x}}\left[\tfrac{1}{\sqrt{\lambda_t}}\nabla\log\pi(Y)\right].
    \end{align*}
\end{lemma}
\begin{proof}
    As the diffusion path marginal $\mu_t$ is a convolution, we have that
    \begin{align*}
        \mu_t(x) = \int \tfrac{1}{\sqrt{\lambda_t(1 - \lambda_t)}^d}\nu\left(\tfrac{u}{\sqrt{1 - \lambda_t}}\right)\pi\left(\tfrac{x-u}{\sqrt{\lambda_t}}\right)\mathrm{d}u.
    \end{align*}
    Solving for the score
    \begin{align*}
        \nabla\log\mu_t(x)
        &= \tfrac{\nabla\mu_t(x)}{\mu_t(x)} = \tfrac{1}{\mu_t(x)} \int\nabla\left(\tfrac{1}{\sqrt{\lambda_t(1 - \lambda_t)}^d}\nu\left(\tfrac{u}{\sqrt{1 - \lambda_t}}\right)\pi\left(\tfrac{x - u}{\sqrt{\lambda_t}}\right)\right)\mathrm{d}u\\
        &= \int \left(\tfrac{1}{\mu_t(x)}\tfrac{1}{\sqrt{\lambda_t(1 - \lambda_t)}^d}\nu\left(\tfrac{u}{\sqrt{1 - \lambda_t}}\right)\pi\left(\tfrac{x-u}{\sqrt{\lambda_t}}\right)\right)\tfrac{1}{\sqrt{\lambda_t}}\nabla\log\pi\left(\tfrac{x-u}{\sqrt{\lambda_t}}\right)\mathrm{d}u.
    \end{align*}
    Using a change of variables $u = x-\sqrt{\lambda_t}y$, we have $\mathrm{d}u = \left\lvert-\sqrt{\lambda_t}^d\right\rvert\mathrm{d}y$ and
    \begin{align*}
        \nabla\log\mu_t(x)&= \int\left(\tfrac{1}{\mu_t(x)}\tfrac{1}{\sqrt{1 - \lambda_t}^d}\nu\left(\tfrac{x-\sqrt{\lambda_t}y}{\sqrt{1 - \lambda_t}}\right)\pi\left(y\right)\right)\tfrac{1}{\sqrt{\lambda_t}}\nabla\log\pi(y)\mathrm{d}y\\
        &= \int \varrho_{t,x}(y)\tfrac{1}{\sqrt{\lambda_t}}\nabla\log\pi(y)\mathrm{d}y = \mathbb{E}_{\varrho_{t,x}}\left[\tfrac{1}{\sqrt{\lambda_t}}\nabla\log\pi(Y)\right].
    \end{align*}
\end{proof}
\begin{proposition}\label{app:prop:convex_combination_matrix} As a consequence of Lemma~\ref{app:lemma:score_identity_general_nu} and Lemma~\ref{app:lemma:target_score_identity_general_nu}, for $\mathbf{A}\in\mathbb{R}^{d\times d}$ and $t \in [0,1]$, it holds that
    \begin{align*}
    \nabla\log\mu_t(x) = \mathbb{E}_{\varrho_{t,x}}\left[\tfrac{1}{\sqrt{1 - \lambda_t}}\mathbf{A}{\nabla\log\nu\left(\tfrac{x - \sqrt{\lambda_t}Y}{\sqrt{1 - \lambda_t}}\right)} + \tfrac{1}{\sqrt{\lambda_t}}(\mathbf{I} - \mathbf{A})\nabla\log\pi(Y)\right].
\end{align*}
\end{proposition}
\begin{proof}
The proof is immediate as a convex combination of the two identities given in Lemma~\ref{app:lemma:score_identity_general_nu} and Lemma~\ref{app:lemma:target_score_identity_general_nu}.
\end{proof}

\subsection{Control Variate Schedules}\label{subsec_app:control_variate_schedule}
\begin{remark}\label{app:remark:matrix_control_variate_identity} The identity in Proposition~\ref{app:prop:convex_combination_matrix} can be viewed as introducing a matrix-valued control variate. To see this, starting from Lemma~\ref{app:lemma:score_identity_general_nu} and the fact that $\mathbb{E}_{\varrho_{t,x}}[\nabla\log\varrho_{t,x}(Y)] = 0$, we can choose a control variate matrix $\mathbf{B}\in\mathbb{R}^{d\times d}$ to get
\begin{align*}
    \nabla\log\mu_t(x) &= \mathbb{E}_{\varrho_{t,x}}\left[\tfrac{1}{\sqrt{1 - \lambda_t}}\nabla\log\nu\left(\tfrac{x - \sqrt{\lambda_t}Y}{\sqrt{1 - \lambda_t}}\right) - \mathbf{B}\nabla\log\varrho_{t,x}(Y) \right] + \mathbf{B}\mathbf{0} \\
    &= \mathbb{E}_{\varrho_{t,x}}\left[\tfrac{1}{\sqrt{1 - \lambda_t}}\nabla\log\nu\left(\tfrac{x - \sqrt{\lambda_t}Y}{\sqrt{1 - \lambda_t}}\right) - \mathbf{B} \left(-\tfrac{\sqrt{\lambda_t}}{\sqrt{1 - \lambda_t}}\nabla\log\nu\left(\tfrac{x - \sqrt{\lambda_t}Y}{\sqrt{1 - \lambda_t}}\right) + \nabla\log\pi(Y)\right)\right]\\
    &= \mathbb{E}_{\varrho_{t,x}}\left[\tfrac{1}{\sqrt{1 - \lambda_t}}(\mathbf{I + \sqrt{\lambda_t}\mathbf{B}})\nabla\log\nu\left(\tfrac{x - \sqrt{\lambda_t} Y}{\sqrt{1 - \;\lambda_t}}\right) - \tfrac{1}{\sqrt{\lambda_t}}(\sqrt{\lambda_t}\mathbf{B})\nabla\log\pi(Y)\right].
\end{align*}
Denoting $\mathbf{A} := \mathbf{I} + \sqrt{\lambda_t}\mathbf{B}$, we precisely have the convex combination stated above. \hfill$\square$
\end{remark}
Next we start our analysis from the scalar case. Recall that $\phi^{\alpha \mathbf{I}}_{t,x}(y) = \tfrac{\alpha}{\sqrt{1 - \lambda_t}}{\nabla\log\nu\left(\tfrac{x - \sqrt{\lambda_t}Y}{\sqrt{1 - \lambda_t}}\right)} + \tfrac{1 - \alpha}{\sqrt{\lambda_t}}\nabla\log\pi(Y)$. Having introduced a control variate, we now seek the schedule $\alpha_t$ minimising the variance of score estimates. 

\begin{lemma}\label{lemma:cv_schedule_single_expectation}
    Fix $x \in \mathbb{R}^d$, the $x$-dependent CV schedule $\alpha_{t,x}^\ast\in\mathbb{R}$ minimising the variance of the score estimate is given by
    \begin{align*}
        \alpha_{t,x}^\ast = \argmin_{\alpha\in\mathbb{R}}\mathrm{Var}_{\varrho_{t,x}}[\phi^{\alpha \mathbf{I}}_{t,x}(Y)] = \frac{\mathbb{E}_{\varrho_{t,x}}\left[\left\langle \nabla\log\pi(Y), \nabla\log\varrho_{t,x}(Y)\right\rangle\right]}{\mathbb{E}_{\varrho_{t,x}}\left[\lVert \nabla\log\varrho_{t,x}(Y)\rVert^2\right]}.
    \end{align*}
\end{lemma}
\begin{proof}

Denote $s_1(x,y):= \tfrac{1}{\sqrt{1 - \lambda_t}}\nabla\log\nu\left(\tfrac{x-\sqrt{\lambda_t}y}{\sqrt{1 - \lambda_t}}\right)$ and $s_2(y):=\tfrac{1}{\sqrt{\lambda_t}}\nabla\log\pi\left(y\right)$. The variance is given by
\begin{align*}
     \mathrm{Var}_{\varrho_{t,x}}[\phi_{t,x}^{\alpha\mathbf{I}}(Y)]= {\mathbb{E}_{\varrho_{t,x}}[\lVert \alpha s_1(x, Y) + (1 - \alpha) s_2(Y) \rVert^2]} - \lVert\nabla\log\mu_t(x)\rVert^2.
\end{align*}
Ignoring the additive constant, we define the objective
\begin{align*}
    \mathcal{L}_{t,x}^{\text{Var}}(\alpha) ={\mathbb{E}_{\varrho_{t,x}}[\lVert \alpha s_1(x, Y) + (1 - \alpha) s_2(Y) \rVert^2]},
\end{align*}
and want to find the minimiser $\alpha_{t,x}^\ast = {\argmin_{\alpha\in\mathbb{R}}}\ \mathcal{L}_{t,x}^{\text{Var}}(\alpha)$. Rearranging the objective in terms of $\alpha$, we get
\begin{align*}
    \mathcal{L}_{t,x}^{\text{Var}}(\alpha)
    &= \alpha^2\mathbb{E}_{\varrho_{t,x}}\left[\lVert s_1(x, Y) - s_2(Y)\rVert^2 \right]\\
    &+ \alpha\mathbb{E}_{\varrho_{t,x}}\left[2\langle s_2(Y), s_1(x,Y) - s_2(Y)\rangle\right] + \mathbb{E}_{\varrho_{t,x}}\left[\lVert s_2(Y)\rVert^2\right].
\end{align*}
Taking the derivative with respect to $\alpha$, we have
\begin{align*}
    \tfrac{\partial}{\partial a}\mathcal{L}_{t,x}^{\text{Var}}(\alpha) = 2\alpha \mathbb{E}_{\varrho_{t,x}}[\lVert s_1(x, Y) - s_2(Y)\rVert^2] + 2\mathbb{E}_{\varrho_{t,x}}[\lVert s_2(Y), s_1(x,Y) - s_2(Y)\rVert^2],
\end{align*}
and equating it to zero yields the stationary solution
\begin{equation*}
    \alpha_{t,x}^\ast = \frac{\mathbb{E}_{\varrho_{t,x}}\left[\langle -s_2(Y), s_1(x,Y) - s_2(Y)\rangle\right]}{\mathbb{E}_{\varrho_{t,x}}\left[\lVert s_1(x,Y) - s_2(Y)\rVert^2\right]}.
\end{equation*}
As the objective has a positive second derivative everywhere
\begin{align*}
    \tfrac{\partial^2}{\partial \alpha^2}\mathcal{L}_{t,x}^{\text{Var}}(\alpha) &= 2\mathbb{E}_{\varrho_{t,x}}\left[\lVert s_1(x,Y) - s_2(Y)\rVert^2\right]\\
    &=2\mathbb{E}_{\varrho_{t,x}}[\lVert\nabla\log\varrho_{t,x}(Y)\rVert^2]> 0,
\end{align*}
this is indeed a global minimiser. Rewriting $-s_1(x,y) + s_2(y) = \tfrac{1}{\sqrt{\lambda_t}}\nabla\log\varrho_{t,x}(y)$, we have
\begin{align*}
    \alpha_{t,x}^\ast = \frac{\mathbb{E}_{\varrho_{t,x}}\left[\left\langle\nabla\log\pi(Y), \nabla\log\varrho_{t,x}(Y)\right\rangle\right]}{\mathbb{E}_{\varrho_{t,x}}\left[\lVert \nabla\log\varrho_{t,x}(Y)\rVert^2\right]}.
\end{align*}
\end{proof}

As we have auxiliary variables for estimating expectations with respect to $\varrho_{t,x}$ and evaluate the target score at their positions when propagating them, we can estimate the optimal schedule above at no additional cost. This schedule, however, is a function of $x$. Being mindful of estimation in practice, it may be beneficial to predict a single time-varying schedule to be used for all the samples. A natural choice is to minimise the expectation of the objective with respect to $\mu_t$, i.e. find $\alpha_{t,x}^\ast = \argmin_{\alpha\in\mathbb{R}}\mathbb{E}_{\mu_t}[\mathcal{L}_{t,X}^{\mathrm{Var}}(\alpha)]$.

To do this, we first show a result allowing us to rewrite joint expectations w.r.t. the path marginal and the posterior into expectations with respect to the endpoint distributions. This is made possible due to the hierarchical structure of how the samples and auxiliary variables are defined.

\begin{lemma}\label{lemma:double_expectation_results}
    For sufficiently smooth test functions $\phi$, we have that
    \begin{align}
        \mathbb{E}_{\mu_t}[\mathbb{E}_{\varrho_{t,X}}[\langle\phi(Y), \nabla\log\varrho_{t,X}(Y)\rangle]] &= \mathbb{E}_{\pi}[\langle\phi(X),\nabla\log\pi(X)\rangle]\label{eq:y_dot_product_double_expectation},\\
        \mathbb{E}_{\mu_t}\left[\mathbb{E}_{\varrho_{t,X}}\left[\left\langle\phi\left(\tfrac{X-\sqrt{\lambda_t}Y}{\sqrt{1 - \lambda_t}}\right), \nabla\log\varrho_{t,X}(Y)\right\rangle\right]\right] &= -\sqrt{\tfrac{\lambda_t}{1 - \lambda_t}}\mathbb{E}_{\nu}[\langle\phi(X), \nabla\log\nu(X)\rangle].\label{eq:diff_dot_product_double_expectation}
    \end{align}
\end{lemma}
\begin{proof}
    These results follow after invoking Stein's lemma and marginalising out a variable. We first show \eqref{eq:y_dot_product_double_expectation}. Denote the LHS by $\Phi_1$. Using Stein's lemma in the inner expectation, we have
    \begin{align*}
        \Phi_1
        &= \mathbb{E}_{\mu_t}\left[\mathbb{E}_{\varrho_{t,X}}\left[\langle \phi(Y), \nabla\log\varrho_{t,X}(Y)\rangle\right]\right] = \mathbb{E}_{\mu_t}\left[\mathbb{E}_{\varrho_{t,X}}\left[\langle \nabla, -\phi(Y)\rangle\right]\right].
    \end{align*}
    As the inner expression only depends on $Y$, we can marginalise out $X$ as follows
    \begin{align*}
        \Phi_1 &= \iint \mu_t(x)\varrho_{t,x}(y)\langle \nabla, -\phi(y)\rangle \mathrm{d} y \mathrm{d} x = \iint\tfrac{1}{\sqrt{1 - \lambda_t}^d}\nu\left(\tfrac{x - \sqrt{\lambda_t}y}{\sqrt{1 - \lambda_t}}\right)\pi\left(y\right)\langle \nabla, -\phi(y) \rangle \mathrm{d} x\mathrm{d} y\\
        &= \int\pi(y)\langle \nabla, -\phi(y) \rangle \int \tfrac{1}{\sqrt{1-\lambda_t}^d}\nu\left(\tfrac{x - \sqrt{\lambda_t} y}{\sqrt{1 - \lambda_t}}\right) \mathrm{d} x \mathrm{d} y\\
        &= \int\pi(y)\langle\nabla, -\phi(y)\rangle \mathrm{d} y = \mathbb{E}_{\pi}[\langle\nabla, -\phi(X)\rangle].
    \end{align*}
    Invoking Stein's lemma once more, we precisely have
    \begin{align*}
        \Phi_1 &= \mathbb{E}_\pi [\langle\phi(X), \nabla\log\pi(X)\rangle].
    \end{align*}
    For \eqref{eq:diff_dot_product_double_expectation}, denote the LHS by $\Phi_2$. We similarly have
    \begin{align*}
        \Phi_2 &=\mathbb{E}_{\mu_t}\left[\mathbb{E}_{\varrho_{t,X}}\left[\left\langle\phi\left(\tfrac{X - \sqrt{\lambda_t} Y}{\sqrt{1 - \lambda_t}}\right), \nabla\log\varrho_{t,X}(Y)\right\rangle\right]\right]\\
        &= \mathbb{E}_{\mu_t}\left[\mathbb{E}_{\varrho_{t,X}}\left[\left\langle \nabla, -\phi\left(\tfrac{X - \sqrt{\lambda_t}Y}{\sqrt{1 - \lambda_t}}\right)\right\rangle\right]\right]\\
        &= \iint \mu_t(x)\varrho_{t,x}(y)\left\langle\nabla,-\phi\left(\tfrac{x - \sqrt{\lambda_t}y}{\sqrt{1 - \lambda_t}}\right)\right\rangle \mathrm{d} y \mathrm{d} x\\
        &= \iint\tfrac{1}{\sqrt{1 - \lambda_t}^d}\nu\left(\tfrac{x - \sqrt{\lambda_t}y}{\sqrt{1 - \lambda_t}}\right)\pi\left(y\right)\left\langle\nabla,-\phi\left(\tfrac{x - \sqrt{\lambda_t}y}{\sqrt{1 - \lambda_t}}\right)\right\rangle \mathrm{d} y \mathrm{d} x.
    \end{align*}
    By a change of variables $y = \tfrac{1}{\sqrt{\lambda_t}}(x-\sqrt{1 - \lambda_t}u)$, we have $\mathrm{d} y = \sqrt{\tfrac{1 - \lambda_t}{\lambda_t}}^d \mathrm{d} u$ and, for any divergence, $\langle \nabla_y, f(y)\rangle = -\sqrt{\tfrac{\lambda_t}{{1 - \lambda_t}}}\left\langle \nabla_u, f\left(\tfrac{1}{\sqrt{\lambda_t}}(x - \sqrt{1 - \lambda_t}u)\right) \right\rangle$. This results in
    \begin{align*}
        \Phi_2 &= \iint\nu\left(u\right)\pi\left(\tfrac{x - \sqrt{1-\lambda_t}u}{\sqrt{\lambda_t}}\right)\left(-\sqrt{\tfrac{\lambda_t}{1 - \lambda_t}}\right)\langle\nabla, -\phi(u)\rangle \tfrac{1}{\sqrt{\lambda_t}^d}\mathrm{d} u\mathrm{d} x\\
        &= -\sqrt{\tfrac{\lambda_t}{1 - \lambda_t}}\int \nu(u) \langle\nabla, -\phi(u)\rangle \int \tfrac{1}{\sqrt{\lambda_t}^d}\pi\left(\tfrac{x - \sqrt{1-\lambda_t}u}{\sqrt{\lambda_t}}\right)\mathrm{d} x\mathrm{d} u\\
        &= -\sqrt{\tfrac{\lambda_t}{1 - \lambda_t}}\int\nu(u)\langle\nabla, -\phi(u)\rangle \mathrm{d} u =\sqrt{\tfrac{\lambda_t}{1 - \lambda_t}}\mathbb{E}_{\nu}\left[\langle\nabla, \phi(X)\right].
    \end{align*}
    Using Stein's lemma, we finally have
    \begin{align*}
        \Phi_2 = -\sqrt{\tfrac{\lambda_t}{1 - \lambda_t}}\mathbb{E}_{\nu}[\langle\phi(X), \nabla\log\nu(X)\rangle].
    \end{align*}
\end{proof}

Now, we minimise the variance in expectation, resulting in a position-independent schedule. It admits a simpler form, revealing a clear dependence on the endpoint distributions as well as the schedule $\lambda_t$.
\subsubsection{Optimal Scalar CV Schedule (Proof of Proposition~\ref{prop:optimal_scalar_cv_schedule})}\label{app:proof:optimal_scalar_cv_schedule}
We restate the proposition for clarity and provide its proof.
\propOptimalScalarCVSchedule*
\begin{proof}
Denote $s_1(x,y):= \tfrac{1}{\sqrt{1 - \lambda_t}}\nabla\log\nu\left(\tfrac{x-\sqrt{\lambda_t}y}{\sqrt{1 - \lambda_t}}\right)$ and $s_2(y):=\tfrac{1}{\sqrt{\lambda_t}}\nabla\log\pi\left(y\right)$. Similar to Lemma \ref{lemma:cv_schedule_single_expectation}, we can define the objective
\begin{align*}
    \mathcal{L}_{t,x}^{\text{Var}}(\alpha) ={\mathbb{E}_{\varrho_{t,x}}[\lVert \alpha s_1(x, Y) + (1 - \alpha) s_2(Y) \rVert^2]}.
\end{align*}
Furthermore, we can follow the workings in Lemma \ref{lemma:cv_schedule_single_expectation} to identically minimise this objective in expectation and yield
\begin{align*}
    \alpha_t^\ast = \argmin_{\alpha\in\mathbb{R}} \mathbb{E}_{\mu_t}\left[\mathcal{L}_{t,X}^{\text{Var}}(\alpha)\right]
    = \frac{\mathbb{E}_{\mu_t}\left[\mathbb{E}_{\varrho_{t,X}}\left[\left\langle \nabla\log\pi(Y), \nabla\log\varrho_{t,X}(Y)\right\rangle\right]\right]}{\mathbb{E}_{\mu_t}\left[\mathbb{E}_{\varrho_{t,X}}\left[\lVert \nabla\log\varrho_{t,X}(Y)\rVert^2\right]\right]}.
\end{align*}
Using results from Lemma \ref{lemma:double_expectation_results}, for test functions $\phi_1(y)=\nabla\log\pi(y)$ and $\phi_2(y)=\sqrt{\tfrac{\lambda_t}{1-\lambda_t}}\nabla\log\nu(y)$, we have that
\begin{align*}
    \mathbb{E}_{\mu_t}\left[\mathbb{E}_{\varrho_{t,X}}\left[\left\langle\phi_1(Y), \nabla\log\varrho_{t,X}(Y)\right\rangle\right]\right] &= \mathbb{E}_\pi[\langle\phi_1(X), \nabla\log\pi(X)\rangle]\\
    &=\mathbb{E}_\pi[\lVert\nabla\log\pi(X)\rVert^2],\\
    \mathbb{E}_{\mu_t}\left[\mathbb{E}_{\varrho_{t,X}}\left[\left\langle\phi_2\left(\tfrac{X - \sqrt{\lambda_t}Y}{\sqrt{1 - \lambda_t}}\right), \nabla\log\varrho_{t,X}(Y)\right\rangle\right]\right] &= -\sqrt{\tfrac{\lambda_t}{1 - \lambda_t}}\mathbb{E}_\nu[\langle\phi(X), \nabla\log\nu(X)\rangle]\\
    &=-\tfrac{\lambda_t}{1 - \lambda_t}\mathbb{E}_{\nu}[\lVert\nabla\log\nu(X)\rVert^2].
\end{align*}
Since $\nabla\log\varrho_{t,x}(y) = -\phi_2\left(\tfrac{x-\sqrt{\lambda_t}y}{\sqrt{1 - \lambda_t}}\right) + \phi_1(y)$, it immediately follows that
\begin{align*}
    \alpha_t^\ast =\frac{\mathbb{E}_\pi [\lVert\nabla\log\pi(X)\rVert^2]}{\tfrac{\lambda_t}{1 - \lambda_t}\mathbb{E}_\nu\left[\lVert\nabla\log\nu(X)\rVert^2\right] + \mathbb{E}_\pi [\lVert\nabla\log\pi(X)\rVert^2]}.
\end{align*}
Rewriting to reflect the path structure and noting $\mathrm{Var}_p[\nabla\log p(X)]=\mathbb{E}_p[\lVert\nabla\log p(X)\rVert^2]$, we have
\begin{align*}
    \alpha_t^\ast
    &= \frac{\tfrac{1}{\lambda_t}\mathrm{Var}_{\pi}[\nabla\log\pi(X)]}{\tfrac{1}{1-\lambda_t}\mathrm{Var}_{\nu}[\nabla\log\nu(X)]+\tfrac{1}{\lambda_t}\mathrm{Var}[\nabla\log\pi(X)]}.
\end{align*}
\end{proof}

A key advantage of estimating such a schedule is we can combine all our auxiliary variables across several samples to yield a larger sample size and a more stable schedule altogether. In practice, we can use the empirical distribution formed by our propagated samples in place of $\mu_t$. 

\subsubsection{Optimal Diagonal CV Schedule}\label{app:proof:optimal_diagonal_cv_schedule}
Now, we consider schedules that are restricted to be diagonal. We have the following result that describes a similar optimum. 
\begin{proposition}
    \label{prop:diagonal_cv_schedule}
    The diagonal CV schedule $\mathbf{a}_t^\ast = (a_{t,1}^{\ast}, \ldots, a_{t,d}^\ast)^\top\in\mathbb{R}^d$, minimising the total variance in expectation, i.e.
    \begin{alignat*}{2}
        \mathbf{a}_t^\ast &= {\argmin_{\mathbf{a}\in\mathbb{R}^d}}\ \mathbb{E}_{X \sim \mu_t}[\mathrm{Var}_{Y\sim \varrho_{t,X}}[\phi_{t,X}^{\mathrm{diag}(\mathbf{a})}(Y)]],
    \end{alignat*}
    has its $i$th component equal to
    \begin{align*}
        (a_{t}^{\ast})_i = \frac{\tfrac{1}{\lambda_t}\mathbb{E}_\pi[(\nabla\log\pi(X))_i^2]}{\tfrac{1}{1 - \lambda_t}\mathbb{E}_{\nu}[(\nabla\log\nu(X))_i^2]+\tfrac{1}{\lambda_t}\mathbb{E}_\pi[(\nabla\log\pi(X))_i^2]}.
    \end{align*}
\end{proposition}
\begin{proof}
Take $\mathbf{a}=(a_1,\ldots,a_d)^\top\in\mathbb{R}^d$. Denote $s_1(x,y):= \tfrac{1}{\sqrt{1 - \lambda_t}}\nabla\log\nu\left(\tfrac{x-\sqrt{\lambda_t}y}{\sqrt{1 - \lambda_t}}\right)$ and $s_2(y):=\tfrac{1}{\sqrt{\lambda_t}}\nabla\log\pi\left(y\right)$. Let us further denote the following, whose expectation is equivalent to the variance objective up to an additive constant
\begin{align*}
    \mathcal{L}_{t}^{\mathrm{Norm}}(x,y,\mathbf{A}) :=\lVert \phi_{t,x}^{\mathbf{A}}(y)\rVert^2 = \lVert \mathbf{A}s_1(x, y) + (\mathbf{I} - \mathbf{A})s_2(y)\rVert^2.
\end{align*}

Since $\mathbf{A}$ is diagonal, we can rewrite it to decouple the contributions of the diagonal components
\begin{align*}
    \mathcal{L}_t^{\text{Norm}}(x,y,\mathrm{diag}(\mathbf{a}))
    &= \sum_{i=1}^{d}{\left[a_i s_1^i(x, y) +(1 - a_i)s_2^i(y)\right]^2},
\end{align*}
where $s_1^i$ and $s_2^i$ are the $i$th components of $s_1$ and $s_2$. Solving for a stationary solution to the original objective, we take the partial derivatives component-wise and move them inside the expectation, i.e. for the $i$th component,
\begin{align*}
    &\tfrac{\partial}{\partial a^i}{\mathbb{E}_{\mu_t}\left[\mathbb{E}_{\varrho_{t,X}}\left[\mathcal{L}_{t}^{\text{Norm}}(X, Y, \text{diag}(\mathbf{a}))\right]\right]}\\
    &= \mathbb{E}_{\mu_t}[\mathbb{E}_{\varrho_{t,X}}[2\left[a_i s_1^i(X, Y) + (1 - a_i) s_2^i(Y))\right](s_1^i(X, Y) - s_2^i(Y))]]\\
    &= 2a_i\mathbb{E}_{\mu_t}[\mathbb{E}_{\varrho_{t,X}}[(s_1^i(X, Y) - s_2^i(Y))^2]]+ 2\mathbb{E}_{\mu_t}[\mathbb{E}_{\varrho_{t,X}}[s_2^i(Y)(s_1^i(X, Y) - s_2^i(Y))]].
\end{align*}
Setting this to zero, we arrive at the stationary solution
\begin{equation*}
    (a_{t}^\ast)_i = \frac{\mathbb{E}_{\mu_t}[\mathbb{E}_{\varrho_{t,X}}[-s_2^i(Y)(s_1^i(X, Y) - s_2^i(Y))]]}{\mathbb{E}_{\mu_t}[\mathbb{E}_{\varrho_{t,X}}[(s_1^i(X, Y) - s_2^i(Y))^2]]}.
\end{equation*}
Define for now $\phi_1(u) = \tfrac{1}{\sqrt{1 - \lambda_t}}(\nabla\log\nu\left(u\right))_i\mathbf{e}_i$, so we have $\phi_1\left(\tfrac{x - \sqrt{\lambda_t}y}{\sqrt{1 - \lambda_t}}\right) = s_1^i(x,y)\mathbf{e}_i$. Define also $\phi_2(u) = \tfrac{1}{\sqrt{\lambda_t}}(\nabla\log\pi(u))_i\mathbf{e}_i = s_2^i(u)\mathbf{e}_i$. Since $s_1(x,y)-s_2(y) = -\tfrac{1}{\sqrt{\lambda_t}}\nabla\log\varrho_{t,x}(y)$, we can invoke Lemma \ref{lemma:double_expectation_results} for test function $\phi_1$ to get
\begin{align*}
     &\mathbb{E}_{\mu_t}[\mathbb{E}_{\varrho_{t,X}}[s_1^i(X,Y)(s_1^i(X, Y) - s_2^i(Y))]]\\ &= \mathbb{E}_{\mu_t}[\mathbb{E}_{\varrho_{t,X}}[\langle s_1^i(X,Y)\mathbf{e}_i, s_1(X, Y) - s_2(Y)\rangle]]\\
     &= -\tfrac{1}{\sqrt{\lambda_t}}\mathbb{E}_{\mu_t}\left[\mathbb{E}_{\varrho_{t,X}}\left[\left\langle \phi_1\left(\tfrac{X - \sqrt{\lambda_t} Y}{\sqrt{1 - \lambda_t}}\right), \nabla\log\varrho_{t,X}(Y)\right\rangle\right]\right]\\
     &= \tfrac{1}{\sqrt{\lambda_t}}\sqrt{\tfrac{\lambda_t}{1 - \lambda_t}}\mathbb{E}_{\nu}[\langle\phi_1(X), \nabla\log\nu(X)\rangle] = \tfrac{1}{1 - \lambda_t}\mathbb{E}_{\nu}[(\nabla\log\nu(X))_i^2],
\end{align*}
and similarly for test function $\phi_2$ to get
\begin{align*}
     &\mathbb{E}_{\mu_t}[\mathbb{E}_{\varrho_{t,X}}[s_2^i(Y)(s_1^i(X, Y) - s_2^i(Y))]]\\
     &= \mathbb{E}_{\mu_t}[\mathbb{E}_{\varrho_{t,X}}[\langle s_2^i(Y)\mathbf{e}_i, s_1(X, Y) - s_2(Y)\rangle]]= -\tfrac{1}{\sqrt{\lambda_t}}\mathbb{E}_{\mu_t}\left[\mathbb{E}_{\varrho_{t,X}}\left[\left\langle \phi_2\left(Y\right), \nabla\log\varrho_{t,X}(Y)\right\rangle\right]\right]\\
     &= -\tfrac{1}{\sqrt{\lambda_t}}\mathbb{E}_{\pi}\left[\left\langle \phi_2\left(X\right), \nabla\log\pi(X)\right\rangle\right] = -\tfrac{1}{{\lambda_t}}\mathbb{E}_\pi[(\nabla\log\pi(X))_i^2].
\end{align*}
Hence the $i$th component of this stationary solution is given by
\begin{align*}
    (a_{t}^\ast)_i = \frac{\tfrac{1}{{\lambda_t}}\mathbb{E}_\pi[(\nabla\log\pi(X))_i^2]}{\tfrac{{1}}{1 - \lambda_t}\mathbb{E}_{\nu}[(\nabla\log\nu(X))_i^2]+\tfrac{1}{{\lambda_t}}\mathbb{E}_\pi[(\nabla\log\pi(X))_i^2]}.
\end{align*}
Inspecting the second-order partial derivatives, we have that
\begin{align*}
    \tfrac{\partial^2}{\partial a_i^2}{\mathbb{E}_{\mu_t}\left[\mathbb{E}_{\varrho_{t,X}}\left[\mathcal{L}_{t}^{\text{Norm}}(X,Y,\mathrm{diag}(\mathbf{a}))\right]\right]}
    &= 2\mathbb{E}_{\mu_t}[\mathbb{E}_{\varrho_{t,X}}[(s_1^i(X, Y) - s_2^i(Y))^2]]\\
    &= \tfrac{{2}}{1 - \lambda_t}\mathbb{E}_{\nu}[(\nabla\log\nu(X))_i^2]+\tfrac{2}{{\lambda_t}}\mathbb{E}_\pi[(\nabla\log\pi(X))_i^2]> 0,
\end{align*}
and $\tfrac{\partial^2}{\partial a_i\partial a_j}\mathbb{E}_{\mu_t}\left[\mathbb{E}_{\varrho_{t,X}}\left[\mathcal{L}_{t}^{\text{Norm}}(X,Y,\mathrm{diag}(\mathbf{a}))\right]\right] = 0$ for $i\neq j$ for all $\mathbf{a}\in\mathbb{R}^d$. Hence, the Hessian is positive definite everywhere, given that it is diagonal with positive eigenvalues. We can conclude that the objective is convex and the stationary solution $\mathbf{a}_{t}^\ast$ is indeed a global minimiser.
\end{proof}

\subsubsection{Optimal Matrix CV Schedule (Proof of Proposition~\ref{prop:optimal_matrix_cv_schedule})}\label{app:proof:optimal_matrix_cv_schedule}
Now, we consider arbitrary matrix schedules. We restate the proposition and provide its proof.
\propOptimalCVSchedule*
\begin{proof}
    Take $\mathbf{A} \in \mathbb{R}^{d\times d}$ and let $A_{ij}$ be its $(i, j)$ entry. Denote $s_1(x,y):= \tfrac{1}{\sqrt{1 - \lambda_t}}\nabla\log\nu\left(\tfrac{x-\sqrt{\lambda_t}y}{\sqrt{1 - \lambda_t}}\right)$ and  $s_2(y):=\tfrac{1}{\sqrt{\lambda_t}}\nabla\log\pi\left(y\right)$. Let us refer to the following, whose expectation is equivalent to the variance objective up to an additive constant
\begin{align*}
    \mathcal{L}_{t}^{\mathrm{Norm}}(x,y,\mathbf{A}) :=\lVert \phi_{t,x}^{\mathbf{A}}(y)\rVert^2 = \lVert \mathbf{A}s_1(x, y) + (\mathbf{I} - \mathbf{A})s_2(y)\rVert^2.
\end{align*}
We can rewrite it as follows to show the contributions of the individual entries of the matrix
    \begin{align*}
        \mathcal{L}_{t}^{\text{Norm}}(x,y,\mathbf{A})
        &= \sum_{l=1}^{d}{\left(s_2(y)+\mathbf{A} \left(s_1(x, y) - s_2(y)\right)\right)_l^2}\\
        &= \sum_{l=1}^{d}{\left(s_2^l(y)+\sum_{k=1}^{d}A_{lk}\left(s_1^k(x,y) -s_2^k(y)\right) \right)^2},
    \end{align*}
    where $s_1^i$ and $s_2^i$ are the $i$th components of $s_1$ and $s_2$. Taking the derivative w.r.t. $A_{ij}$, we have
    \begin{align*}
        \tfrac{\partial}{\partial A_{ij}}\mathcal{L}_t^{\text{Norm}}(x, y, \mathbf{A})
        &= 2\left(s_2^i(y)+\sum_{k=1}^{d}{A_{ik}(s_1^k(x, y) -s_2^k(y))}\right)(s_1^j(x, y)-s_2^j(y))\\
        &= 2s_2^i(y)(s_1^j(x, y)-s_2^j(y))\\
        &+2\sum_{k=1}^{d}{A_{ik}(s_1^k(x, y)-s_2^k(y))(s_1^j(x, y)-s_2^j(y))}.
    \end{align*}
    Now, we can consider the actual objective and take the derivative w.r.t. $A_{ij}$
    \begin{align*}
        \tfrac{\partial}{\partial A_{ij}}\mathbb{E}_{\mu_t}[\mathbb{E}_{\varrho_{t,X}}[\mathcal{L}_t^{\text{Norm}}(X,Y,\mathbf{A})]]
        &= \mathbb{E}_{\mu_t}\left[\mathbb{E}_{\varrho_{t,X}}\left[\tfrac{\partial}{\partial A_{ij}}\mathcal{L}_{t}^{\text{Norm}}(X,Y,\mathbf{A})\right]\right]\\
        &=2\mathbb{E}_{\mu_t}[\mathbb{E}_{\varrho_{t,X}}[s_2^i(Y)(s_1^j(X, Y)-s_2^j(Y))]]\\ &+ 2\sum_{k=1}^d{A_{ik}\mathbb{E}_{\mu_t}[\mathbb{E}_{\varrho_{t,X}}[s_1^k(X, Y)-s_2^k(Y))(s_1^j(X, Y)-s_2^j(Y)]]}.
    \end{align*}
    Define $\phi_1(u) = \tfrac{1}{\sqrt{1 - \lambda_t}}(\nabla\log\nu\left(u\right))_i\mathbf{e}_j$, so we have $\phi_1\left(\tfrac{x - \sqrt{\lambda_t}y}{\sqrt{1 - \lambda_t}}\right) = s_1^i(x,y)\mathbf{e}_j$, and also $\phi_2(u) = \tfrac{1}{\sqrt{\lambda_t}}(\nabla\log\pi(u))_i\mathbf{e}_j = s_2^i(u)\mathbf{e}_j$. Using the fact that $s_1(x,y)-s_2(y) = -\tfrac{1}{\sqrt{\lambda_t}}\nabla\log\varrho_{t,x}(y)$, we can invoke Lemma \ref{lemma:double_expectation_results} for test function $\phi_1$ to get
    \begin{align*}
        &\mathbb{E}_{\mu_t}[\mathbb{E}_{\varrho_{t,X}}[s_1^i(X,Y)(s_1^j(X, Y) - s_2^j(Y))]]\\
        &=\mathbb{E}_{\mu_t}[\mathbb{E}_{\varrho_{t,X}}[\langle s_1^i(X,Y)\mathbf{e}_j, s_1(X, Y) - s_2(Y)\rangle]]\\
        &=-\tfrac{1}{\sqrt{\lambda_t}}\mathbb{E}_{\mu_t}\left[\mathbb{E}_{\varrho_{t,X}}\left[\left\langle \phi_1\left(\tfrac{X - \sqrt{\lambda_t} Y}{\sqrt{1 - \lambda_t}}\right), \nabla\log\varrho_{t,X}(Y)\right\rangle\right]\right]\\
        &=\tfrac{1}{\sqrt{\lambda_t}}\sqrt{\tfrac{\lambda_t}{1 - \lambda_t}}\mathbb{E}_{\nu}[\langle\phi_1(X), \nabla\log\nu(X)\rangle]=\tfrac{1}{1 - \lambda_t}\mathbb{E}_{\nu}[(\nabla\log\nu(X))_i(\nabla\log\nu(X))_j],
    \end{align*}
    and similarly for test function $\phi_2$ to get
    \begin{align*}
        &\mathbb{E}_{\mu_t}[\mathbb{E}_{\varrho_{t,X}}[s_2^i(Y)(s_1^j(X, Y) - s_2^j(Y))]]\\
        &=\mathbb{E}_{\mu_t}[\mathbb{E}_{\varrho_{t,X}}[\langle s_2^i(Y)\mathbf{e}_j, s_1(X, Y) - s_2(Y)\rangle]] = -\tfrac{1}{\sqrt{\lambda_t}}\mathbb{E}_{\mu_t}\left[\mathbb{E}_{\varrho_{t,X}}\left[\left\langle \phi_2\left(Y\right), \nabla\log\varrho_{t,X}(Y)\right\rangle\right]\right]\\
        &=-\tfrac{1}{\sqrt{\lambda_t}}\mathbb{E}_{\pi}\left[\left\langle \phi_2\left(X\right), \nabla\log\pi(X)\right\rangle\right] =-\tfrac{1}{{\lambda_t}}\mathbb{E}_\pi[(\nabla\log\pi(X))_i(\nabla\log\pi(X))_j].
    \end{align*}
    Note that we can rewrite 
    \begin{align*}
        \mathbb{E}_p[(\nabla\log p(X))_i(\nabla\log p(X))_j] = (\mathbb{E}_p[\nabla\log p(X)\nabla\log p(X)^\top])_{ij} = (\mathcal{I}_p)_{ij},
    \end{align*}
    and we have the simplified equalities
    \begin{align*}
        \mathbb{E}_{\mu_t}[\mathbb{E}_{\varrho_{t,X}}[s_1^i(X,Y)(s_j^j(X, Y) - s_2^j(Y))]] &= \tfrac{1}{1 - \lambda_t}(\mathcal{I}_\nu)_{ij},\\
        \mathbb{E}_{\mu_t}[\mathbb{E}_{\varrho_{t,X}}[s_2^i(Y)(s_1^j(X, Y) - s_2^j(Y))]] &= - \tfrac{1}{\lambda_t}(\mathcal{I}_\pi)_{ij}.
    \end{align*}
    We therefore have
    \begin{align*}
        \tfrac{\partial}{\partial A_{ij}}\mathbb{E}_{\mu_t}[\mathbb{E}_{\varrho_{t,X}}[\mathcal{L}_{t}^{\text{Norm}}(X,Y,\mathbf{A})]]
        &=-\tfrac{2}{\lambda_t}(\mathcal{I}_\pi)_{ij} + 2\sum_{k=1}^{d}{A_{ik}\left(\tfrac{1}{1 - \lambda_t}(\mathcal{I}_{\nu})_{kj} + \tfrac{1}{\lambda_t}(\mathcal{I}_\pi)_{kj}\right)}\\
        &= -\tfrac{2}{\lambda_t}(\mathcal{I}_\pi)_{ij} + 2\left(\mathbf{A}\left(\tfrac{1}{1 - \lambda_t}\mathcal{I}_{\nu} + \tfrac{1}{\lambda_t}\mathcal{I}_{\pi}\right)\right)_{ij}.
    \end{align*}
    Setting this to zero, the stationary solution must satisfy
    \begin{align*}
        \left(\mathbf{A}_{t}^\ast\left(\tfrac{1}{1 - \lambda_t}\mathcal{I}_{\nu} + \tfrac{1}{\lambda_t}\mathcal{I}_{\pi}\right)\right)_{ij} =\left(\tfrac{1}{\lambda_t}\mathcal{I}_\pi\right)_{ij}.
    \end{align*}
    Since this holds for all entries $(i,j)$, we therefore have that
    \begin{align}
        \mathbf{A}_{t}^\ast\left(\tfrac{1}{1 - \lambda_t}\mathcal{I}_{\nu} + \tfrac{1}{\lambda_t}\mathcal{I}_{\pi}\right) = \tfrac{1}{\lambda_t}\mathcal{I}_\pi \iff \mathbf{A}_{t}^{\ast}& = \tfrac{1}{\lambda_t}\mathcal{I}_{\pi}\left(\tfrac{1}{1 - \lambda_t}\mathcal{I}_{\nu} + \tfrac{1}{\lambda_t}\mathcal{I}_{\pi}\right)^{-1}.
    \end{align}
    Inspecting the second-order partial derivatives, we have
    \begin{align*}
        \tfrac{\partial^2}{\partial A_{ij}\partial A_{ik}}\mathbb{E}_{\mu_t}[\mathbb{E}_{\varrho_{t,X}}[\mathcal{L}_{t}^{\text{Norm}}(X,Y,\mathbf{A})]]
        &=2\mathbb{E}_{\mu_t}[\mathbb{E}_{\varrho_{t,X}}[s_1^k(X, Y)-s_2^k(Y))(s_1^j(X, Y)-s_2^j(Y)]]\\
        &=\tfrac{2}{1 - \lambda_t}(\mathcal{I}_{\nu})_{kj} + \tfrac{2}{\lambda_t}(\mathcal{I}_{\pi})_{kj} = 2\left(\tfrac{1}{1 - \lambda_t}\mathcal{I}_{\nu} + \tfrac{1}{\lambda_t}\mathcal{I}_{\pi}\right)_{kj},
    \end{align*}
    and $\tfrac{\partial^2}{\partial A_{ij}\partial{A_{lk}}}\mathbb{E}_{\mu_t}[\mathbb{E}_{\varrho_{t,X}}[\mathcal{L}_{t}^{\text{Norm}}(X,Y,\mathbf{A})]] = 0$ when $i \neq l$. Let us denote by $\mathbf{a}=(a_1,\ldots,a_{d^2})^\top \in \mathbb{R}^{d^2}$, the flattened representation of $\mathbf{A}$, with $a_{d(i-1)+j} = A_{ij}$. It follows then that our Hessian is a block diagonal matrix
    \begin{align*}
        \mathbf{H}(\mathbf{a}) = \left(\frac{\partial^2 \mathbb{E}_{\mu_t}[\mathbb{E}_{\varrho_{t,X}}[\mathcal{L}_{t}^{\text{Norm}}(X,Y,\mathbf{A})]]}{\partial a_i \partial a_j}\right) = \begin{bmatrix}
            \tilde{\mathbf{H}} & 0 & 0 &  0 \\
            0 & \tilde{\mathbf{H}}& 0 & 0 \\
            0 & 0 & \ddots & 0 \\
            0 & 0 & 0  & \tilde{\mathbf{H}}
        \end{bmatrix},
    \end{align*}
    where the identical blocks are given by
    \begin{align*}
        \tilde{\mathbf{H}} = \tfrac{2}{1 - \lambda_t}\mathcal{I}_\nu + \tfrac{2}{\lambda_t}\mathcal{I}_\pi.
    \end{align*}
    As the score covariances are positive definite, their linear combination in $\tilde{\mathbf{H}}$ is also positive definite and has positive eigenvalues. The block diagonal $\mathbf{H}(\mathbf{a})$ inherits the eigenvalues of its blocks, which are all positive, and hence the Hessian is positive definite everywhere. We can conclude that our objective is convex and the stationary solution $\mathbf{A}_{t}^\ast$ is indeed the global minimiser.
\end{proof}

%% file: appendix/d_diffusion_path_design.tex
\section{Diffusion Path Details}\label{app:sec:diffusion_path_details}

In this section, we point out the biasedness of \gls{dald}, how we can use this to inform our choice for the $\lambda_t$ schedule, and stochastic interpolants dynamics for inducing the diffusion path.

\subsection{Diffusion Annealed Langevin Dynamics}\label{sec_app:action_analysis}
It is worth noting that, even when \gls*{dald} \eqref{eq:diffusion_annealed_langevin_dynamics} is simulated without discretisation error, it remains biased, since the marginal distributions of the solution to the corresponding SDE do not exactly coincide with $\mu_{t}$, in contrast to the stochastic interpolants formulation proposed by \citet{albergo2025stochastic}.
However, this bias can be quantified. In particular, bounds have been established in \citet[Theorem 1]{guo2025provable} and \citet[Theorem A.3]{cordero-encinar2025nonasymptotic}. We state the result below.
\begin{theorem}\label{theorem:preliminaries_continuous_time_kl_bound}
    Let \emph{$\mathbb{P}_{\text{DALD}} = (p_{t,\text{DALD}})_{t\in[0, 1]}$} be  the path measure of the diffusion annealed Langevin dynamics \eqref{eq:diffusion_annealed_langevin_dynamics}, and $\mathbb{P}=({\mu}_{ t})_{t\in[0, 1]}$ that of a reference {SDE} such that the marginals at each time have distribution ${\mu}_{t}$. If \emph{$ p_{0, \text{DALD}}= \mu_0$}, the KL divergence between the path measures is given by
    \emph{\begin{equation*}     \mathrm{KL}(\hat\mu_T\;\|\;p_{T,\text{DALD}})\le \mathrm{KL}(\mathbb{P}\;\|\;\mathbb{P}_{\text{DALD}}) = \frac{\epsilon}{4}\mathcal{A}(\mu).
    \end{equation*}}
\end{theorem}
The proof follows from an application of Girsanov’s theorem \citep{karatzas1991brownian}, we refer to the aforementioned works for details. Following \citet[Lemma 3.3]{cordero-encinar2025nonasymptotic}, we can obtain an upper bound for the action of the diffusion path \eqref{eq:diffusion_path_marginal}. We outline the details below.

First, consider the reparametrised version of $\mu_t$ in terms of the schedule $\lambda_t$, denoted as $\Tilde{\mu}_\lambda$ and let $X_{\lambda}\sim\Tilde{\mu}_{\lambda}$ and $X_{\lambda + \delta}\sim\Tilde{\mu}_{\lambda + \delta}$.
Recall that 
\begin{equation}\label{appendix:eq_conv_path_random_variables}
    X_{\lambda} = \sqrt{\lambda} X + \sqrt{1-\lambda}\sigma Z 
\end{equation}
where $X\sim\pi$ and $Z\sim \mathcal{N}(0, \mathbf{I})$ are independent random variables. 
The Wasserstein-2 distance between $\Tilde{\mu}_\lambda$ and $\Tilde{\mu}_{\lambda+\delta}$ is given by
\begin{align*}
     W_2^2(\Tilde{\mu}_\lambda, \Tilde{\mu}_{\lambda+\delta}) &\leq \mathbb{E}\left[\Vert X_\lambda-X_{\lambda + \delta}\Vert^2\right]\\
&=\mathbb{E}\left[\left\Vert(\sqrt{\lambda}-\sqrt{\lambda+\delta})X\right\Vert^2\right] + \mathbb{E}\left[\left\Vert\left(\sqrt{1-\lambda}-\sqrt{1-\lambda-\delta}\right)\sigma Z\right\Vert^2\right]\\
&= (\sqrt{\lambda}-\sqrt{\lambda+\delta})^2 \mathbb{E}\left[\Vert X\Vert^2\right] + \left(\sqrt{1-\lambda}-\sqrt{1-\lambda-\delta}\right)^2\sigma^2d.
\end{align*}
Using the definition of the metric derivative we have
\begin{align*}
    \left\vert\Dot{\Tilde{\mu}}\right\vert_\lambda^2 = \lim_{\delta\to 0}\frac{ W_2^2(\Tilde{\mu}_\lambda, \Tilde{\mu}_{\lambda+\delta})}{\delta^2} \leq \frac{\mathbb{E}[\Vert X\Vert^2]}{4\lambda} + \frac{\sigma^2d}{4(1-\lambda)}.
\end{align*}
Therefore, we have the following expression for the action
\begin{align*}
    \mathcal{A}_{\lambda}(\mu) &= \int_0^1 \vert\dot{\mu}\vert_t^2 \mathrm{d} t = \int_0^1 \left\vert\dot{\tilde{\mu}}\right\vert_\lambda^2\left\lvert\dot{\lambda}_t\right\rvert^2 \mathrm{d} t \\
    &\le \int_0^1 \left(\frac{\mathbb{E}[\Vert X\Vert^2]}{4 \lambda_t} + \frac{\sigma^2d}{4(1-\lambda_t)}\right){\dot{\lambda}_t^2}\mathrm{d} t 
\end{align*}
This expression will be used later to derive a principled schedule $\lambda_t$ that minimises this upper bound.

\subsection{Schedule Minimising Action Upper Bound}\label{subsec_app:action_minimising_schedule}

We want to find the curve $\lambda = (\lambda_t)_{t\in[0,1]}$ that minimises the upper bound of the action functional. In Theorem \ref{theorem:preliminaries_continuous_time_kl_bound}, we were able to bound the action by a simpler functional. The following result indicates its minimiser.

\begin{proposition}
    The cosine schedule $\lambda_t = \sin^2\left(\frac{\pi t}{2}\right)$ is the unique minimiser of the functional
    \begin{align*}
        J[\lambda]= \int_0^1\left( \frac{\sigma^2d}{4(1 - \lambda_t)} +\frac{\mathbb{E}_\pi[\lVert X\rVert^2]}{4\lambda_t}\right) \dot{\lambda}_t^2\mathrm{d}t.
    \end{align*}
\end{proposition}
\begin{proof}
Define the Lagrangian as
\begin{equation*}
    L(t, \lambda_t, \dot{\lambda}_t) = \dot{\lambda}_t^2\left(\frac{\sigma^2 d}{4(1 - \lambda_t)} + \frac{\mathbb{E}_\pi[\lVert X\rVert^2]}{4\lambda_t}\right),
\end{equation*}
then any stationary points of the functional $J[\lambda] = \int_0^1 L(t, \lambda_t, \dot{\lambda}_t)dt$ must satisfy the Euler-Lagrange equation
\begin{equation*}
    \frac{\partial L}{\partial \lambda_t}  - \frac{d}{dt}\frac{\partial L}{\partial \dot{\lambda}_t} = 0.
\end{equation*}
We drop the time index in writing $\lambda$ and $\dot{\lambda}$ for conciseness. We have that
\begin{align*}
    \frac{\partial L}{\partial \lambda} &= \dot{\lambda}^2\left(-\frac{\mathbb{E}_\pi[\lVert X\rVert^2]}{4\lambda^2} + \frac{\sigma^2 d}{4(1 - \lambda)^2}\right)\\
    -\frac{d}{dt}\frac{\partial L}{\partial \dot{\lambda}} &= -\frac{d}{dt} \left( 2\dot{\lambda}\left(\frac{\mathbb{E}_\pi[\lVert X\rVert^2]}{4\lambda} + \frac{\sigma^2 d}{4(1 - \lambda)}\right)\right)\\
    &=-2\ddot{\lambda}\left(\frac{\mathbb{E}_\pi[\lVert X\rVert^2]}{4\lambda} + \frac{\sigma^2 d}{4(1 - \lambda)}\right) - 2\dot{\lambda}^2\left(-\frac{\mathbb{E}_\pi[\lVert X\rVert^2]}{4\lambda^2} + \frac{\sigma^2 d}{4(1 - \lambda)^2}\right).
\end{align*}
Equating their sum to zero, we can solve for $\ddot{\lambda}$
\begin{align*}
    \ddot{\lambda} &= -\frac{1}{2}\dot{\lambda}^2\frac{-\frac{\mathbb{E}_\pi[\lVert X\rVert^2]}{4\lambda^2} + \frac{\sigma^2 d}{4(1 - \lambda)^2}}{\frac{\mathbb{E}_\pi[\lVert X\rVert^2]}{4\lambda}+\frac{\sigma^2 d}{4(1 - \lambda)}}\\
    &= \frac{1}{2}\dot{\lambda}^2 \frac{(1-\lambda)^2\mathbb{E}_\pi\left[\lVert X\rVert^2\right] - \lambda^2\sigma^2d}{\lambda(1 - \lambda)^2 \mathbb{E}_\pi\left[\lVert X\rVert^2\right] + \lambda^2(1 - \lambda) \sigma^2 d}.
\end{align*}
We denote the ratio $\sigma^2 d / \mathbb{E}_\pi\left[\lVert X\rVert^2\right] =: R$ and rewrite the above as
\begin{align*}
    \ddot{\lambda} = \frac{1}{2}\dot{\lambda}^2 \frac{(1 - \lambda)^2 - R\lambda^2}{\lambda(1 - \lambda)^2 + R\lambda^2(1 - \lambda)},
\end{align*}
As there are no $t$ terms, we can use the substitution $\dot{\lambda} = v(\lambda)$, which yields $\ddot{\lambda} = \dot{v}v$. We have
\begin{align*}
    \dot{v}(\lambda)v(\lambda) &= \frac{1}{2}v(\lambda)^2 \frac{(1 - \lambda)^2 - R\lambda^2}{\lambda(1 - \lambda)^2 + R\lambda^2(1 - \lambda)},
\end{align*}
and rearranging reveals a first-order differential equation
\begin{equation*}
    \dot{v}(\lambda) + v(\lambda)\underbrace{\left[-\frac{(1 - \lambda)^2 - R\lambda^2}{2\lambda(1 - \lambda)^2 + 2R\lambda^2(1 - \lambda)}\right]}_{p(\lambda)} = 0.
\end{equation*}
which we can solve using an integrating factor
\begin{align*}
    I(\lambda) &=  \exp\left(\int p(\lambda) \mathrm{d}\lambda\right) = \exp\left(\int \frac{R\lambda^2 - (1 - \lambda)^2}{2\lambda(1 - \lambda)((R-1)\lambda +1)} \mathrm{d}\lambda\right)\\
    &= \exp\left(\int \frac{1}{2}\left(\frac{R-1}{(R-1)\lambda + 1} - \frac{1}{\lambda} + \frac{1}{1 - \lambda}\right) \mathrm{d}\lambda\right)\\
    &= \exp\left({\frac{1}{2}(\log \lvert(R - 1)\lambda + 1\rvert - \log\lvert \lambda\rvert - \log \lvert 1 -\lambda \rvert)}\right)
    = \sqrt{\frac{\lvert(R-1)\lambda+1\rvert}{\lambda(1-\lambda)}}.
\end{align*}
This leads to
\begin{align*}
    \frac{\mathrm{d}}{\mathrm{d} t}\left(v(\lambda) I(\lambda)\right) = 0 \iff v(\lambda)I(\lambda) = c_1.
\end{align*}
Since $\dot{\lambda} = v(\lambda)$, the above is a separable first-order differential equation
\begin{align*}
    \frac{\mathrm{d}\lambda}{\mathrm{d} t} I(\lambda) = c_1
    \iff \int I(\lambda) \mathrm{d}\lambda &= \int c_1 \mathrm{d} t\\
    \int \sqrt{\frac{\lvert(R-1)\lambda + 1\rvert}{\lambda(1 - \lambda)}} \mathrm{d}\lambda &= c_1 t + c_2
\end{align*}
For arbitrary $R>0$, the above integral does not appear to have a closed-form solution. Crucially, however, when $R = 1$, i.e. $\sigma^2 = \mathbb{E}_\pi\left[\lVert X\rVert^2\right]/d$, then we have
\begin{align*}
    c_1 t + c_2 = \int \sqrt{\frac{1}{\lambda(1 - \lambda)}}\mathrm{d}\lambda = \arcsin(2\lambda - 1).
\end{align*}
The boundary conditions, $\lambda_0=0$ and $\lambda_1=1$, yield $c_2 = -\pi/2$ and $c_1 1 + c_2 = \pi / 2$, respectively, so $c_1 = \pi$ and we arrive at the solution
\begin{equation*}
\lambda_t= \frac{1 + \sin\left({\pi t} -\frac{\pi}{2}\right)}{2} = \frac{1 - \cos\left({\pi t}\right)}{2} = \sin^2\left(\frac{\pi t}{2}\right).
\end{equation*}

To conclude, we show that the functional $J[\lambda]$ is convex, which should imply the stationary point we found is indeed a global minimiser. We check below that for any $\lambda_1, \lambda_2\in C^1[0, 1]$ and any $\alpha\in [0,1]$, it holds 
$$
J[\alpha\lambda_1 + (1-\alpha)\lambda_2]\leq \alpha J[\lambda_1] + (1-\alpha)  J[\lambda_2].
$$
For simplicity, we remove the constants and work with
$$
J[\lambda] = \int_0^1 \frac{\dot\lambda_t^2}{\lambda_t(1-\lambda_t)} \mathrm{d} t.
$$
Consider the multivariate function $L(\lambda, \dot\lambda) = \frac{\dot\lambda^2}{\lambda(1-\lambda)}$, where $\lambda\in[0,1]$ and $\dot\lambda\in\mathbb{R}$. The Hessian of this function is given by
\begin{equation*}
    \nabla^2 L = 2\begin{pmatrix}
     \dot{\lambda}^2\frac{\lambda(1 - \lambda)+(1-2\lambda)^2}{\lambda^3(1-\lambda)^3}   & -\dot{\lambda}\frac{1-2\lambda}{\lambda^2(1-\lambda)^2}\\
    -\dot{\lambda}\frac{1-2\lambda}{\lambda^2(1-\lambda)^2}    &\frac{1}{\lambda(1-\lambda)}
    \end{pmatrix}.
\end{equation*}
It is immediate to check that the determinant is
$$
\det \nabla^2 L = \frac{4\dot{\lambda}^2}{\lambda^3(1-\lambda)^3}\geq 0,
$$
and $(\nabla^2 L)_{22}\geq 0$, therefore, $\nabla^2 L$ is positive semidefinite and hence $L(\lambda, \dot\lambda)$ is jointly convex in $(\lambda, \dot{\lambda})$. Using \citet[Theorem 10.7.1]{vanbrunt2003the} this leads to 
\begin{align*}
J[\alpha\lambda_1 + (1-\alpha)\lambda_2] &=\int_0^1 L(\alpha\lambda_1 + (1-\alpha)\lambda_2, \alpha\dot\lambda_1 + (1-\alpha)\dot\lambda_2) \, \mathrm{d}t\\
& \leq \int_0^1 \alpha L(\lambda_1, \dot\lambda_1) + (1-\alpha) L(\lambda_2, \dot\lambda_2) \,\mathrm{d}t\\
& = \alpha J[\lambda_1] + (1-\alpha)  J[\lambda_2],
\end{align*}
which concludes that $J[\lambda]$ is convex and hence $\lambda_t = \sin^2\left(\frac{\pi t}{2}\right)$ is a minimiser.
\end{proof}

In fact, we now have the following upper bound for the action
\begin{corollary}\label{cor:action_upper_bound}
    Under the choice $\sigma^2 = \mathbb{E}_\pi[\lVert X\rVert^2] / d$ and the cosine schedule $\lambda_t = \sin^2\left(\frac{\pi t}{2}\right)$, the action is bounded by
    \begin{align*}
        \mathcal{A}(\mu)\leq \frac{\pi^2}{4}\mathbb{E}_\pi\left[\lVert X\rVert^2\right].
    \end{align*}
\end{corollary}

\subsection{Stochastic Interpolants Dynamics for the Diffusion Path}\label{app:subsec:stochastic_interpolants_dynamics}

Stochastic interpolants \citep{albergo2025stochastic} are a general framework for characterising flow- and diffusion-based methods. We follow their framework and consider the stochastic interpolant $X_t = I(t,X_0,X_1)$ corresponding to the same process discussed in the main text, i.e.
\begin{align*}
    I(t,X_0,X_1) = \sqrt{1-\lambda_t}X_0 +\sqrt{\lambda_t}X_1,
\end{align*}
with boundary conditions $\lambda_0=0$ and $\lambda_1=1$ and random variables $X_0\sim\nu$ and $X_1\sim\pi$ independent. As such, $\mathrm{Law}(X_t)=\mu_t$, and we are interested in dynamics that induces this stochastic process. From \citet[Theorem 6]{albergo2025stochastic}, $X_t$ has a density w.r.t. Lebesgue that satisfies the continuity equation
\begin{equation*}
    \partial_t\mu_t(x) + \langle\nabla, \mu_t(x)b_t(x)\rangle = 0,
\end{equation*}
where $b_t:\mathbb{R}^d\rightarrow\mathbb{R}^d$ is the velocity defined by
\begin{equation*}
    b_t(x) = \mathbb{E}[\partial_t I(t,X_0,X_1)\mid X_t=x].
\end{equation*}
We have the time-derivative
\begin{equation*}
    \partial_t I(t,X_0,X_1) = -\tfrac{\dot{\lambda}_t}{2\sqrt{1-\lambda_t}}X_0 + \tfrac{\dot{\lambda}_t}{2\sqrt{\lambda_t}}X_1,
\end{equation*}
and, hence, conditioned on $X_t=x$,
\begin{align*}
    b_t(x)
    &= \frac{\int \left[ -\frac{\dot{\lambda}_t}{2\sqrt{1-\lambda_t}}\left(\frac{x-\sqrt{\lambda}_t x_1}{\sqrt{1-\lambda}_t}\right) + \frac{\dot{\lambda}_t}{2\sqrt{\lambda_t}}x_1\right]\nu\left(\frac{x-\sqrt{\lambda}_t x_1}{\sqrt{1-\lambda}_t}\right)\pi(x_1)\mathrm{d}x_1}{\int \nu\left(\frac{x-\sqrt{\lambda}_t x_1}{\sqrt{1-\lambda}_t}\right)\pi(x_1)\mathrm{d}x_1}\\
    &=\tfrac{\dot{\lambda}_t}{2}\int\left(\tfrac{\sqrt{\lambda_t}x_1-x}{1-\lambda_t}+\tfrac{x_1}{\sqrt{\lambda_t}}\right)\varrho_{t,x}(x_1)\mathrm{d}x_1 = \tfrac{\dot{\lambda}_t}{\lambda_t}\int\left(\tfrac{\sqrt{\lambda_t}x_1-x}{1-\lambda_t}+\tfrac{{x_1}/{\sqrt{\lambda_t}}-\sqrt{\lambda_t}x_1}{1-\lambda_t}\right)\varrho_{t,x}(x_1)\mathrm{d}x_1\\
    &=\tfrac{\dot{\lambda}_t}{2\lambda_t}\int\left(x+\tfrac{\sqrt{\lambda_t}x_1-x}{1-\lambda_t}\right)\varrho_{t,x}(x_1)\mathrm{d}x_1
    = \tfrac{\dot{\lambda}_t}{2\lambda_t}\left(x+ \sigma^2\nabla\log\mu_t(x)\right),
\end{align*}
where we used DSI at the end. Furthermore, we have, for any $\gamma_t \geq 0$,
\begin{equation*}
    \partial_t\mu_t(x) = -\langle\nabla, \mu_t(x)(b_t(x)+\gamma_t\nabla\log\mu_t(x))\rangle + \gamma_t\Delta \mu_t(x).
\end{equation*}
The above Fokker-Planck equation precisely describes the density evolution of the stochastic process $(X_t^\text{F})_{t\in[0,1]}$ governed by the SDE
\begin{equation*}
    \mathrm{d}X_t^\text{F} = \left[\tfrac{\dot{\lambda}_t}{2\lambda_t}X_t^\text{F} + \left(\sigma^2\tfrac{\dot{\lambda}_t}{2\lambda_t} + \gamma_t\right)\nabla\log\mu_t(X_t^\text{F})\right]\mathrm{d}t + \sqrt{2\gamma_t}\mathrm{d}W_t,\quad X_0^\text{F}\sim\nu.
\end{equation*}
With $X_t$ and $X_t^{\text{F}}$ coinciding in their marginal law, we indeed have that \eqref{eq:general_diffusion_path_sde} induces the diffusion path.

%% file: appendix/e_implementation.tex
\section{DPSMC Implementation}\label{sec_app:dpsmc_implementation}

In this section, we briefly outline how DPSMC is implemented in practice. Our base algorithm in Algorithm \ref{alg:diffusion_path_sampling_smc} is presented as a generic, idealised version with non-interacting particles. However, as we do not know certain quantities, such as the target score covariance and the right step size for auxiliary variables, we estimate them using our ensemble of samples and auxiliary variables. Algorithm \ref{alg:dpsmc_interacting} depicts the scheme we use in practice. Below we detail our sampler choices.

\begin{algorithm}[!p]
    \caption{DPSMC (Interacting) \hlmath{with Normalising Constant Estimation}}
    \label{alg:dpsmc_interacting}
    \begin{algorithmic}[1]
        \INPUT base variance $\sigma^2$, initial aux. distribution $q_{0}$, particle count $N_Y$, sample count $N_X$, disc. step count $K$, \gls*{ll} step count $N_\text{LL}$, \gls*{ll} step size $h_{\text{LL}}$, ESS threshold $\kappa$, SDE forward kernel $\mathsf{K}_{k}^{\text{sde}}$, and \hlmath{SDE backward kernel $\mathsf{L}_{k}^\text{sde}$}
        \FOR{$i=1,\ldots, N_{X}$}
            \STATE $X_0^i \sim \mathcal{N}(0, \sigma^2\mathbf{I})$, $Y^{i,j}_0 \sim q_{0}$ for all $j\in[N_Y]$
            \STATE $W^{i,j}_0 = \tilde{\rho}_0(Y_0^{i,j})/q_{0}(Y_0^{i,j})$ for all $j\in[N_Y]$
            \STATE $w^{i,j}_{0} = {W_0^{i,j}}/{\sum_{l=1}^{N_Y} W_0^{i,l}}$
            \STATE  $S_{0}^i=\sum_{j=1}^{N_Y}{w^{i,j}_{0} \varphi_{0,X_{0}^i}({Y}_{0}^{i, j})}$
            \STATE \hlmath{$M_0^i = \frac{1}{N_Y}\sum_{j=1}^{N_Y}{W^{i,j}_0}$, $M_{-1}^{i}:=1$} {\small\COMMENT{Marginal density estimate}}
            \STATE \hlmath{$\Lambda_0^i = M_0^i/ \mathcal{N}(X_0^i; 0, \sigma^2\mathbf{I})$} {\small\COMMENT{$X$ weight}}
            \STATE \hlmath{$\tau_0^i = 0$} {\small{\COMMENT{Last time resampled}}}
        \ENDFOR
        
        \FOR{$k=1,\ldots,K-1$}
            \FOR{$i=1,\ldots, N_X$}
                \IF{$k< N_{\text{LL}}$}
                    \STATE $X_k^i \sim \mathcal{N}(X_{k-1}^i+h_{\text{LL}}S_{k-1}^i,2h_{\text{LL}}\mathbf{I})${\small\COMMENT{LL}}
                \ELSE
                    \STATE $X_k^i \sim \mathsf{K}_{k}^{\mathrm{sde}}(\cdot \mid X_{k-1}^i, S_{k-1}^i, h)$ {\small\COMMENT{SDE step}}
                \ENDIF
                \FOR{$j=1,\ldots,N_Y$}
                    \STATE $\bar{Y}^{i,j}_{k} \sim \mathsf{K}_{k}^{\text{MALA}}(\cdot \mid Y_{k-1}^{i,j}, X_{k}^i)$ {\small\COMMENT{MALA step, see Alg. \ref{alg:mala_kernel}}}
                    \STATE $W_k^{i,j} = W_{k-1}^{i,j}\frac{\mathcal{N}\left(\sqrt{\lambda_{t_k}}Y_{k-1}^{i,j};X_k^i,\sigma^2(1-\lambda_{t_k})\right)}{\mathcal{N}\left(\sqrt{\lambda_{t_{k-1}}}Y_{k-1}^{i,j};X_{k-1}^i,\sigma^2(1-\lambda_{t_{k-1}})\right)}$
                \ENDFOR
            \ENDFOR
            \STATE $w^{i,j}_{k} = {W_k^{i,j}}/{\sum_{l=1}^{N_Y} W_k^{i,l}}$ for all $i\in[N_X]$, $j\in[N_Y]$
            \STATE $\mathbf{A}_{t_k}=\mathrm{CVSE}(\{w_k^{i,j}\}, \{\bar{Y}_{k}^{i,j}\}, \{X_k^i\})$ {\small\COMMENT{CV schedule estimation, see Alg. \ref{alg:cv_schedule_estimation}}}
            \FOR{$i=1,\ldots, N_X$}
                \STATE  $S_{k}^i=\sum_{j=1}^{N}{w^{i,j}_{k} \phi_{t_k,X_{k}^i}^{\mathbf{A}_{t_k}}(\bar{Y}_{k}^{i,j})}$ {\small\COMMENT{Score estimation}}
                \STATE \hlmath{$M_k^i = M_{\tau_{k-1}^i-1}^i\frac{1}{N_Y}\sum_{j=1}^{N_Y}W_k^{i,j}$}
                \IF{$k < N_{\text{LL}}$}
                    \STATE \hlmath{$\Lambda_k^i = \Lambda_{k-1}^i \frac{M_k^i \mathcal{N}(X_{k-1}^i;X_{k}^i+h_{\text{LL}}S_{k}^i, 2h_{\text{LL}}\mathbf{I})}{M_{k-1}^i \mathcal{N}(X_{k}^i;X_{k-1}^i+h_{\text{LL}}S_{k-1}^i, 2h_{\text{LL}}\mathbf{I})}$}
                \ELSE
                    \STATE \hlmath{$\Lambda_k^i = \Lambda_{k-1}^i\frac{M_k^i \mathsf{L}_{k-1}^\mathrm{sde}(X_{k-1}^i;X_k^i,h)}{M_{k-1}^i \mathsf{K}_{k}^{\mathrm{sde}}(X_{k}^i;X_{k-1}^i, S_{k-1}^i,h)}$}
                \ENDIF
                \IF{$\widehat{\mathrm{ESS}}(\{w^{i,j}_k\}_{i=1}^{N_Y}) < \kappa N_Y$}
                    \STATE $\{Y^{i,j}_{k}\}_{j=1}^{N_Y} \sim \sum_{l=1}^{N_Y} {w_k^{i,l}\delta_{\bar{Y}_{k}^{i,l}}}$ {\small\COMMENT{Resampling}}
                    \STATE $W_k^{i,j} = 1$ for all $j \in[N_Y]$
                    \STATE \hlmath{$\tau_k^i = k$}
                \ELSE
                    \STATE $Y_k^{i,j} = \bar{Y}_{k}^{i,j}$ for all $j\in[N_Y]$
                    \STATE \hlmath{$\tau_k^i = \tau_{k-1}^i$}
                \ENDIF
            \ENDFOR
        \ENDFOR
        \FOR{$i=1,\ldots, N_X$}
            \STATE $X_K^i \sim \mathsf{K}_{K}^{\mathrm{sde}}(\cdot \mid X_{K-1}^i, S_{K-1}^i, h)$ {\small\COMMENT{SDE step}}
        \ENDFOR
        \STATE \hlmath{$\Lambda_K^i = \Lambda_{K-1}^i\frac{\exp(-V_\pi(X_{K}^i)) \mathsf{L}_{k-1}^\mathrm{sde}(X_{K-1}^i;X_K^i,h)}{M_{K-1}^i \mathsf{K}_{k}^{\mathrm{sde}}(X_{K}^i;X_{K-1}^i, S_{K-1}^i,h)}$}
        \STATE \hlmath{$\hat{Z}_\pi = \frac{1}{N_X}\sum_{i=1}^{N_X} \Lambda_K^i$}
        \OUTPUT approximate samples from the target $\{X_K^i\}_{i=1}^{N_X}$ \hlmath{and normalising constant estimate $\hat{Z}_\pi$}
    \end{algorithmic}
\end{algorithm}

\fparagraph{MALA proposal.} Note that $\mathsf{K}_{k}^{\text{MALA}}$ is $\rho_k$-invariant, and under a suitable backwards kernel choice $\mathsf{L}_{k-1}$ as in \eqref{eq:markov_backwards_kernel}, we have a simplified weight update given by the potential
\begin{align*}
    G_k(y_{k-1}, y_k) = \frac{\tilde{\varrho}_{t_k,X_k}(y_{k-1})}{\tilde{\varrho}_{t_{k-1},X_{k-1}}(y_{k-1})} = \frac{\mathcal{N}(\sqrt{\lambda_{t_k}}y_{k-1}; X_k, \sigma^2(1-\lambda_{t_k})\mathbf{I})}{\mathcal{N}(\sqrt{\lambda_{t_{k-1}}}y_{k-1}; X_{k-1}, \sigma^2(1-\lambda_{t_{k-1}})\mathbf{I})}.
\end{align*}
For MALA, we scale the step size geometrically (i.e. by a factor of $1.1$) to maintain an acceptance ratio of 75\%. We choose a larger scaling factor in contrast to existing works using MCMC as we only do a single MALA step per iteration. This is particularly important to allow the proposal to adapt well to the rapidly changing posterior. During \gls*{ll}, we use a larger scale factor of $2.0$ to quickly alleviate poor initial step size choices. The MALA proposal is presented in Algorithm \ref{alg:mala_kernel}.

\begin{algorithm}[H]
    \caption{MALA Kernel}
    \label{alg:mala_kernel}
    \begin{algorithmic}[1]
        \INPUT auxiliary variable $Y_{k-1}$, sample $X_{k}$, step size $h_{k-1}$ (adjusted automatically)
        \STATE $q_{k}(y^\prime \mid y) :=  \mathcal{N}(y^\prime;y + h_{k-1}\nabla\log{\varrho}_{t_k,X_{k}}(y), {2h_{k-1}}\mathbf{I})$\\
        \STATE $\tilde{Y}_{k} \sim q_{k}(\cdot \mid Y_{k-1})$\\
        \STATE $\alpha  = \min\left\{1, \frac{\tilde{\varrho}_{t_k,X_{k}}(\tilde{Y}_{k})q_{k}(Y_{k-1} \mid \tilde{Y}_{k})}{\tilde{\varrho}_{t_k,X_{k}}({Y}_{k-1})q_{k}(\tilde{Y}_{k} \mid {Y}_{k-1})}\right\}$\\
        \STATE $Y_{k} \sim \alpha\delta_{\tilde{Y}_{k}}(\cdot) + (1 - \alpha)\delta_{Y_{k-1}}(\cdot)$\\
        \OUTPUT $Y_{k}$
    \end{algorithmic}
\end{algorithm}

\fparagraph{Control variate schedule estimation.} We use MMCVSI for score estimation. Algorithm \ref{alg:cv_schedule_estimation} describes the process of estimating the matrix schedule using our samples, including the scalar and diagonal variants for completeness. Notably, since the target score covariance is a static unknown, we can reuse estimates from previous time-steps. We maintain a running average in practice. While a matrix inversion is performed, we do not perform any additional matrix regularisation or pseudoinverse. We believe the stability of our matrix schedule is aided by a form of Tikhonov regularisation resulting from the Gaussian score covariance (see line 4 in Algorithm \ref{alg:cv_schedule_estimation}).

\begin{algorithm}[H]
    \caption{CV Schedule Estimation}
    \label{alg:cv_schedule_estimation}
    \begin{algorithmic}[1]
        \INPUT weights $\{w_k^{i,j}\}$, auxiliary variables $\{Y_k^{i,j}\}$, samples $\{X_k^i\}$
        \STATE $\hat{\mathcal{I}}_\pi =\frac{1}{N_x}$ {\scriptsize $\sum_{i=1}^{N_x}\sum_{j=1}^{N_y}{w^{i,j}_{k}\nabla\log\pi(Y_k^{i,j}) \nabla\log\varrho_{t_k,X_k^i}(Y_k^{i,j})^\top}$}
        \STATE $\alpha_{t_k} = \frac{(1-\lambda_{t_k}) \text{Tr}(\hat{\mathcal{I}}_\pi)}{\lambda_{t_k} d / \sigma^2 + (1-\lambda_{t_k}) \text{Tr}(\hat{\mathcal{I}}_\pi)}$\\
        \STATE $\mathbf{a}_{t_k} = \mathrm{diag}\left(\left\{\frac{(1-\lambda_{t_k})(\hat{\mathcal{I}}_\pi)_{ii}}{\lambda_{t_k}/\sigma^2 + (1-\lambda_{t_k})(\hat{\mathcal{I}}_\pi)_{ii}}\right\}_{i=1}^{d}\right)$\\
        \STATE $\mathbf{A}_{t_k} = \hat{\mathcal{I}}_{\pi}\left(\frac{\lambda_{t_k}}{\sigma^2(1-\lambda_{t_k})}\mathbf{I}+\hat{\mathcal{I}}_{\pi}\right)^{-1}$\\
        \OUTPUT $\alpha_{t_k},\ \mathbf{a}_{t_k},\ \mathbf{A}_{t_k}$
    \end{algorithmic}
\end{algorithm}

\fparagraph{Taming terminal instabilities.} While our scheme allows for single \gls*{mala} steps per iteration, auxiliary variables have less time to equilibrate w.r.t. the next target. This can be detrimental when the sample makes large jumps at the end when the posterior contracts to a Dirac centred at the sample. In the absence of a finer discretisation, we turn to a simple heuristic that we find to work well empirically. It turns out that we consistently maintain our target acceptance ratio up to some tolerance, right until the posterior's Gaussian component dominates. We choose to heuristically halt the SMC sampler when the acceptance ratio goes below a threshold, which we set to 10\%, and substitute score estimates with the target score itself. This only occurs in the final few discretisation steps and can be sufficiently avoided when using more than one \gls*{mala} step.

%% file: appendix/f_theoretical_results.tex
\section{Theoretical Analysis}\label{sec_app:theoretical_analysis}

In this section, we provide proofs for Proposition \ref{prop:score_estimation_error} and Theorem \ref{thm:final_error} in the main text.
\subsection{Score Estimation Error (Proof of Proposition~\ref{prop:score_estimation_error})}\label{subsec_app:score_convergence_analysis}

As discussed in the main text, when $\lambda_{t_k}\approx 0$, the posterior is approximately the target, i.e. $\varrho_{t_k, X}\approx \pi$, which appears circular given we aim to draw samples from $\pi$. However, as observed by \citet{he2024zerothorder}, when using DSI for score estimation, it is possible to rely on less accurate samples from $\varrho_{t, X}$ in the Wasserstein-2 sense for small $t$, while still keeping the overall MSE of the Monte Carlo estimator below a given threshold.

We analyse this behaviour for the weighted estimator constructed from the weighted particle system $\{(w_k^i, y_k^i)\}_{i=1}^N$ at the iteration $k$
\begin{equation*}
    S_k^{N}(x) = \mathbb{E}_{Y\sim \varrho_{t_k,x}^N}\left[\phi_{t_k,x}^{\alpha_{t_k}\mathbf{I}}(Y)\right] = \sum_{i=1}^N w_k^i \phi_{t_k,x}^{\alpha_{t_k}\mathbf{I}}(y_k^i),
\end{equation*}
where the test function
\begin{equation}
    \phi_{t_k,x}^{\alpha_{t_k}\mathbf{I}}(y) = \alpha_{t_k} \frac{\sqrt{\lambda_{t_k}}\, y-x}{\sigma^2(1-\lambda_{t_k})} + (1-\alpha_{t_k})\frac{\nabla\log\pi(y)}{\sqrt{\lambda_{t_k}}},
\end{equation}
is the control variate score expression.

We restate Proposition~\ref{prop:score_estimation_error}, which provides an upper bound on the $L^2$ error of the weighted score estimator \eqref{eq:mc_score_estimator_control_variate}, and present its proof.
\scoreEstimationError*
\begin{proof}
Let $\mathcal{G}_k^N=\sigma(X_k,\{(w_k^i,\bar{Y}_k^i)\}_{i=1}^N)$ denote the sigma-field generated by $X_k$ and the weighted particle system. Conditional on $\mathcal{G}_k^N$, the particle measure $\varrho_{t_k,X_k}^N = \sum_{i=1}^{N} w_k^i \delta_{\bar{Y}_k^i}$ is a deterministic probability measure, and so is $\varrho_{t_k,X_k}$, which is indexed by the realised value of $X_k$. Fix a realisation of this sigma-field: write $X_k=x$ and $\varrho_{t_k,x}^N=\sum_{i=1}^N w^i_k\delta_{y^i_k}$. Define the test functions
\begin{align*}
    g_{k,x}(y) := \frac{\sqrt{\lambda_{t_k}}y-x}{\sigma^2(1-\lambda_{t_k})},\quad 
    h_k(y) := \frac{\nabla\log\pi(y)}{\sqrt{\lambda_{t_k}}},\\
    f_{k,x}(y) := \alpha_{t_k}g_{k,x}(y)+(1-\alpha_{t_k})h_k(y)
    = \phi_{t_k,x}^{\alpha_{t_k}\mathbf{I}}(y).
\end{align*}
By the score identity, 
\begin{align*}
    \nabla\log\mu_{t_k}(x)
    =
    \int f_{k,x}(u)\,\varrho_{t_k,x}(\mathrm{d}u),
    \quad
    S_k^N(x)
    =
    \int f_{k,x}(y)\,\varrho_{t_k,x}^N(\mathrm{d}y).
\end{align*}
Let $\Gamma$ be any coupling of $\varrho_{t_k,x}$ and $\varrho_{t_k,x}^N$, then
\begin{align*}
    \left\Vert \nabla\log\mu_{t_k}(x)-S_k^N(x)\right\Vert^2
    &=
    \left\Vert
    \int\left(f_{k,x}(u)-f_{k,x}(y)\right)\Gamma(\mathrm{d}u,\mathrm{d}y)
    \right\Vert^2\\
    &\leq
    \int
    \left\Vert f_{k,x}(u)-f_{k,x}(y)\right\Vert^2
    \Gamma(\mathrm{d}u,\mathrm{d}y),
\end{align*}
where we used Jensen's inequality. Next, we bound the Lipschitz constant of $f_{k,x}$. Since
\begin{equation*}
    \left\Vert g_{k,x}(u)-g_{k,x}(y)\right\Vert
    =
    \frac{\sqrt{\lambda_{t_k}}}{\sigma^2(1-\lambda_{t_k})}\Vert u-y\Vert,
\end{equation*}
and Assumption~\ref{assumption:target_requirements} implies
\begin{equation*}
    \left\Vert h_k(u)-h_k(y)\right\Vert
    \leq
    \frac{L_\pi}{\sqrt{\lambda_{t_k}}}\Vert u-y\Vert,
\end{equation*}
we have, using $\lVert a+b\rVert^2\leq2\lVert a\rVert^2+2\lVert b\rVert^2$,
\begin{equation*}
    \left\Vert f_{k,x}(u)-f_{k,x}(y)\right\Vert^2
    \leq
    2\left(
    \alpha_{t_k}^2\frac{\lambda_{t_k}}{\sigma^4(1-\lambda_{t_k})^2}
    +(1-\alpha_{t_k})^2\frac{L_\pi^2}{\lambda_{t_k}}
    \right)\Vert u-y\Vert^2.
\end{equation*}
Substituting this estimate into the previous display gives
\begin{align*}
    \left\Vert \nabla\log\mu_{t_k}(x)-S_k^N(x)\right\Vert^2
    \leq
    2\left(
    \alpha_{t_k}^2\frac{\lambda_{t_k}}{\sigma^4(1-\lambda_{t_k})^2}
    +(1-\alpha_{t_k})^2\frac{L_\pi^2}{\lambda_{t_k}}
    \right)
    \int \Vert u-y\Vert^2\,\Gamma(\mathrm{d}u,\mathrm{d}y).
\end{align*}
Taking the infimum over all couplings $\Gamma$ of $\varrho_{t_k,x}$ and $\varrho_{t_k,x}^N$ yields, for this fixed realisation,
\begin{align*}
    \left\Vert \nabla\log\mu_{t_k}(x)-S_k^N(x)\right\Vert^2
    \leq
    2\left(
    \alpha_{t_k}^2\frac{\lambda_{t_k}}{\sigma^4(1-\lambda_{t_k})^2}
    +(1-\alpha_{t_k})^2\frac{L_\pi^2}{\lambda_{t_k}}
    \right)
    W_2^2\left(\varrho_{t_k,x},\varrho_{t_k,x}^N\right).
\end{align*}
Since the preceding inequality holds for every realisation of $\mathcal{G}_k^N$ for which the weights define a probability measure, it holds almost surely after substituting $x=X_k$ and $\varrho_{t_k,x}^N=\varrho_{t_k,X_k}^N$. Taking expectations with respect to the joint randomness of $X_k$, the particles, and the weights, and then using the assumed bound on the expected squared Wasserstein error gives
\begin{align*}
   \mathbb{E}\left[\Vert\nabla\log\mu_{t_k}(X_k) - S_k^{N}( X_k)\Vert^2\right]\leq& \,2\, \delta(k, N)\left(\alpha_{t_k}^2\frac{\lambda_{t_k}}{\sigma^4(1-\lambda_{t_k})^2} +(1-\alpha_{t_k})^2\frac{L_\pi^2}{\lambda_{t_k}}\right).
\end{align*}
\end{proof}
Due to the choice of $\alpha_t^\ast$, our estimator exhibits similar behaviour to that of \citet{he2024zerothorder} as $t\to0$. 
In contrast, the setting considered in \citet{he2024zerothorder} (where $\alpha_t^\ast = 1$) requires an arbitrarily large number of samples as the algorithm approaches the target distribution $(t\to 1)$ to maintain the $L^2$ bounded below some threshold. 
In our setting, however, as $t\to 1$, we can take $\alpha_t^\ast\approx 1-\lambda_t$, implying that, in the last expression above, the term inside the parenthesis is of order $\mathcal{O}(1)$. 
This observation is consistent with the fact that DSI suffers from high variance in low-noise regimes.

The bound in Proposition \ref{prop:score_estimation_error} demonstrates that with sufficiently large $N$, we can ensure that
\begin{equation*}
        \sum_{k=0}^{K-1} h\,\mathbb{E}\left[\left\Vert \nabla \log{\mu}_{t_k}(X_{k}) - S_k^{N}(X_{k})\right\Vert^2\right] \leq \varepsilon^2_{\text{score}}.
    \end{equation*}    
Therefore, we can apply different results in the literature to bound the final error of the algorithm.

\subsection{Overall Algorithm Error (Proof of Theorem \ref{thm:final_error})}\label{subsec_app:final_error_proof}

We restate Theorem \ref{thm:final_error} and present its proof.
\finalError*
\begin{proof}
Let $p_1$ denote the law of the algorithm at time $t=1$ for different cases.
\begin{enumerate}
    \item[\emph{(i)}] Using the proof of \citet[Theorem 1]{benton2024nearly} in combination with \citet[eq.~(17)]{benton2024nearly}, we have that
    \begin{align*}
        \mathrm{KL}\left(\pi\;\|\;p_1\right) \leq \mathrm{KL}(\mu_0\;\|\; p_0) + \varepsilon^2_{\text{score}} + \tfrac{d (1 + T)}{K - N_{\text{LL}}} 
    \end{align*}
    While \citet{benton2024nearly} assumes $T \ge 1$ to ensure the convergence error of the forward OU process is sufficiently small, our analysis does not require this since we explicitly account for the initialisation mismatch at $t=0$ through $\text{KL}(\mu_0\;\|\; p_0)$. As the discretisation error of the reverse SDE is independent of the forward mixing time, the result holds for any $T > 0$.
    \item[\emph{(ii)}] Let $\mathbb{P}$ and $\mathbb{Q}$ denote the exact path measure of the continuous-time stochastic interpolant dynamics and the algorithmic path measure of its time-discretised counterpart, respectively. By the data processing inequality, the KL divergence between the final marginals is bounded by the divergence between their path measures$$\text{KL}(\pi\;\|\;p_1) \le \text{KL}(\mathbb{P}\;\|\;\mathbb{Q}).$$Because $\mathbb{P}$ and $\mathbb{Q}$ share the same diffusion coefficient, an application of Girsanov's theorem bounds their divergence by the integrated expected squared difference of their drifts:$$\text{KL}(\mathbb{P}\;\|\;\mathbb{Q}) \lesssim  \int \mathbb{E}_{\mathbb{Q}} \left[ \left\| \nabla \log \mu_t(X_t) - S_k^N(X_{t_k}) \right\|^2 \right] dt.$$Applying the triangle inequality, we decompose this drift error into a time discretisation error and a score estimation error
    \begin{align*}
        \mathbb{E} \left[ \left\| \nabla \log \mu_t(X_t) - S_k^N(X_{t_k}) \right\|^2 \right] \le& 2 \mathbb{E} \left[ \left\| \nabla \log \mu_t(X_t) - \nabla \log \mu_{t_k}(X_{t_k}) \right\|^2 \right] \\&+ 2 \mathbb{E} \left[ \left\| \nabla \log \mu_{t_k}(X_{t_k}) - S_k^N(X_{t_k}) \right\|^2 \right].
    \end{align*}
    By definition, the integrated score estimation error over the $K$ steps is bounded by $\epsilon_{\text{score}}^2$. To bound the time discretisation error, we use the transition relationship $X_t = \alpha_t X_s + \sqrt{1 - \alpha_t^2} Z$ with $\alpha_t = \sqrt{\lambda_t / \lambda_s}$ for $t < s$. Under Assumptions \ref{assumption:target_requirements} and \ref{assumption:target_requirements_2}, the exact score functions $\nabla \log \mu_t$ are $L$-Lipschitz \citep[Lemma 3.2]{cordero-encinar2025nonasymptotic}. We can therefore apply \citet[Lemma 16]{chen2023improved} across the $K$ integration steps to bound the discretisation error
    \begin{equation*}
        \int \mathbb{E} \left[ \left\| \nabla \log \mu_t(X_t) - \nabla \log \mu_{t_k}(X_{t_k}) \right\|^2 \right] dt \lesssim \frac{dL^2}{K}.
    \end{equation*}
    Substituting both bounds into the Girsanov inequality yields the final result
    \begin{equation*}
        \mathrm{KL}(\pi\;\|\;p_1) \le \frac{dL^2}{K} + \epsilon_{\text{score}}^2.
    \end{equation*}
    \item[\emph{(iii)}] This follows from \citet[Theorem 3.4]{cordero-encinar2025nonasymptotic}.
\end{enumerate}
\end{proof}

%% file: appendix/g_experiments.tex
\section{Numerical Experiments}\label{sec_app:numerical_experiments}

\subsection{Methods and Hyperparameters}\label{subsec_app:benchmark_methods}

We detail here the hyperparameters of methods we examine. For a majority of them, we follow the settings and hyperparameter grids outlined in \citet{grenioux2024stochastic} when evaluating on targets that were previously considered. This includes incorporating information about the target, by way of $R$ and $\tau$, to each sampler. We recall their definition as coming from the following assumption.

\begin{assumption}[Log concavity outside a compact]
    There exists $R > 0$ and $\tau > 0$ such that $\pi$ is the convolution of $\mu$ and $\mathcal{N}(0, \tau^2\mathbf{I})$, where $\mu$ is a distribution compactly supported on $B=\mathrm{B}(\mathbb{E}_\pi[X], R\sqrt{d})$, i.e. $\mu(\mathbb{R}^{d}\setminus B) = 0$.
\end{assumption}

As our specific instantiation of GMM40 does not appear in \citet{grenioux2024stochastic} and may be different from other setups, we test every combination of hyperparameters in the defined grid of each sampler and take the best result. Unless otherwise specified, samplers with MALA kernels use an adaptive step size that is geometrically adjusted to maintain an acceptance ratio of $75\%$. Additionally, we set the following parameters fixed globally across all samplers (except RDMC) and benchmark targets:
\begin{itemize}
    \item Number of discretisation steps: $K = 1024$.
    \item Maximum number of energy evaluations (per sample): $N_{\text{evals}}= 1.32 \times 10^5$.
\end{itemize}

\fparagraph{AIS and SMC.} Both AIS and SMC operate on the geometric path, i.e. $\rho_k(x) \propto \rho_0(x)^{1 - \beta_k}\pi(x)^{\beta_k}$ for $k\in\{1,\ldots, K\}$, where we have a Gaussian initial density $\rho_0(x) := \mathcal{N}(0, (R^2d + \tau^2)\mathbf{I})$ and a linear annealing schedule $\beta_k = k/K$. Transitions between the path marginals are made via MALA transition kernels. Specifically, each sample runs 128 MALA steps per iteration. In SMC, the samples are adaptively resampled according to their importance weights whenever the ESS drops below half the number of particles. We find this performs better than the original settings in \citep{grenioux2024stochastic} where they use fewer MALA steps and resample at every time step.

\fparagraph{RDMC.} Here, the final time $T$ in the (forward) OU process is treated as a hyperparameter. Following \citep{huang2024reverse}, we search for its values in the grid: $\{-\log(x)\mid x\in\{0.99,0.95,0.9,0.8,0.7\}\}$. We explain briefly how RDMC is implemented. To start, samples are drawn from $\mathcal{N}(0, (1-\exp({-2T}))\mathbf{I})$ and 16 iterations of \gls*{ll} are performed. Here, and at every iteration afterwards, 128 samples are drawn from the Gaussian component of the posterior to form an IS estimate of the posterior mean. This serves as the initialisation of four chains that run 32 MALA steps, the first half of which are considered part of a warm-up and discarded. All combined, the 64 samples (auxiliary variables) are used to estimate the posterior mean and hence the score. Note that this totals to $(16 + 1024) \times (128 + 4 \times 32) = 2.66\times 10^5$ energy evaluations. While this is more than the imposed budget, we stick to this formulation so we can choose the optimal hyperparameters from \citep{grenioux2024stochastic} for benchmarks other than GMM40. Its chosen settings across benchmarks are listed in Table \ref{tab:rdmc_hyperparameters}.

\begin{table}[h]
    \centering
    \caption{RDMC hyperparameters across the benchmark targets.}
    \setlength{\tabcolsep}{2.5pt}
    {\begin{tabular}{lcccccc}
        \toprule
        & GMM40 & GMM40 ($d=50$) & Rings & Funnel & Ionosphere & Sonar \\
        \midrule
        Final OU Time ($T$) & -$\log(0.70)$ & -$\log(0.80)$ & -$\log(0.80)$ & -$\log(0.90)$ & -$\log(0.95)$ & -$\log(0.95)$\\
         \bottomrule
    \end{tabular}
    \label{tab:rdmc_hyperparameters}}
\end{table}

\fparagraph{SLIPS.} We use the classic variant of SLIPS in our experiments. Similar to RDMC, it undergoes \gls*{ll} for the same number of iterations and estimates the posterior mean using four chains taking 32 MALA steps, half being warm-up steps. However, instead of using IS estimates to initialise the chains, it uses the positions of the chains in the previous time step. Furthermore, SLIPS starts at some later time $t_0$ for which both the path marginals and posterior distribution are assumed to be approximately log-concave. For GMM40, we perform a search along the grids $\eta = \{5.0, 5.7\}$ and $t_0 = \{{0.1, 0.2, 0.4, 1.0, 1.2}\}$, similar to the grid search suggested for mixtures in \citep{grenioux2024stochastic}. Its chosen settings across benchmarks are listed in Table \ref{tab:slips_hyperparameters}.

\begin{table}[h]
    \centering
    \caption{SLIPS hyperparameters across the benchmark targets.}
    
    \setlength{\tabcolsep}{2.5pt}
    {\begin{tabular}{lcccccc}
        \toprule
        & GMM40 & GMM40 ($d=50$) & Rings & Funnel & Ionosphere & Sonar\\
        \midrule
        Log SNR Threshold ($\eta$) & $5.0$ & $5.7$ & $4.6$ & $5.0$ & $5.0$  & $5.0$\\
        Start Time ($t_0$) & $1.20$ & $0.10$ & $1.20$ & $1.00$ & $0.03$ & $0.03$\\
         \bottomrule
    \end{tabular}}
    \label{tab:slips_hyperparameters}
\end{table}

\fparagraph{DPSMC (OU).} We set the Gaussian variance $\sigma^2=R^2 d +\tau^2$, similar to AIS and SMC. Given we do not start at a Gaussian, we can search for a better time to initialise. We search for $\lambda_{0} = \exp(-2T)$ uniformly in the grid $\{1/6,2/6,\ldots,5/6\}$, discretising the unit interval. We believe this is more robust compared to the grid in RDMC, as we search along the path bridging the base Gaussian to the target rather than a heuristic range of values. The chosen values are listed in Table \ref{tab:dpsmc_ou_hyperparameters}. Notably, we get similar values with RDMC.

\begin{table}[h]
    \centering
    \caption{DPSMC (OU) hyperparameters across the benchmark targets.}
    
    \setlength{\tabcolsep}{2.5pt}
    {\begin{tabular}{lcccccc}
        \toprule
        & GMM40 & GMM40 ($d=50$) & Rings & Funnel & Ionosphere & Sonar\\
        \midrule
        Initial Schedule for \gls*{ll} ($\lambda_0$) & $3/6$ & $5/6$ & $3/6$ & $3/6$ & $5/6$ & $5/6$\\
        \midrule
        Final OU Time ($T$) & $0.3465$ & $0.0911$ & $0.3465$ & $0.3465$ & $0.0911$ & $0.0911$ \\
        \bottomrule
    \end{tabular}}
    \label{tab:dpsmc_ou_hyperparameters}
\end{table}

\fparagraph{DPSMC (\gls*{dald}).} Following our findings in Appendix \ref{app:sec:diffusion_path_details}, we set the base Gaussian variance to be $\sigma^2 = \mathbb{E}_\pi[\lVert X\rVert^2]/d$ and the schedule to be $\lambda_t=\sin^2(\frac{\pi t}{2})$. \gls*{dald} has $\epsilon$ as a hyperparameter. We choose it as follows. Corollary \ref{cor:action_upper_bound} indicates the following upper bound on the action
\begin{align*}
    \mathcal{A}(\mu) \leq \frac{\pi^2}{4}\mathbb{E}_\pi[\lVert X\rVert^2],
\end{align*}
which is achieved precisely under our choices of $\sigma^2$ and $\lambda_t$. Now, using this result in the workings of \citet[Theorem 3.4]{cordero-encinar2025nonasymptotic}, we can write
\begin{align*}
    \mathrm{KL}(\mathbb{P}\;\|\;\mathbb{Q}) \lesssim \epsilon\left(1 + \frac{L^2}{K^2\epsilon^4}\right)\frac{\pi^2}{4}\mathbb{E}_\pi[\lVert X\rVert^2] + \frac{dL^2}{K\epsilon^2}\left(1 + \frac{L}{K\epsilon}\right) + \varepsilon^2_{\text{score}},
\end{align*}
where $\mathbb{P}$ is the reference path measure, $\mathbb{Q}$ is the path measure of the continuous-interpolation of \gls*{dald}, $L$ is the supremum of the Lipschitz constants of $\nabla\log\mu_t$ along the path, and $\varepsilon^2_{\text{score}}$ is the score error. Loosely, we choose $\epsilon$ such that it minimises the first two terms. Taking the derivative with respect to $\epsilon$ and setting to zero, we have
\begin{align*}
    \frac{\pi^2}{4}\left(1-\frac{3L^2}{K^2\epsilon^4}\right)\mathbb{E}_\pi[\lVert X\rVert^2] - \frac{2dL^2}{K\epsilon^3} - \frac{3dL^3}{K^2\epsilon^4} = 0.
\end{align*}
Assuming the $O(1/K^2)$ terms vanish, e.g. for large enough $K$, we then have
\begin{align*}
    \frac{\pi^2}{4}\mathbb{E}_\pi[\lVert X\rVert^2] - \frac{2dL^2}{K\epsilon^3} = 0 \iff \frac{1}{\epsilon} = \left(\frac{K\mathbb{E}_\pi[\lVert X\rVert^2] \pi^2}{8dL^2}\right)^{1/3}.
\end{align*}
We use this crude approximation and set $1/\epsilon = \xi \left({K\mathbb{E}_\pi[\lVert X\rVert^2]}/{d}\right)^{1/3}$, for some hyperparameter $\xi$ which we find in the grid
\begin{align*}
    \left\{2^{{\left[-2.5\cdot(1-\frac{x}{10}) + 3.5\cdot \frac{x}{10}\right]}} \mid x \in \{0,1,\ldots, 10\}\right\},
\end{align*}
i.e. linearly interpolating in log-space. Notably, \gls*{dald} is typically presented as running the dynamics
\begin{equation*}
    \mathrm{d}X_t = \nabla\log\mu_{t/T}(X_t)\mathrm{d}t + \sqrt{2}\mathrm{d}W_t,\quad X_0\sim \nu,
\end{equation*}
for $t\in[0,T]$, where $T = 1/\epsilon$. For completeness, we present the chosen values for both quantities in Table \ref{tab:dpsmc_hyperarameters}. See Table \ref{tab:target_spread_values} for the second moments of targets used to compute $\epsilon$. For \gls*{dald} with \gls*{ll}, we reuse the same $\epsilon$ but search for $\lambda_0$ as in DPSMC (OU). Because now we have $\lambda_0\neq 0$, we recompute the schedule minimising the action upper bound with general boundary conditions to arrive at
\begin{equation*}
    \lambda_t = \frac{1}{2}\left(1+\sin\left(\frac{\pi t}{2}+\arcsin(2\lambda_0-1)(1-t)\right)\right).
\end{equation*}

\begin{table}[h]
    \centering
    \caption{DPSMC (\gls*{dald}) hyperparameters across the benchmark targets.}
    
    \setlength{\tabcolsep}{2.5pt}
    \begin{tabular}{lcccccc}
        \toprule
         & GMM40 & GMM40 {($d=50$)} & Rings & Funnel & Ionosphere & Sonar \\
         \midrule
         $\xi$ & $2^{3.5}$ & $2^{2.9}$ & $2^{-2.5}$ & $2^{-0.1}$ & $2^{-2.5}$ & $2^{-2.5}$\\
         \midrule
         Tracking Strength $(\epsilon)$ & $0.0017$ & $0.0025$ & $0.3610$ & $0.0652$ & $0.5319$ & $0.5435$\\
         Simulation Time $(T)$ & $584.25$ & $387.66$ & $2.77$ & $15.33$ & $1.88$ & $1.84$\\
         \midrule
         Initial Schedule for \gls*{ll} ($\lambda_0$) & $4/6$ & $4/6$ & $4/6$ & $4/6$ & $1/6$ & $1/6$\\
         \bottomrule
    \end{tabular}
    \label{tab:dpsmc_hyperarameters}
\end{table}

\fparagraph{DPSMC (\gls*{si}).} Similar to \gls*{dald}, we set the base Gaussian variance to be $\sigma^2=\mathbb{E}_\pi[\lVert X\rVert^2]/d$ and the schedule to be $\lambda_t=\sin^2(\frac{\pi t}{2})$. Matching our OU implementation, we set the diffusion coefficient to be $\gamma_t=\sigma^2$. We choose to use an \gls*{em} discretisation for the SDE in \eqref{eq:general_diffusion_path_sde}.

\subsection{Target Distributions}

\begin{table}[]
    \centering
    \caption{The specific values of target information used to inform samplers across several target distributions.}
    \begin{tabular}{lcccc}
        \toprule
        Target & $\mathbb{E}_\pi[\lVert X\rVert^2]$ & $\sigma$ & $R$ & $\tau$ \\
        \midrule
        GMM40 ($d=2$) & $268.98$ & $11.60$ & $18.33$ & $1$\\
        GMM40 ($d=50$) & $6840.25$ & $11.70$ & $13.41$ & $1$\\
        Rings ($d=2$) & $7.52$ & $1.94$ & $4/\sqrt{2}$ & $0.15$ \\
        Funnel ($d=10$) & $43.34$ & $2.08$ & $2.12/\sqrt{10}$ & $0$ \\
        Ionosphere ($d=35$) & $\approx41.25$ & $1.09$ & $2.5/\sqrt{36}$ & 0\\
        Sonar ($d=61$)& $\approx67.25$ & $1.05$ & $2.5/\sqrt{62}$ & 0\\
        \bottomrule
    \end{tabular}
    \label{tab:target_spread_values}
\end{table}

We detail here the targets that were considered in our experiments. As we use the expected squared norm $\mathbb{E}_\pi[\lVert X\rVert^2]$ in our methods, we precisely write down their expressions or approximations if unavailable. A summary of their values are available in Table \ref{tab:target_spread_values}, including the choices for $(R,\tau)$ made by \citep{grenioux2024stochastic} to be used for the other samplers.

\fparagraph{GMM40.} We consider a uniform mixture of 40 Gaussians in $d$ dimensions with identity covariance matrices and means sampled uniformly from a $d$-dimensional hypercube of side length 40. The density is exactly given by
\begin{align*}
    \pi(x) = \frac{1}{40}\sum_{i=1}^{40}{\mathcal{N}(x; \mathbf{m}_i, \mathbf{I})},
\end{align*}
for $\mathbf{m}_i \sim U[-20, 20]^{d}$. We therefore have the expected squared norm equal to
\begin{align*}
    \mathbb{E}_\pi[\lVert X\rVert^2] = d + \frac{1}{40}\sum_{i=1}^{40}{\lVert\mathbf{m}_i\rVert^2}.
\end{align*}
We then appropriately set $R = \frac{1}{\sqrt{d}}\max_{i\in[40]}\{\lVert \mathbf{m}_i - \frac{1}{40}\sum_{j=1}^{40}{\mathbf{m}_j}\rVert\}$ and $\tau = 1$. In our experiments, we fix the seed (0) when sampling the mean vectors. Under these choices, we arrive at values shown in Table \ref{tab:target_spread_values}.

\fparagraph{Rings.} The rings distribution \citep{grenioux2024stochastic} is a radially symmetric distribution featuring four concentric rings of different radii as the regions with high density. First, we define two univariate distributions: a radial one $p_r$ given by
\begin{align*}
    p_r(x) = \frac{1}{4}\sum_{i=1}^{4}{\mathcal{N}(x;i, 0.15^2)},
\end{align*}
and an angular one $p_\theta$ that is a uniform distribution over $[0, 2\pi]$. The rings distribution is precisely the inverse polar reparameterisation of the joint distribution of both radial and angular components
\begin{align*}
    \pi(x) = p_r(\lVert x\rVert)p_{\theta}(\tan^{-1}(x_1/x_2)).
\end{align*}
Given the rotational symmetry, it suffices to look at the radial component to compute the expected squared norm. We have
\begin{align*}
    \mathbb{E}_\pi[\lVert X\rVert^2] = \mathbb{E}_{p_r}[X^2] =\frac{1}{4}\sum_{i=1}^{4}{(i^2+0.15^2)} = 7.5225.
\end{align*}

\fparagraph{Funnel.} The funnel distribution \citep{neal2003slice} is a hierarchical distribution featuring an exponentially narrowing region of high density. It is exactly given by
\begin{align*}
    \pi(x) = \mathcal{N}(x_1; 0, \eta^2)\prod_{i=2}^{10}\mathcal{N}(x_{i};0,\exp(x_1)),
\end{align*}
for some $\eta^2$. While \citet{grenioux2024stochastic} claim to use $\eta = 3$, we found their code to have actually used $\eta^2=3$ when assessing their sampler and its hyperparameters. We accordingly use $\eta^2 = 3$ for our experiments. We now solve for its expected squared norm. We have that
\begin{align*}
    \mathbb{E}_\pi[\lVert X\rVert^2] &= \eta^2 + \int \tfrac{1}{\sqrt{2\pi\eta^2}}\exp\left(-\tfrac{x_1^2}{2\eta^2}\right)\int\tfrac{1}{\sqrt{2\pi\exp(x_1)}^{9}}\exp\left(-\tfrac{\lVert x_{2:10}\rVert^2}{2\exp(x_1)}\right)\lVert x_{2:10}\rVert^2 \mathrm{d} x_{2:10} \mathrm{d} x_1\\
    &= \eta^2 + \int \tfrac{1}{\sqrt{2\pi\eta^2}}\exp\left(-\tfrac{x_1^2}{2\eta^2}\right)\left(9\exp(x_1)\right) \mathrm{d} x_1\\
    &=\eta^2 + 9 \exp(\eta^2/2) \approx 43.34
\end{align*}
where we used the fact that the MGF for a univariate gaussian $p(x)=\mathcal{N}(x;\mu,\sigma^2)$ is $\mathbb{E}_p[\exp(tX)] = \exp(\mu t + t^2\sigma^2/2)$.

\fparagraph{Bayesian logistic regression.} Given a dataset $\mathcal{D} = \{(x_i,y_i)\}_{i=1}^{M}$ with vector-valued features $x_i\in\mathbb{R}^d$ and binary outcomes $y_i\in\{0,1\}$ for $i\in[M]$, a Bayesian logistic regression model is given by a prior on its weight vector $w\in\mathbb{R}^{d}$ and intercept $b\in\mathbb{R}$ which we choose to be
\begin{align*}
    p(w, b) = \mathcal{N}(w;0,\mathbf{I}) \mathcal{N}(b;0,2.5^2),
\end{align*}
and a likelihood for a pair $(x, y)$ in terms of the model parameters
\begin{align*}
    p(y\mid x; w,b) = \mathrm{Bernoulli}\left(y; \sigma(w^\top x + b)\right),
\end{align*}
where $\sigma(x)=\frac{\exp(x)}{1+\exp(x)}$ is the sigmoid function. The goal is to sample from the posterior distribution
\begin{align*}
    \pi(w,b):=p(w, b\mid \mathcal{D}) \propto p(\mathcal{D} \mid w, b) p(w, b) = \left[\prod_{i=1}^{M}{p(y_i \mid x_i;w,b)}\right]p(w,b).
\end{align*}
Instead of computing the expected squared norm of the model parameters with respect to the posterior, we compute our statistics based on the prior directly, i.e. we have
\begin{align*}
    \mathbb{E}_{p}[\lVert W\rVert^2 + B^2] = d \cdot 1 + 2.5^2 = d +6.25.
\end{align*}

\subsection{Normalising Constant Estimation}
We briefly detail our setup for estimating normalising constants. The necessary modifications to our method are highlighted in Algorithm \ref{alg:dpsmc_interacting}. Here, we have included an additional backwards kernel $\mathsf{L}_k^{\text{sde}}$ for defining the extended target space. In the OU case, we choose \gls*{em} discretisations of the forward OU SDE and its time-reversal, i.e.
\begin{align*}
    \mathsf{K}_{k+1}^\text{sde}(X_{k+1} \mid X_{k}, S_k, h) &= \mathcal{N}(X_{k+1}; X_k + hT(X_k + 2\sigma^2S_k), 2hT\sigma^2\mathbf{I})),\\
    \mathsf{L}_k^\text{sde}(X_k \mid X_{k+1}, h) &=  \mathcal{N}(X_k; X_{k+1} - hTX_{k+1},2hT\sigma^2\mathbf{I}),
\end{align*}
which we found to be more stable than their \gls*{ei} counterparts.

In this experiment, we follow roughly the same number of energy evaluations and discretisation as RDSMC \citep{wu2025reverse}. For Rings, we run \gls*{ais} for $32$ steps between $\mathcal{N}(x;0,\tilde{\sigma}^2\min(1,(1-\lambda_{0})/\lambda_{0})\mathbf{I})$ and $\varrho_{0,X_0}$ to initialise our auxiliary variables, then we run our algorithm for a smaller number of discretisation steps, $K=124$, but compensate with $8$ \gls*{mala} steps per iteration. This gives the same number of energy evaluations as specified in our previous benchmarks. For Funnel, we run $32\times 5$ \gls*{ais} steps and $8\times 5$ \gls*{mala} steps per iteration. We choose to do an \gls*{ais} initialisation as it helps provide an accurate estimate for the initial normalising constant. Otherwise, simple IS may provide a high variance estimate that drowns the contribution of the remaining telescoping ratio. Further, unlike RDSMC, we perform an \gls*{ll} initialisation with $N_{\text{LL}}=8$ steps, with each step being accounted for in the weights using symmetric forward and backward ULA kernels.

As noticed by \citet{wu2025reverse}, the marginal density estimates may be unstable at certain times, especially when the posterior coverage is poor, as is typically the case in high dimensions. While RDSMC has a hyperparameter for when resampling is allowed to take place, we choose to not resample for simplicity. We previously mentioned taming terminal instabilities by halting the SMC sampler when the acceptance ratio drops below a threshold and using the true target scores. In this setting where we still need marginal density estimates, we perform SNIS with the Gaussian component of the posterior as the proposal. This performs quite well, given it has a dominant contribution for $\lambda_t$ close to 1.

We perform a hyperparameter search for $\lambda_0$ in the same uniform grid $\{1/6,2/6,\ldots, 5/6\}$. The selected values are found in Table \ref{tab:dpsmc_ou_hyperparameters_normalising_constant}. 

\begin{table}[H]
    \centering
    \caption{DPSMC (OU) hyperparameters for normalising constant estimation.}
    {\begin{tabular}{lcc}
        \toprule
        & Rings & Funnel\\
        \midrule
        Initial Schedule for \gls*{ll} ($\lambda_0$) & $4/6$ & $4/6$ \\
        \midrule
        Final OU Time ($T=-\log(\lambda_0)/2$) & $0.2027$ & $0.2027$ \\
        \bottomrule
    \end{tabular}}
    \label{tab:dpsmc_ou_hyperparameters_normalising_constant}
\end{table}

%% file: appendix/h_additional_experiments.tex
\section{Ablations and Additional Experiments}\label{sec_app:additional_experiments}

\subsection{Choice of Score Identity}\label{app:subsec:score_identities_ablation}
In the main text, we proposed control variate schedules that were independent of position, yielding a score identity we term MCVSI. Our experiments made use of this identity with the matrix schedule. Here, we investigate how using the other score identities affects the performance of our algorithm. 

\subsubsection{Benchmarking Other Score Identities}

We run our benchmarks on the other set of identities: DSI, TSI, MSI \citep{de2024target, he2025training}, the control-variate score identity (CVSI) \citep{kahouli2025control, ko2025latent} (as written in Lemma \ref{lemma:cv_schedule_single_expectation}), (scalar) MCVSI, and (diagonal) DMCVSI. We test the identities on two regimes: one with many auxiliary variables (same setup as our main experiments) and one with fewer. This is to see whether score identities that rely on estimation do worse with fewer auxiliary variables than those that do not, e.g. MSI, and whether MCVSI, which uses the same schedule for each sample, loses some fidelity in contrast to CVSI, which estimates schedules on a per-sample basis. Table \ref{tab:ablation_score_identities} summarises our findings.

Overall, we find that: (i) our matrix schedule was best in handling the anisotropic Funnel target with the diagonal schedule in second, (ii) DSI and TSI alone are quite unstable (without clipping score estimates), (iii) identities relying on estimation work even with few auxiliary vriables, and (iv) MCVSI provides comparable if not better performance than CVSI.

\begin{table*}[h]
    \centering
    \caption{Performance of DPSMC (DALD) using different settings and control variate schedules across ten seeds. Best result under each DPSMC setting is in \textbf{bold}, and second-best is \underline{underlined}.}
    \setlength{\tabcolsep}{2.8pt}
    {\small \begin{tabular}{lcccccc}
        \toprule
         \makecell{Setting /\\ Score Identity} & \makecell{GMM40\\ ($d=2$)\\ $W_2$ ($\downarrow$)} & \makecell{GMM40\\ ($d=50$)\\ $W_2$ ($\downarrow$)} & \makecell{Rings\\ ($d=2$)\\ $W_2$ ($\downarrow$)} & \makecell{Funnel\\ ($d=10$)\\ KS ($\downarrow$)} & \makecell{Ionosphere\\ ($d=35$)\\ Log-LL ($\uparrow$)} & \makecell{Sonar\\ ($d=61$)\\ Log-LL ($\uparrow$)}\\
        \midrule
        {\footnotesize $K=1024, N_Y=128$}\\
        \midrule
        DSI & \tabres{5.58}{1.72} & \tabres{29.30}{1.47} & \tabres[\mathbf]{0.19}{0.02} & \tabres{0.094}{0.004} & \tabres{\text{-}84.81}{1.23} & \tabres{\text{-}108.58}{0.09}\\
        TSI & \tabres[\mathbf]{1.82}{0.26} & \tabres{29.08}{1.07} & \tabres[\mathbf]{0.19}{0.02} & \tabres{0.097}{0.003} & \tabres{\text{-}438.94}{2.75} & \tabres{\text{-}330.93}{0.89}\\
        MSI & \tabres{1.89}{0.26} & \tabres{29.59}{1.00} & \tabres{0.21}{0.02} & \tabres{0.069}{0.003} & \tabres{\text{-}83.96}{0.11} & \tabres{\text{-}109.06}{0.07} \\
        CVSI & \tabres[\underline]{1.83}{0.27} & \tabres{29.62}{0.76} & \tabres{0.21}{0.02} & \tabres{0.073}{0.006} & \tabres{\text{-}83.76}{0.10} & \tabres[\underline]{\text{-}107.92}{0.07}\\
        MCVSI & \tabres{1.89}{0.28} & \tabres[\mathbf]{29.05}{1.51} & 	\tabres{0.21}{0.02} & \tabres{0.070}{0.003} &  \tabres[\mathbf]{\text{-}83.71}{0.10} & \tabres[\mathbf]{\text{-}107.67}{0.07}\\
        DMCVSI & \tabres{1.94}{0.25} & \tabres{29.40}{1.26} & \tabres{0.21}{0.02} & \tabres[\underline]{0.063}{0.004} & \tabres[\underline]{\text{-}83.74}{0.10} & \tabres{\text{-}108.01}{0.05}\\
        MMCVSI & \tabres{1.92}{0.22} & \tabres[\mathbf]{29.05}{1.51} & \tabres{0.20}{0.02} & \tabres[\mathbf]{0.062}{0.003} & \tabres{\text{-}83.92}{0.09} & \tabres{\text{-}108.84}{0.09}\\
        \midrule 
        {\footnotesize $K=8192, N_Y=16$}\\
        \midrule
        DSI & \tabres{2.06}{3.32} & \tabres{30.04}{1.82} & nan & 	\tabres{0.154}{0.022} & \tabres{\text{-}91.22}{1.08} & \tabres{\text{-}110.31}{0.15}\\
        TSI & \tabres[\mathbf]{1.42}{0.34} & nan & nan & \tabres{0.189}{0.004} & \tabres{\text{-}98.21}{0.20} & \tabres{\text{-}119.84}{0.19} \\
        MSI & \tabres[\underline]{1.43}{0.25} & \tabres{30.39}{1.35} & \tabres{0.22}{0.02} & \tabres{0.041}{0.002} & \tabres{\text{-}86.00}{0.09} & \tabres{\text{-}109.47}{0.11}\\
        CVSI & \tabres{1.44}{0.27} & \tabres[\mathbf]{29.47}{0.89} & \tabres{0.20}{0.02} & \tabres{0.045}{0.003} & \tabres{\text{-}86.02}{0.09} & \tabres{\text{-}109.39}{0.11}\\
        MCVSI & \tabres{1.44}{0.23} & \tabres{30.20}{1.31} & \tabres[\mathbf]{0.18}{0.22} & \tabres{0.036}{0.005} & \tabres[\mathbf]{\text{-}85.92}{0.09} & \tabres[\mathbf]{\text{-}109.22}{0.11}\\
        DMCVSI & \tabres{1.47}{0.29} & \tabres{30.20}{1.31} & \tabres{0.20}{0.02} & \tabres[\underline]{0.035}{0.005} & \tabres[\underline]{\text{-}85.93}{0.10} & \tabres[\underline]{\text{-}109.29}{0.07}\\
        MMCVSI &	\tabres{1.46}{0.26} & \tabres[\underline]{29.94}{1.19} & \tabres[\mathbf]{0.18}{0.01} & \tabres[\mathbf]{0.031}{0.002} & \tabres{\text{-}85.96}{0.09} & \tabres{\text{-}109.43}{0.10}\\
        \bottomrule
        \end{tabular}}
        \label{tab:ablation_score_identities}
\end{table*}

\subsubsection{High-Dimensional Scaling of CV Schedules}

We provide three varying levels of our control variate schedules in Appendix \ref{subsec_app:control_variate_schedule}. While choosing the matrix schedule is provably best in the idealised setting, it is unclear in practice when to choose between scalar, diagonal, or matrix schedules because of differences in their dimension, estimation fidelity, numerical instabilities, and complexity. We consider two targets with different score covariance structures: Funnel \citep{neal2003slice}, which is diagonal, and anisotropic Gaussians with non-zero off-diagonal entries. We probe \gls*{dpsmc} (\gls*{si}) for this experiment, under increasing dimension, featuring the same settings and budget as in our benchmarks, regardless of dimension.

\fparagraph{Funnel.} In Table \ref{tab:scaling_funnel}, we report the sliced \gls*{ks} distance, mean ESS, and runtime, averaged over three seeds. We observe two things at very high dimensions. First, the scalar schedule becomes unstable near the end of the path ($\lambda_t\to1$) and diverges. Note that Funnel's score covariance is $\mathcal{I}_\pi=\mathrm{diag}(\tfrac{1}{\eta}+\tfrac{d}{2},e^{\eta/2},\ldots,e^{\eta/2})$. The score variance of the first coordinate grows with $d$, yet the scalar schedule has no per-coordinate degree of freedom to absorb it. Under DPSMC (OU) we found the scalar schedule remains stable, suggesting it is SI's transport, not the control variate alone, that tips it into divergence. Second, the diagonal schedule slightly outperforms the matrix one, presumably because the matrix schedule adds variance from off-diagonal terms that are identically zero.

\begin{table*}[h]
    \centering
    \caption{Dimension scaling of DPSMC (SI) on Funnel using different control variate schedules across three seeds. The KS floor (sliced KS between two independent sets of 4096 exact draws) is ${\approx}0.016$ for all $d$. Best KS under each dimension is in \textbf{bold}, and second-best is \underline{underlined}.}
    \setlength{\tabcolsep}{2.4pt}
    {\small \begin{tabular}{lccccccccc}
        \toprule
         & \multicolumn{3}{c}{Scalar} & \multicolumn{3}{c}{Diagonal} & \multicolumn{3}{c}{Matrix}\\
        \cmidrule(lr){2-4} \cmidrule(lr){5-7} \cmidrule(lr){8-10}
        $d$ & KS ($\downarrow$) & ESS ($\uparrow$) & \makecell{Runtime\\{\small (s) ($\downarrow$)}} & KS ($\downarrow$) & ESS ($\uparrow$) & \makecell{Runtime\\{\small (s) ($\downarrow$)}} & KS ($\downarrow$) & ESS ($\uparrow$) & \makecell{Runtime\\{\small (s) ($\downarrow$)}}\\
        \midrule
        $10$ & \tabres[\mathbf]{0.022}{0.001} & \tabres{91.6}{0.0} & \tabres{57}{1} & \tabres[\underline]{0.023}{0.001} & \tabres{91.6}{0.0} & \tabres{52}{1} & \tabres[\underline]{0.023}{0.001} & \tabres{91.6}{0.0} & \tabres{56}{0}\\
        $20$ & \tabres[\mathbf]{0.036}{0.001} & \tabres{86.9}{0.0} & \tabres{57}{1} & \tabres[\underline]{0.037}{0.002} & \tabres{86.8}{0.0} & \tabres{55}{2} & \tabres[\underline]{0.037}{0.002} & \tabres{86.8}{0.0} & \tabres{52}{2}\\
        $40$ & \tabres{0.074}{0.002} & \tabres{83.0}{0.0} & \tabres{67}{2} & \tabres[\underline]{0.070}{0.002} & \tabres{82.6}{0.3} & \tabres{70}{1} & \tabres[\mathbf]{0.069}{0.002} & \tabres{82.8}{0.3} & \tabres{66}{4}\\
        $80$ & \tabres{0.110}{0.002} & \tabres{79.7}{0.0} & \tabres{86}{2} & \tabres[\underline]{0.090}{0.002} & \tabres{79.7}{0.0} & \tabres{84}{3} & \tabres[\mathbf]{0.089}{0.003} & \tabres{79.8}{0.0} & \tabres{86}{1}\\
        $160$ & \tabres{0.125}{0.002} & \tabres{77.7}{0.0} & \tabres{126}{1} & \tabres[\underline]{0.099}{0.002} & \tabres{77.3}{0.1} & \tabres{126}{2} & \tabres[\mathbf]{0.098}{0.003} & \tabres{77.3}{0.1} & \tabres{145}{1}\\
        $320$ & nan & -- & -- & \tabres[\mathbf]{0.107}{0.002} & \tabres{75.9}{0.1} & \tabres{205}{3} & \tabres[\underline]{0.108}{0.002} & \tabres{75.8}{0.0} & \tabres{216}{1}\\
        $640$ & nan & -- & -- & \tabres[\mathbf]{0.109}{0.001} & \tabres{75.3}{0.0} & \tabres{324}{7} & \tabres[\underline]{0.115}{0.002} & \tabres{75.3}{0.0} & \tabres{361}{2}\\
        \bottomrule
        \end{tabular}}
        \label{tab:scaling_funnel}
\end{table*}

\fparagraph{Anisotropic Gaussian.} We consider a zero-mean Gaussian $\mathcal{N}(0, \Sigma)$ with covariance $\Sigma=Q\Lambda Q^\top$, for a rotation matrix $Q$ fixed per dimension, and $\Lambda=\mathrm{diag}(\lambda_1,\ldots, \lambda_d)$, for $\lambda_i = \kappa^{\frac{i-1}{d-1}-\frac{1}{2}}$ equally log-spaced, fixing the condition number $\kappa=10$. In Table \ref{tab:scaling_gaussian}, we report relative covariance error $\text{RCE}(\Sigma,\hat{\Sigma})=\lVert\Sigma-\hat{\Sigma}\rVert_F/\lVert\Sigma\rVert_F$, average relative score error along the path, and mean ESS, averaged over three seeds. We omit the runtimes which scale similarly as in Table \ref{tab:scaling_funnel}. We find all schedules reach the RCE floor, with the matrix schedule having the smallest relative score error.

\begin{table*}[h]
    \centering
    \caption{Dimension scaling of DPSMC (SI) on an anisotropic Gaussian with condition number $10$ using different control variate schedules across three seeds. We report relative covariance error (RCE) and average relative score error along the path (RSE). The RCE floor is the RCE of exact samples. Best RSE under each dimension is in \textbf{bold}, and second-best is \underline{underlined}.}
    \setlength{\tabcolsep}{3.0pt}
    {\small\begin{tabular}{lccccccc}
        \toprule
         & & \multicolumn{2}{c}{Scalar} & \multicolumn{2}{c}{Diagonal} & \multicolumn{2}{c}{Matrix}\\
        \cmidrule(lr){3-4} \cmidrule(lr){5-6} \cmidrule(lr){7-8}
        $d$ & \makecell{RCE\\{\small floor}} & RCE ($\downarrow$) & RSE ($\downarrow$) & RCE ($\downarrow$) & RSE ($\downarrow$) & RCE ($\downarrow$) & RSE ($\downarrow$)\\
        \midrule
        $10$ & 0.043 & \tabres{0.042}{0.004} & \tabres{2.75\mathrm{e}\text{-}2}{1.0\mathrm{e}\text{-}4} & \tabres{0.042}{0.004} & \tabres[\underline]{2.53\mathrm{e}\text{-}2}{2.1\mathrm{e}\text{-}5} & \tabres{0.041}{0.004} & \tabres[\mathbf]{2.42\mathrm{e}\text{-}4}{2.2\mathrm{e}\text{-}5}\\
        $20$ & 0.061 & \tabres{0.055}{0.003} & \tabres{2.58\mathrm{e}\text{-}2}{2.0\mathrm{e}\text{-}4} & \tabres{0.055}{0.003} & \tabres[\underline]{2.49\mathrm{e}\text{-}2}{1.9\mathrm{e}\text{-}4} & \tabres{0.055}{0.003} & \tabres[\mathbf]{3.42\mathrm{e}\text{-}4}{3.7\mathrm{e}\text{-}6}\\
        $40$ & 0.087 & \tabres{0.086}{0.002} & \tabres{2.54\mathrm{e}\text{-}2}{1.2\mathrm{e}\text{-}4} & \tabres{0.086}{0.003} & \tabres[\underline]{2.46\mathrm{e}\text{-}2}{1.2\mathrm{e}\text{-}4} & \tabres{0.086}{0.003} & \tabres[\mathbf]{4.94\mathrm{e}\text{-}4}{1.6\mathrm{e}\text{-}5}\\
        $80$ & 0.118 & \tabres{0.117}{0.002} & \tabres{2.52\mathrm{e}\text{-}2}{3.6\mathrm{e}\text{-}5} & \tabres{0.117}{0.002} & \tabres[\underline]{2.47\mathrm{e}\text{-}2}{8.9\mathrm{e}\text{-}5} & \tabres{0.117}{0.002} & \tabres[\mathbf]{7.06\mathrm{e}\text{-}4}{5.0\mathrm{e}\text{-}6}\\
        $160$ & 0.167 & \tabres{0.166}{0.002} & \tabres{2.50\mathrm{e}\text{-}2}{5.6\mathrm{e}\text{-}5} & \tabres{0.166}{0.003} & \tabres[\underline]{2.49\mathrm{e}\text{-}2}{7.8\mathrm{e}\text{-}5} & \tabres{0.166}{0.003} & \tabres[\mathbf]{1.00\mathrm{e}\text{-}3}{6.5\mathrm{e}\text{-}6}\\
        $320$ & 0.236 & \tabres{0.236}{0.001} & \tabres[\underline]{2.49\mathrm{e}\text{-}2}{2.6\mathrm{e}\text{-}5} & \tabres{0.236}{0.001} & \tabres[\underline]{2.49\mathrm{e}\text{-}2}{3.6\mathrm{e}\text{-}5} & \tabres{0.237}{0.000} & \tabres[\mathbf]{1.41\mathrm{e}\text{-}3}{3.8\mathrm{e}\text{-}6}\\
        $640$ & 0.333 & \tabres{0.331}{0.001} & \tabres[\underline]{2.49\mathrm{e}\text{-}2}{4.8\mathrm{e}\text{-}6} & \tabres{0.331}{0.001} & \tabres[\underline]{2.49\mathrm{e}\text{-}2}{3.0\mathrm{e}\text{-}5} & \tabres{0.334}{0.001} & \tabres[\mathbf]{2.00\mathrm{e}\text{-}3}{6.0\mathrm{e}\text{-}6}\\
        \bottomrule
        \end{tabular}}
        \label{tab:scaling_gaussian}
\end{table*}

Informed by our findings, a rule of thumb is to choose between the three depending on the perceived structure of the target score covariance. As we have seen, the scalar schedule can diverge when there is inherent anisotropy not being captured. On the more practical side, the matrix schedule requires keeping in memory the entire covariance matrix, which becomes expensive as dimension is increased, not to mention the matrix inversion involved. The diagonal schedule hence strikes a balance between fidelity and ease in computation, making it a solid first choice.

\subsection{Wall-clock Times}\label{app:subsec:wall_clock_times_ablation}

We measure wall-clock times for running SMC, SLIPS, and DPSMC (DALD) across all the benchmarks on a single Nvidia RTX 6000 Ada GPU. Each result is in seconds and reported as an average of five seeds, with each seed producing 4096 samples. Note that all implementations are in PyTorch. As in Table \ref{tab:ablation_wall_clock_times}, we find empirically that DPSMC is significantly faster than other samplers on the toy targets and comparable in logistic regression. In principle, one can choose different hyperparameters for the competing samplers to match the speed of DPSMC. Doing so for SMC and SLIPS typically results in worse performance compared to those reported and DPSMC.

\begin{table*}[h]
    \centering
    \caption{Wall-clock times of samplers in seconds averaged across five seeds. Best result is in \textbf{bold}.}
    \setlength{\tabcolsep}{2.8pt}
    {\small \begin{tabular}{lccccccc}
        \toprule
         Algorithm & \makecell{Batched \\ Evals} & \makecell{GMM40\\{ ($d=2$)}} & \makecell{GMM40\\{ ($d=50$)}} & \makecell{Rings\\ {($d=2$)}} & \makecell{Funnel\\{ ($d=10$)}} & \makecell{Ionosphere\\{($d=35$)}} & \makecell{Sonar\\{($d=61$)}}\\
        \midrule
        SMC & 131K & \tabres{321.68}{3.07} & \tabres{364.44}{21.00} & \tabres{603.88}{16.27} & \tabres{363.07}{5.96} & \tabres{357.12}{27.16} & \tabres{400.10}{13.68}\\
        SLIPS &  32K & \tabres{77.25}{2.43} & \tabres{100.68}{8.20} & \tabres{170.59}{11.72} & \tabres{150.56}{10.65} & \tabres{89.12}{1.43} & \tabres[\mathbf]{94.70}{2.00} \\
        DPSMC & \textbf{1K} & \tabres[\mathbf]{9.66}{1.99} & \tabres[\mathbf]{53.77}{1.27} & \tabres[\mathbf]{9.51}{0.29} & \tabres[\mathbf]{10.60}{0.48} & \tabres[\mathbf]{83.14}{1.24} & \tabres{99.77}{0.71} \\
        \bottomrule
        \end{tabular}}
        \label{tab:ablation_wall_clock_times}
\end{table*}

\subsection{Comparison with Parallel Tempering}

Parallel tempering (PT) \citep{swendsen1986replica, geyer1995annealing} has largely remained as the gold standard amongst sampling methods. It defines a sequence of inverse temperatures $0 < \beta_1 <\ldots < \beta_{N_B} = 1$ and subsequently targets the joint distribution $p(x_1,\ldots, x_{N_B})\propto\prod_{i=1}^{N_B}\pi(x_i)^{\beta_i}$. To do this, it runs several Markov chains in parallel, one at each temperature level, often via Hamiltonian Monte Carlo (HMC) \citep{neal2011mcmc} samplers. After every step, chains at different temperature levels, e.g. adjacent ones, are swapped according to some probability that leaves the joint distribution invariant. This scheme allows the chain at the normal temperature to traverse between modes through swaps with chains at higher temperatures where the energy barriers present are significantly lower.

Now, we describe our choices for its settings. We fix the number of chains to be $N_{B}=16$ and set the inverse temperatures to be $\beta_i = 10^{-(i-1)/5}$. While we would normally run a single instance of PT and take the positions of the coldest chain at every other step, we chose to run 16 parallel instances to make runtimes tractable. Each chain runs an HMC sampler. Simulating Hamiltonian dynamics requires specifying the mass matrix $M$, the step size $h$, and the number of leapfrog steps $L$. We set the mass matrix to be the identity $M = \mathbf{I}$ and choose a step size $h$ that achieves roughly 0.65 acceptance ratio \citep{neal2011mcmc}. First, we initialise 64 chains per temperature level from $\mathcal{N}(0, (R^2 d + \tau^2)\mathbf{I})$ and do 200 warm-up steps to geometrically adjust the step size (per temperature) to achieve the target acceptance ratio. Then, we only keep 16 samples (number of instances) at each temperature level, and, for every subsequent HMC step, we use the same step size obtained after the warm-up. We choose the number of leapfrog steps $L$ to be equal to 128, and we take the coldest chain's position after every 64 HMC steps and treat them as our samples. To be precise, the total energy evaluations (which should be less than $N_{\text{evals}}\times 4096$) is the product of the following quantities: number of instances (16), number of chains/temperatures (16), number of HMC steps (64), number of leapfrog steps (128), number of samples per instance (256).

We include our findings in Table \ref{tab:main_results} for completeness but remark that PT takes significantly longer to complete, despite having parallel instances in the same temperatures, each using a fraction of the total energy budget.

\subsection{Initialisation Dependency}\label{app:subsec:initialisation_dependency_ablation}

We discuss our setup for measuring the initialisation dependency of our algorithm, as documented in Figure \ref{fig:ablation_experiments} (left) in the main text. Recall that $\varrho_{t,x}$ is the density of $X_1$ given $X_t=x$. It is then reasonable to expect that poor coverage of posterior modes by initial auxiliary variables can lead to final $X_1$ samples missing the same set of modes. We aim to understand how our sampler behaves precisely under this poor initialisation regime. To quantify mode coverage, we first describe a means for estimating the mode weights of a mixture distribution.

\fparagraph{Mode weight estimation.} Let $\pi(x) = \sum_{i=1}^{N_m}{w_i\pi_i(x)}$ be some arbitrary mixture distribution, with mixture weights $w_i$ and component distributions $\pi_i$. Given the final samples $\{X_K^j\}_{j=1}^{N}$ of our algorithm, we say the predicted $i$th mode weight is given by
\begin{equation}
    \hat{w}_i = \frac{1}{N}\sum_{j=1}^{N}{\frac{w_i\pi_i(X_K^j)}{\sum_{k=1}^{N_{m}}{w_k \pi_k(X_K^j)}}}.\label{eq:soft_mode_weight_estimate}
\end{equation}
Later, we either measure individual weights or the mode TVD, i.e. $\mathrm{MTVD}(w, \hat{w}) = \frac{1}{2}\lVert w - \hat{w}\rVert_1$.

\fparagraph{Poor initialisation recovery.} Inspired by \citet{grenioux2024stochastic}, we consider the five-dimensional uneven bi-modal Gaussian mixture
\begin{equation}
    \pi(x) = \tfrac{9}{10}\mathcal{N}(x; -\tfrac{2a}{3}\mathbf{1}_5, \tfrac{1}{10}\mathbf{I}_5) + \tfrac{1}{10}\mathcal{N}(x;\tfrac{4a}{3}\mathbf{1}_5, \tfrac{1}{10}\mathbf{I}_5).\label{eq:uneven_bimodal_mixture_target}
\end{equation}
We run \gls*{dpsmc} (\gls*{dald}) and \gls*{dpsmc} (\gls*{si}), which start at $\lambda_0=0$, with their auxiliary variables initialised at the light mode, i.e. $q_0(x) = \mathcal{N}(x;\frac{4a}{3}\mathbf{1}_5, \frac{1}{10}\mathbf{I}_5)$, and measure how the final samples capture the heavy mode weight $w_1=0.9$. Here, $a$ (mode separation) is increased so the ratio $r=R^2d/\tau^2$, with $R=1.8a$ and $\tau^2=0.1$, exceeds the known threshold for the window's existence, i.e. $r<6.48$ \citep[Proposition 19]{grenioux2024stochastic}. To provide a concrete baseline, we also run a single \gls*{mala} chain, initialised at the light mode, with a comparable energy budget. To make the \gls*{mala} chain comparison feasible, however, we need to downsize our sampler to settings $N_X=256, N_Y=16, K=1024$, i.e. \gls*{mala} runs for $N_X\times N_Y\times K = 4,194,304$ steps. For simplicity, we run \gls*{dald} fixing $\xi=1.0$. While this may be suboptimal, the goal of this experiment is to see how \gls*{dald} and \gls*{si} differ in their recovery trends. We can, in fact, significantly improve mode coverage with sufficient tuning. In Table \ref{tab:uneven_mixture_recovery}, we report the mode weight estimates for \gls*{dald} and \gls*{si} under their original budget, in line with Figure \ref{fig:ablation_experiments} (left), which features similar results to the downsized budget. Despite an adversarial initialisation, we find \gls*{dald} is able to recover and reproduce the heavy mode weight up to some large separation $a$, where particles eventually become stuck in their initial mode. \gls*{si} recovers to some extent but largely maintains its error. \gls*{mala} also recovers the heavy mode weight for moderately separated mixtures. At $r=64.0$, however, it mixes poorly despite the metric suggesting otherwise. Given it has significantly more steps, it escapes the light mode but remains stuck at the heavy mode for the rest of the simulation. \gls*{dald} also has particles that reach the heavy mode, but it behaves fundamentally different from a single Markov chain where this greatly influences the entire trajectory or ensemble of samples.

\begin{table*}[h]
    \centering
    \caption{Estimated heavy mode weight $\hat{w}_1$ (true value $0.9$) on the bi-modal mixture in \eqref{eq:uneven_bimodal_mixture_target}, with an adversarial initialisation, across ten seeds. The ratio $r=R^2d/\tau^2$ grows with the mode separation $a$. \textsuperscript{*} Here, \gls*{mala} mixes poorly and never leaves the heavy mode, resulting in an overestimation.}
    \setlength{\tabcolsep}{2.0pt}
    {\small \begin{tabular}{lccccc}
        \toprule
         \makecell{Method / $r$} & $0.5$ & $1.0$ & $2.0$ & $4.0$ & $8.0$\\
        \midrule
        DPSMC (DALD) & \tabres{0.900}{0.001} & \tabres{0.900}{0.002} & \tabres{0.897}{0.004} & \tabres{0.898}{0.004} & \tabres{0.899}{0.007}\\
        DPSMC (SI) & \tabres{0.881}{0.002} & \tabres{0.858}{0.004} & \tabres{0.806}{0.007} & \tabres{0.678}{0.011} & \tabres{0.416}{0.017}\\
        MALA & \tabres{0.900}{0.000} & \tabres{0.900}{0.000} & \tabres{0.900}{0.000} & \tabres{0.900}{0.000} & \tabres{0.900}{0.000} \\
        \midrule
         \makecell{Method / $r$} & $16.0$ & $32.0$ & $64.0$ & $128.0$ & $256.0$\\
        \midrule
        DPSMC (DALD) & \tabres{0.900}{0.001} & \tabres{0.879}{0.008} & \tabres{0.024}{0.004} & \tabres{0.000}{0.000} & \tabres{0.000}{0.000} \\
        DPSMC (SI) & \tabres{0.100}{0.008} & \tabres{0.001}{0.000} & \tabres{0.000}{0.000} & \tabres{0.000}{0.000} & \tabres{0.000}{0.000} \\
        MALA & \tabres{0.901}{0.001} & \tabres{0.897}{0.003} & \tabres{0.921}{0.069}\textsuperscript{*} & \tabres{0.000}{0.000} & \tabres{0.000}{0.000} \\
        \bottomrule
        \end{tabular}}
        \label{tab:uneven_mixture_recovery}
\end{table*}

\subsection{Posterior and Score Error}\label{app:subsec:vanishing_error_ablation}

We discuss our setup for measuring the empirical convergence rates of the posterior and score error, as shown in Figure \ref{fig:ablation_experiments} (middle and right) in the main text. Here, our goal is to understand how the score MSE and $W_2^2$ distance between the SMC empirical distribution and posterior are controlled by the particle count. To measure this empirically, we run \gls*{dpsmc} (\gls*{dald}) under the standard settings on the same bi-modal mixture target in \eqref{eq:uneven_bimodal_mixture_target}, under an increasing particle count and separation, and average over the samples and auxiliary variables at $\lambda_k\approx 0.5$, to form Monte Carlo estimates for the quantities: $\mathbb{E}[W_2^2(\sum_{i}{w_k^i\delta_{\bar{Y}_k^i}},\varrho_{t_k,X_k})]$ and $\mathbb{E}[\lVert\nabla\log\mu_{t_k}(X_k)-S_k^N(X_k)\rVert^2]$. Table \ref{tab:score_error_vs_N} shows our findings, including the slope from an OLS fit of the log estimate on $\log N$. As mentioned in the main text, results for the modest separation regime closely recapitulate existing rates. Notably, for the high separation regime, $r=64.0$, we find the score MSE has a rather flat trend whereas the $W_2^2$ from the posterior maintains a sizeable slope. In Proposition \ref{prop:score_estimation_error}, we bound the score MSE above by the $W_2^2$ from the posterior multiplied by some constant. We believe the difference in rates can be explained by the constant depending on the target, i.e. through the Lipschitz constant and optimal control variate coefficients used in the simulation, which is changing as the separation is increased. The score MSE standard deviation is also fairly substantial here. We remark that our assumption in Proposition \ref{prop:score_estimation_error} is uniform in time whereas we conduct this experiment for a single time point.

\begin{table*}[h]
    \centering
    \caption{Monte Carlo estimates of the two quantities in Proposition \ref{prop:score_estimation_error} on the uneven bi-modal mixture in \eqref{eq:uneven_bimodal_mixture_target}, averaged over three runs of \gls*{dald}. The ratio $r=R^2d/\tau^2$ grows with the mode separation $a$. The slope comes from an OLS fit of the log estimate on $\log N$.}
    \setlength{\tabcolsep}{2.0pt}
    {\small
    \begin{tabular}{lccc}
        \toprule
        & \multicolumn{3}{c}{$W_2^2$ from Posterior}\\
        \cmidrule(lr){2-4}
        $N$ / $r$ & $1.0$ & $16.0$ & $64.0$\\
        \midrule
        $16$   & \tabres{\sci{2.38}{-1}}{\sci{1.9}{-5}} & \tabres{\sci{2.83}{-1}}{\sci{3.0}{-4}} & \tabres{\sci{3.76}{-1}}{\sci{2.8}{-3}}\\
        $64$   & \tabres{\sci{1.40}{-1}}{\sci{3.4}{-4}} & \tabres{\sci{1.65}{-1}}{\sci{4.1}{-4}} & \tabres{\sci{2.45}{-1}}{\sci{9.5}{-4}}\\
        $256$  & \tabres{\sci{8.12}{-2}}{\sci{2.0}{-4}} & \tabres{\sci{9.52}{-2}}{\sci{1.0}{-4}} & \tabres{\sci{1.67}{-1}}{\sci{8.0}{-3}}\\
        $1024$ & \tabres{\sci{4.70}{-2}}{\sci{2.7}{-5}} & \tabres{\sci{5.50}{-2}}{\sci{1.7}{-4}} & \tabres{\sci{1.08}{-1}}{\sci{1.9}{-3}}\\
        \midrule
        Slope & $-0.390$ & $-0.394$ & $-0.298$\\
        \bottomrule
    \end{tabular}
    \hfill
    \begin{tabular}{lccc}
        \toprule
         & \multicolumn{3}{c}{Score MSE}\\
        \cmidrule(lr){2-4}
        $N$ / $r$ & $1.0$ & $16.0$ & $64.0$\\
        \midrule
        $16$   & \tabres{\sci{1.71}{-3}}{\sci{5.4}{-5}} & \tabres{\sci{4.89}{-1}}{\sci{4.5}{-2}} & \tabres{\sci{1.31}{+0}}{\sci{1.1}{-1}}\\
        $64$   & \tabres{\sci{5.21}{-4}}{\sci{9.2}{-6}} & \tabres{\sci{2.26}{-1}}{\sci{4.9}{-3}} & \tabres{\sci{1.40}{+0}}{\sci{2.9}{-2}}\\
        $256$  & \tabres{\sci{1.49}{-4}}{\sci{2.0}{-6}} & \tabres{\sci{1.01}{-1}}{\sci{8.1}{-3}} & \tabres{\sci{1.42}{+0}}{\sci{1.9}{-1}}\\
        $1024$ & \tabres{\sci{3.96}{-5}}{\sci{2.2}{-6}} & \tabres{\sci{4.35}{-2}}{\sci{4.0}{-3}} & \tabres{\sci{1.08}{+0}}{\sci{6.7}{-2}}\\
        \midrule
        Slope & $-0.906$ & $-0.582$ & $-0.041$\\
        \bottomrule
    \end{tabular}}
    \label{tab:score_error_vs_N}
\end{table*}

\subsection{Hyperparameter Sensitivity}

Our algorithms require various hyperparameters whose choices may be non-trivial but greatly influence performance. Here, we conduct a simple experiment to better understand how such choices can be informed. We consider running \gls*{dpsmc} (\gls*{ou}) through different settings for the \gls*{ll} initialisation and the underlying SMC sampler on GMM40 and Funnel targets. For \gls*{dald}, we point to our discussion in Appendix \ref{subsec_app:benchmark_methods} which selects a value for $\epsilon$ that heuristically minimises the overall algorithm bound. We measure performance under GMM40 in MTVD as described in Appendix \ref{app:subsec:initialisation_dependency_ablation}.

\fparagraph{\gls*{ll} initialisation.} Here, we test sensitivity to the \gls*{ll} step count $N_{\text{LL}}$ and the initial time $\lambda_0$ which are responsible for \gls*{ll} initialisation. We fix the energy budget and hence have $N_{\text{LL}}$ steps for LL and $K-N_{\text{LL}}$ steps for the remainder of the path. In Table \ref{tab:sensitivity_ll}, we find target-dependent sweet-spots for initialisation, with results most sensitive to $\lambda_0$. Both parameters trade off between better-initialised $X$ samples (small $\lambda_0$ and large $N_{\text{LL}}$) and finer discretisation of the remaining path (large $\lambda_0$ and small $N_{\text{LL}}$). The auxiliary variable ESS is highest when the initial marginal is close to the base distribution and no \gls*{ll} is performed. In general, there are values $l_\mu$ and $l_\rho$ such that $\mu_t$ is log-concave when $\lambda_t<l_\mu$ and $\varrho_{t,x}$ is log-concave when $\lambda_t>l_\rho$. When the duality of log-concavity \citep{grenioux2024stochastic} holds, i.e. $l_\rho < l_\mu$, we want to initialise $\lambda_0\in(l_\rho, l_\mu)$. However, the window's existence and location are target dependent and generally unavailable unless we know, e.g., the Lipschitz constant of the target score (for $l_\rho$) or $R$ and $\tau$ \citep{grenioux2024stochastic} (for $l_\rho$ and $l_\mu$ bounds). When doing LL, we expect $N_{\text{LL}}$ to depend on how fast the chains initialised at $\nu$ can equilibrate to $\mu_{\lambda_0}$.

\begin{table*}[h]
    \centering
    \caption{Sensitivity of DPSMC (OU) to \gls*{ll} initialisation parameters $\lambda_0$ and $N_{\text{LL}}$ across three seeds. Omitted ESS error bars are at most $\pm 0.1$. Best result for each metric under each target is in \textbf{bold}.}
    \setlength{\tabcolsep}{1.0pt}
    {\small \begin{tabular}{l*{5}{c}*{5}{c}}
        \toprule
        \multicolumn{11}{c}{GMM40 ($d=2$)}\\
        \midrule
         & \multicolumn{5}{c}{MTVD ($\downarrow$)} & \multicolumn{5}{c}{ESS ($\uparrow$)}\\
        \cmidrule(lr){2-6}\cmidrule(lr){7-11}
        \makecell{$\lambda_0$ / $N_{\text{LL}}$} & $0$ & $16$ & $32$ & $64$ & $128$ & $0$ & $16$ & $32$ & $64$ & $128$\\
        \midrule
        $0.167$ & \tabres{0.040}{0.005} & \tabres{0.034}{0.004} & \tabres{0.034}{0.002} & \tabres{0.036}{0.002} & \tabres{0.035}{0.004} & $\mathbf{112.1}$ & $105.4$ & $104.8$ & $103.8$ & $102.3$\\
        $0.333$ & \tabres{0.038}{0.005} & \tabres{0.030}{0.005} & \tabres{0.034}{0.002} & \tabres{0.039}{0.001} & \tabres{0.038}{0.004} & $109.5$ & $102.8$ & $102.2$ & $101.7$ & $101.3$\\
        $0.5$   & \tabres{0.036}{0.001} & \tabres[\mathbf]{0.029}{0.003} & \tabres{0.033}{0.003} & \tabres{0.034}{0.002} & \tabres{0.036}{0.002} & $107.7$ & $102.2$ & $101.6$ & $101.3$ & $100.5$\\
        $0.667$ & \tabres{0.057}{0.003} & \tabres{0.059}{0.008} & \tabres{0.062}{0.006} & \tabres{0.068}{0.010} & \tabres{0.073}{0.010} & $106.0$ & $101.4$ & $100.8$ & $100.2$ & $99.1$\\
        $0.833$ & \tabres{0.123}{0.005} & \tabres{0.144}{0.007} & \tabres{0.152}{0.009} & \tabres{0.151}{0.002} & \tabres{0.159}{0.004} & $104.8$ & $99.8$ & $99.2$ & $98.4$ & $97.3$\\
        \midrule
        \multicolumn{11}{c}{Funnel ($d=10$)}\\
        \midrule
         & \multicolumn{5}{c}{KS ($\downarrow$)} & \multicolumn{5}{c}{ESS ($\uparrow$)}\\
        \cmidrule(lr){2-6}\cmidrule(lr){7-11}
        \makecell{$\lambda_0$ / $N_{\text{LL}}$} & $0$ & $16$ & $32$ & $64$ & $128$ & $0$ & $16$ & $32$ & $64$ & $128$\\
        \midrule
        $0.167$ & \tabres{0.024}{0.001} & \tabres{0.036}{0.003} & \tabres{0.033}{0.003} & \tabres{0.030}{0.000} & \tabres{0.025}{0.000} & $\mathbf{94.2}$ & $92.7$ & $92.1$ & $91.0$ & $89.0$\\
        $0.333$ & \tabres{0.033}{0.001} & \tabres{0.041}{0.002} & \tabres{0.032}{0.002} & \tabres{0.025}{0.001} & \tabres[\mathbf]{0.021}{0.001} & $93.1$ & $92.1$ & $91.4$ & $90.1$ & $87.8$\\
        $0.5$   & \tabres{0.059}{0.002} & \tabres{0.047}{0.002} & \tabres{0.029}{0.001} & \tabres{0.028}{0.001} & \tabres{0.034}{0.001} & $92.5$ & $91.5$ & $90.9$ & $89.5$ & $86.8$\\
        $0.667$ & \tabres{0.090}{0.001} & \tabres{0.043}{0.000} & \tabres{0.051}{0.001} & \tabres{0.065}{0.002} & \tabres{0.072}{0.002} & $92.2$ & $91.2$ & $90.5$ & $89.0$ & $86.0$\\
        $0.833$ & \tabres{0.135}{0.002} & \tabres{0.082}{0.001} & \tabres{0.106}{0.001} & \tabres{0.130}{0.003} & \tabres{0.141}{0.004} & $91.7$ & $90.6$ & $89.9$ & $88.3$ & $84.9$\\
        \bottomrule
    \end{tabular}}
    \label{tab:sensitivity_ll}
\end{table*}

\fparagraph{SMC sampler.} Here, we test sensitivity to the particle count $N_Y$, \gls*{mala} step count $N_{\text{MALA}}$, and ESS threshold $\kappa$ which encompass the SMC sampler parameters. Once again, we fix the energy budget and hence the product $N_Y \times N_{\text{MALA}}$. In Table \ref{tab:sensitivity_particles_mala}, we find that GMM40, and likely other multimodal targets, benefit monotonically from more particles and better mode coverage. Funnel instead suggests that targets with complex geometries need a minimum particle count for the auxiliary SMC approximation to be reliable and, once met, further \gls*{mala} equilibration is more valuable than additional particles. We remark that these choices also depend on how finely discretised the diffusion process is. In the event where the discretised process takes large steps, more \gls*{mala} steps are required for particles to adapt to the likewise rapidly changing posterior. In Table \ref{tab:sensitivity_kappa}, we find reasonable choices for the ESS threshold $\kappa$ have no major effect on sample quality for these targets.

\begin{table*}[h]
    \centering
    \caption{Sensitivity of DPSMC (OU) to trading off particle count $N_Y$ against MALA step count $N_{\text{MALA}}$ under a fixed energy budget, across ten seeds. NESS is the normalised ESS, i.e. ESS$/N_Y$. Best result for each metric under each target is in \textbf{bold}, and second-best is \underline{underlined}.}
    \setlength{\tabcolsep}{2.8pt}
    {\small \begin{tabular}{clcccccc}
        \toprule
         \multicolumn{2}{c}{\makecell{Target / $(N_Y, N_{\text{MALA}})$}} & $(4, 32)$ & $(8, 16)$ & $(16, 8)$ & $(32, 4)$ & $(64, 2)$ & $(128, 1)$\\
        \midrule
        \multirow{2}{*}{\makecell{GMM40\\ ($d=2$)}} & MTVD ($\downarrow$) & \tabres{0.069}{0.005} & \tabres{0.049}{0.004} & \tabres{0.042}{0.004} & \tabres{0.040}{0.004} & \tabres[\mathbf]{0.037}{0.005} & \tabres[\mathbf]{0.037}{0.004}\\
                      & NESS ($\uparrow$) & \tabres[\mathbf]{0.80}{0.00} & \tabres{0.77}{0.00} & \tabres{0.77}{0.00} & \tabres{0.77}{0.00} & \tabres{0.78}{0.00} & \tabres[\underline]{0.79}{0.00}\\
        \midrule
        \multirow{2}{*}{\makecell{Funnel\\ ($d=10$)}} & KS ($\downarrow$) & \tabres{0.040}{0.004} & \tabres[\mathbf]{0.020}{0.001} & \tabres[\underline]{0.029}{0.004} & \tabres{0.034}{0.003} & \tabres{0.032}{0.002} & \tabres{0.030}{0.002}\\
                        & NESS ($\uparrow$) & \tabres[\mathbf]{0.76}{0.00} & \tabres[\underline]{0.74}{0.00} & \tabres{0.72}{0.00} & \tabres{0.71}{0.00} & \tabres{0.70}{0.00} & \tabres{0.70}{0.00}\\
        \bottomrule
        \end{tabular}}
        \label{tab:sensitivity_particles_mala}
\end{table*}

\begin{table*}[h]
    \centering
    \caption{Sensitivity of DPSMC (OU) to the ESS threshold $\kappa$, across ten seeds. Resampling is triggered only when the ESS falls below $\kappa N_Y$.}
    \setlength{\tabcolsep}{2.8pt}
    {\small \begin{tabular}{llccc}
        \toprule
         \makecell{Target / $\kappa$} & & $0.3$ & $0.5$ & $0.7$\\
        \midrule
        \multirow{2}{*}{\makecell{GMM40\\ ($d=2$)}} & MTVD ($\downarrow$) & \tabres{0.035}{0.006} & \tabres{0.033}{0.001} & \tabres{0.036}{0.004}\\
                      & ESS ($\uparrow$) & \tabres{84.1}{0.2} & \tabres{101.3}{0.1} & \tabres{113.0}{0.0}\\
        \midrule
        \multirow{2}{*}{\makecell{Funnel\\ ($d=10$)}} & KS ($\downarrow$) & \tabres{0.029}{0.002} & \tabres{0.030}{0.002} & \tabres{0.030}{0.002}\\
                        & ESS ($\uparrow$) & \tabres{73.3}{0.1} & \tabres{89.5}{0.0} & \tabres{102.6}{0.0}\\
        \bottomrule
        \end{tabular}}
        \label{tab:sensitivity_kappa}
\end{table*}